\documentclass[11pt]{article}
\usepackage[left=1in,right=1in,top=1in,bottom=1in]{geometry}

\usepackage[utf8]{inputenc}
\usepackage[T1]{fontenc}
\usepackage{setspace}

\usepackage{algorithm}
\usepackage{algpseudocode}
\usepackage{enumitem}

\usepackage{amsmath,amssymb,amsfonts,amsthm}

\usepackage{graphicx}
\usepackage{booktabs}
\usepackage{array}
\usepackage{makecell}
\usepackage{tabularx}
\usepackage{placeins}

\usepackage{natbib}

\usepackage{xcolor}

\usepackage[colorlinks=true,allcolors=blue]{hyperref}

\theoremstyle{plain}
\newtheorem{theorem}{Theorem}[section]
\newtheorem{proposition}[theorem]{Proposition}
\newtheorem{lemma}[theorem]{Lemma}
\newtheorem{corollary}[theorem]{Corollary}

\theoremstyle{definition}

\theoremstyle{remark}
\newtheorem{remark}{Remark}[section]
\graphicspath{{figures/}}

\title{Differentially Private Manifold Denoising}

\author{
Jiaqi Wu\textsuperscript{*},
Yiqing Sun\textsuperscript{*},
Zhigang Yao\textsuperscript{\textdagger}\\[0.5em]
Department of Statistics and Data Science, National University of Singapore
}

\date{}

\begin{document}
\addtocontents{toc}{\protect\setcounter{tocdepth}{-1}}
\maketitle

\begingroup
\renewcommand{\thefootnote}{\fnsymbol{footnote}}
\footnotetext[1]{J.W. and Y.S. contributed equally to this work.}
\footnotetext[2]{To whom correspondence should be addressed. E-mail: zhigang.yao@nus.edu.sg}
\endgroup

\setstretch{1.3}

\begin{abstract}
Modern data analysis increasingly relies on manifold-based methods that use local geometry to extract low-dimensional structure from noisy observations. Yet in domains such as biomedicine and public health, privacy regulations restrict access to individual-level data, while perturbations introduced to protect privacy can distort the geometric information needed for analysis. We introduce a manifold denoising framework that uses private reference data to correct noisy public queries, bringing differential privacy to local geometric estimation. Starting from a hard correction based on the estimated tangent space, we develop a noise-adaptive soft correction that avoids specifying the intrinsic dimension. Furthermore, we construct a shared Gaussian field that couples evaluations across locations and bandwidths with moderate marginal noise inflation. Under standard assumptions on manifold regularity and sampling density, we establish non-asymptotic utility guarantees that characterize how denoising accuracy depends on sample size, noise level, bandwidth, and the privacy budget. Simulations and UK Biobank applications demonstrate accurate signal recovery under moderate privacy budgets. These results illustrate clear utility--privacy trade-offs and provide a practical DP component for manifold-based workflows in regulated environments.

\end{abstract}

{\noindent\textit{Keywords:} Differential privacy, Smooth latent structure, Manifold geometry, Kernel smoothing\par}

\setstretch{1.5}

\section{Introduction}
Modern datasets grow rapidly in both scale and complexity, spanning gene expression profiles in genomics, images in computer vision, and word embeddings in natural language processing. A central challenge is the curse of dimensionality. As dimension increases, data become sparsely distributed, making statistical inference, model training, and generalization difficult. Fortunately, a widely observed structural regularity helps mitigate the difficulty. Under the \emph{manifold hypothesis}, observations in a high-dimensional ambient space $\mathbb{R}^D$ concentrate near a smooth $d$-dimensional manifold $\mathcal{M}$ with $d \ll D$. This low-dimensional geometric structure appears as pose manifolds in images \citep{tenenbaum2000global}, orientation manifolds in structural biology \citep{dashti2014trajectories,frank2016continuous}, and latent cell-state manifolds in single-cell transcriptomics \citep{haghverdi2016diffusion,yao2024single,li2025manifold}. Using this low-dimensional structure has become important for statistically efficient analysis.

Yet in domains where such structure is most valuable (e.g., biomedicine, public health, finance), these structured datasets often contain sensitive individual-level information, including medical records, genomic sequences, and financial transactions. Exploiting manifold structure can therefore conflict with protecting the privacy of the individuals whose data contribute to the learned geometry. Regulatory frameworks such as HIPAA for health privacy \citep{HIPAA164514}, the GDPR for personal data protection \citep{GDPR2016}, and the EU AI Act for AI governance \citep{EUAIAct2024} motivate safeguards against disclosures that could identify individuals or reveal personal information through released outputs, such as trained models, embeddings, or denoised signals. \emph{Differential privacy} (DP) formalizes individual privacy protection by bounding the influence of any individual record on the output distribution \citep{dwork2006calibrating}, with privacy guarantees that admit rigorous composition. The challenge is to exploit manifold structure under a finite privacy budget.

In this work, we focus on a common scenario in which a \emph{private reference dataset} contains geometric signal but cannot be exposed directly. Practitioners would like to use this sensitive reference to denoise or interpret new, noisy observations (\emph{public queries}) without compromising privacy. Moreover, the reference observations are themselves noisy, so the geometry used for denoising also has to be estimated from imperfect data.

Privacy and denoising cannot be treated in isolation: naively added DP noise can erase the local geometry that denoising requires, whereas unguarded denoising risks leaking whether a person's data were used or even revealing sensitive attributes. Our approach bridges manifold denoising and private data analysis by jointly recovering local geometric structure and enforcing explicit privacy accounting.

\par\medskip
\noindent\textbf{From noisy observations to geometry: non-private methods.}
When privacy is not a constraint, extensive work connects point clouds to smooth geometric structures on an unknown manifold $\mathcal{M}$ \citep{niyogi2008finding,ozertem2011locally,boissonnat2014manifold}. Measurement noise complicates recovery of the latent manifold $\mathcal{M}$, motivating manifold denoising and manifold fitting \citep{hein2006manifold,fefferman2018fitting}. Foundational work \citep{genovese2012manifold,genovese2012minimax,genovese2014nonparametric} established the statistical theory and limits for manifold recovery from noisy observations. Recent progress has led to stable and computationally tractable reconstruction algorithms \citep{fefferman2018fitting,aamari2018stability}, including structure-adaptive tangent refinement \citep{puchkin2022structure}, reconstruction under unbounded Gaussian noise with smoothness guarantees \citep{yao2025manifold}, and local polynomial regression for projection and tangent estimation \citep{aizenbud2025estimation}. In particular, \citet{yao2023manifold} introduced a manifold projection estimator via local contraction with computational efficiency and favorable sample complexity. This approach has enabled applications in latent map learning, single-cell genomics, and population-scale metabolomics \citep{yao2024single,yao2024manifold,li2025manifold}.

These methods assume unrestricted access to the reference data. In regulated settings, however, such access may be restricted, and naively adding perturbations can distort the local geometry needed for denoising. 

\par\medskip
\noindent\textbf{From sensitive data to releases: differentially private methods.}
Differential privacy provides output-level protection by limiting the effect of changing a single record on the release distribution \citep{dwork2006calibrating,dwork2014algorithmic}. In Euclidean settings, standard mechanisms are well established. For means and linear models, Laplace and Gaussian mechanisms provide standard approaches by adding sensitivity-calibrated noise to suitable statistics \citep{dwork2006calibrating,dwork2014algorithmic}, the exponential mechanism handles utility-driven releases \citep{mcsherry2007mechanism}, and objective-perturbation procedures provide private regularized empirical risk minimization (ERM) \citep{chaudhuri2011differentially}. For function-valued outputs, Gaussian-process mechanisms exploit correlations across locations to enable joint private release \citep{hall2013functions}. When combining multiple separately privatized releases, tight privacy accounting is enabled by concentrated/R\'enyi DP and Gaussian DP \citep{bun2016concentrated,mironov2017renyi,dong2022gaussian}, with the moments accountant providing a practical implementation for DP-SGD \citep{abadi2016deep}. Across these tasks, the statistical cost of privacy depends critically on the data-access model. While local DP often incurs substantial statistical slowdowns \citep{duchi2018minimax}, the central model typically retains non-private minimax rates up to an explicit privacy penalty \citep{cai2021cost}.

For private PCA, early work added noise to empirical covariance matrices or applied the exponential mechanism to estimate principal components \citep{blum2005practical,chaudhuri2013near,dwork2014analyze}. Later studies improved covariance estimation \citep{amin2019differentially,dong2022differentially} and established near-optimal rates for principal subspace recovery \citep{liu2022dp,cai2024optimal}. Beyond linear subspaces, DP has been developed for known manifolds. \citet{reimherr2021differential} privatized manifold-valued summaries (e.g., Fr\'echet means) by extending the Laplace/$K$-norm mechanism using geodesic distances, whereas \citet{han2024differentially} performed differentially private Riemannian optimization by adding noise to gradients in tangent spaces when parameters live on a manifold.

However, global PCA captures only linear structure, while these manifold mechanisms assume the manifold is known. In contrast, we study geometry-aware denoising when the manifold is unknown and the reference set is accessed only through privatized local summaries. We ask:

\par\medskip
\noindent
\emph{How can $(\varepsilon,\delta)$-DP local geometric summaries of the reference dataset guide noisy queries toward their projections onto an unknown $C^2$ manifold with meaningful utility guarantees?}
\par\medskip

This question reflects two competing goals: manifold denoising seeks signal recovery by reducing noise, whereas differential privacy requires signal obfuscation through noise injection. Concretely, privacy noise in geometric estimation can propagate through corrections and shift queries away from the manifold. Moreover, naive releases of local summaries quickly accumulate privacy loss and degrade geometric estimates. An effective DP manifold denoiser needs to pair careful privacy accounting with estimation to jointly manage measurement noise and the privacy noise introduced when releasing reference summaries. 

Few works target the intersection of differential privacy and manifold methods. Existing approaches privatize representations through DP variants of t-SNE, Laplacian eigenmaps, and supervised embeddings \citep{arora2019differentially,saha2022privacy,vepakomma2022privatemail}, but focus on downstream representations without preserving the underlying geometry with theoretical guarantees. To the best of our knowledge, this is the first work to study manifold denoising under differential privacy for unknown manifolds with non-asymptotic error guarantees.

\par\medskip
\noindent\textbf{Model and setup.}
Let $\mathcal D = \{\mathbf{y}_i\}_{i=1}^n \subset \mathbb{R}^D$ be a private reference dataset generated according to
\begin{equation}
\label{eq:model}
    \mathbf{y}_i = \mathbf{x}_i + \boldsymbol{\epsilon}_i, 
\end{equation}
where the latent points $\mathbf{x}_i$ are drawn i.i.d.\ from a distribution supported on a compact, $d$-dimensional $C^{2}$ submanifold $\mathcal M \subset \mathbb R^D$ without boundary.\footnote{In the paper, we do not require the manifold to be connected.} The sampling density $f(\mathbf{x})$ with respect to the $d$-dimensional Hausdorff measure on $\mathcal M$ satisfies $0 < f_{\min} \le f(\mathbf{x}) \le f_{\max} < \infty$ for all $\mathbf{x} \in \mathcal M$. The measurement perturbations $\boldsymbol{\epsilon}_i \in \mathbb{R}^D$ satisfy $\|\boldsymbol{\epsilon}_i\| \le \sigma$. The manifold $\mathcal M$ is assumed to have reach at least $\tau$, and we denote by $\mathcal M_{\gamma} := \mathcal M \oplus B_D(0,\gamma)$ its $\gamma$-tubular neighborhood, with $0 \le \gamma < \tau$ to guarantee a well-defined projection.

Given a query point $\mathbf z\in\mathcal M_\gamma$ (which need not belong to the sample), the goal is to estimate its orthogonal projection $\mathbf z^\star:=\pi_{\mathcal M}(\mathbf z)$ onto the latent manifold $\mathcal M$ (with multiple queries considered later).\footnote{For queries generated by \eqref{eq:model} with normal-direction noise, projection onto $\mathcal M$ recovers the latent point.} We treat $\mathbf z$ as public and protect the reference dataset $\mathcal D$. The public/private split is standard \citep{wang2023generalized}, where externally provided parameters (e.g., evaluation points) are non-sensitive and can be used without extra privacy cost. We operate in the central model of differential privacy, where a trusted curator has direct access to $\mathcal D$ and answers public queries via an $(\varepsilon,\delta)$-DP mechanism.

\par\medskip
\noindent\textbf{Our contributions.}

\par\medskip
\noindent
(i) \emph{Private local geometry for denoising.} We develop a differentially private manifold denoising framework in the private-reference/public-queries regime by jointly privatizing local moments up to second order. We establish sensitivity bounds and finite-sample guarantees quantifying the effects of geometric approximation, measurement noise, and privacy noise.

\par\medskip
\noindent
(ii) \emph{Noise-adaptive correction.} Starting from a hard-projection correction, we develop a soft correction that adapts to privacy noise without requiring the intrinsic dimension.

\par\medskip
\noindent
(iii) \emph{Shared releases across queries and scales.} We construct a shared private field for multiple queries and bandwidths, avoiding separate privatization of local summaries, with a moderate increase in privacy noise over the single-query, single-bandwidth setting.

\par\medskip
The remainder of the paper is organized as follows. Section~\ref{sec:private_local_geometry} presents private local geometry estimation and hard-projection denoising. Section~\ref{sec:soft_correction} develops a noise-adaptive soft correction, followed by an extension of the framework to multiple queries and bandwidths in Section~\ref{sec:shared_private_field}. Sections~\ref{sec:sim} and~\ref{sec:application} evaluate the methods on simulated data and UK Biobank applications. Section~\ref{sec:discussion} concludes with future directions.

\par\medskip
\noindent\textbf{Notation.}
We write $C,c>0$ for constants whose values may change from line to line. Let $[n]=\{1,\ldots,n\}$ and $a\vee b=\max(a,b)$. For a symmetric matrix $\mathbf A\in\mathbb R^{D\times D}$, define $\operatorname{svec}(\mathbf A):=\big((A_{jj})_{j=1}^D,(\sqrt2 A_{jk})_{1\le j<k\le D}\big)^\top$, so that $\|\operatorname{svec}(\mathbf A)\|=\|\mathbf A\|_{\mathrm F}$. We denote its inverse by $\operatorname{smat}$ and write $\mathbf A_+$ for the positive semidefinite part of $\mathbf A$, obtained by replacing negative eigenvalues with zero. Write $B_D(\mathbf x,r)$ for the closed Euclidean ball and $T_{\mathbf x}\mathcal M$ for the tangent space at $\mathbf x\in\mathcal M$, with orthogonal projector $\boldsymbol\Pi_{\mathbf x}$. The reach of $\mathcal M$ is the supremum of radii $r$ such that every point within distance $r$ of $\mathcal M$ has a unique nearest point on $\mathcal M$; the corresponding projection is denoted by $\pi_{\mathcal M}$.
\section{Private local geometry}
\label{sec:private_local_geometry}

We first introduce some basic concepts of differential privacy. Two datasets $\mathcal D,\mathcal D'$ of size $n$ are \emph{adjacent} if they differ by a single record. A randomized mechanism $\mathcal A$ is \emph{$(\varepsilon,\delta)$-differentially private} if, for all adjacent datasets $\mathcal D,\mathcal D'$ and all measurable sets $S$,
\[
\mathbb P\big(\mathcal A(\mathcal D)\in S\big)
\le e^{\varepsilon}\mathbb P\big(\mathcal A(\mathcal D')\in S\big)+\delta.
\]
This guarantees that modifying a single individual's record changes the output distribution only slightly. It thereby limits the information revealed about that individual. The parameter $\varepsilon$ controls the multiplicative change, and $\delta$ allows a small additive tolerance; smaller values yield stronger privacy \citep{dwork2006calibrating,dwork2014algorithmic}. In the paper we use bounded (replace-one) adjacency. Results extend to unbounded (add/remove-one) adjacency by constant-factor sensitivity rescaling.

For a query $\mathbf f:\mathcal X^n\to\mathbb R^k$, its $\ell_2$-sensitivity is $\Delta_2(\mathbf f):=\sup_{\mathcal D\sim\mathcal D'}\|\mathbf f(\mathcal D)-\mathbf f(\mathcal D')\|$. The \emph{Gaussian mechanism} releases $\mathcal A(\mathcal D)=\mathbf f(\mathcal D)+\boldsymbol\xi$, where $\boldsymbol\xi\sim\mathcal N(\mathbf 0,\varsigma^2\mathbf I_k)$ with noise scale $\varsigma>0$. To guarantee $(\varepsilon,\delta)$-DP for any $\varepsilon>0$ and $\delta\in(0,1)$, setting $\mu=\Delta_2(\mathbf f)/\varsigma$, it suffices to choose $\varsigma$ such that
\begin{equation}
\label{eq:gaussian-calibration}
\Phi\big(\frac{\mu}{2}-\frac{\varepsilon}{\mu}\big)
-e^{\varepsilon}\Phi\big(-\frac{\mu}{2}-\frac{\varepsilon}{\mu}\big)
\le\delta,
\end{equation}
where $\Phi$ is the standard normal CDF \citep{balle2018improving}. Throughout, we fix $\mu>0$ as the largest value satisfying \eqref{eq:gaussian-calibration}. Here $\mu$ also characterizes $\mu$-Gaussian differential privacy ($\mu$-GDP).  Adaptive composition of $\mu_1,\dots,\mu_r$-GDP mechanisms satisfies $\mu$-GDP with $\mu^2 = \sum_{s=1}^r \mu_s^2$ \citep{dong2022gaussian}. Furthermore, post-processing ensures that $\phi\circ\mathcal A$ remains $(\varepsilon,\delta)$-DP for any mapping $\phi$ that does not access the private dataset.

\subsection{Tangent space estimation}
\label{subsec:tangent_space_estimation}
For a public query $\mathbf z\in\mathcal M_\gamma$, we first estimate the local tangent space $T_{\mathbf z^\star}\mathcal M$ from the private reference data $\mathcal D$ under differential privacy. We fix a bandwidth $h>0$ and suppress it in the notation unless multiple bandwidths are considered, when we include it as a subscript.

Define the weighted local mean $\bar{\mathbf y}_{\mathbf z}:=\sum_{i=1}^n w_{i,\mathbf z}\mathbf y_i/\sum_{i=1}^n w_{i,\mathbf z}$ and the weighted local covariance
\begin{equation}
\label{eq:local-cov}
\widehat{\boldsymbol\Sigma}_{\mathbf z}
:=\frac{\sum_{i=1}^n w_{i,\mathbf z}
(\mathbf y_i-\bar{\mathbf y}_{\mathbf z})
(\mathbf y_i-\bar{\mathbf y}_{\mathbf z})^\top}
{\sum_{i=1}^n w_{i,\mathbf z}},
\end{equation}
where $w_{i,\mathbf z}:=k_h(\mathbf y_i-\mathbf z)$ are Gaussian kernel weights, with $k_h(\mathbf v):=\exp(-\|\mathbf v\|^2/(2h^2))$.

Intuitively, $\widehat{\boldsymbol\Sigma}_{\mathbf z}$ captures the directional spread of local observations. For a localization scale small relative to the reach, the manifold is locally close to its affine tangent space $\mathbf z^\star+T_{\mathbf z^\star}\mathcal M$. With sufficient sampling and measurement noise small relative to the localization scale, tangential variation dominates normal fluctuations. As a result, the top-$d$ eigenvectors of $\widehat{\boldsymbol\Sigma}_{\mathbf z}$ align with $T_{\mathbf z^\star}\mathcal M$, motivating local PCA \citep{aamari2018stability,puchkin2022structure}.

\begin{remark}
The Gaussian kernel provides smooth localization of the reference data around $\mathbf z$. Crucially, its decay bounds the sensitivity of the local moments on unbounded data domains (see Proposition~\ref{prop:local-moment-sensitivity}), while the kernel also allows us to exploit correlations across query locations and bandwidths to control the noise in joint private releases (see Section~\ref{sec:shared_private_field}).
\end{remark}

To privatize the local mean and tangent space, we express both in terms of weighted local moments up to order two:
\[
m_{0,\mathbf z}:=\frac1n\sum_{i=1}^n w_{i,\mathbf z},
\quad
\mathbf m_{1,\mathbf z}:=\frac1n\sum_{i=1}^n w_{i,\mathbf z}(\mathbf y_i-\mathbf z),
\quad
\mathbf M_{2,\mathbf z}:=\frac1n\sum_{i=1}^n w_{i,\mathbf z}(\mathbf y_i-\mathbf z)(\mathbf y_i-\mathbf z)^\top.
\]
Then $\bar{\mathbf y}_{\mathbf z}=\mathbf z+\mathbf m_{1,\mathbf z}/m_{0,\mathbf z}$ and
\begin{equation}
\widehat{\boldsymbol\Sigma}_{\mathbf z}
=\frac{\mathbf M_{2,\mathbf z}}{m_{0,\mathbf z}}
-\frac{\mathbf m_{1,\mathbf z}\mathbf m_{1,\mathbf z}^{\top}}
{m_{0,\mathbf z}^{2}}.
\end{equation}
Thus, private estimation of the local mean and covariance reduces to jointly privatizing these three moments. Define $\mathbf F_{\mathbf z}(\mathcal D) := (m_{0,\mathbf z}, \mathbf m_{1,\mathbf z}/h, \operatorname{svec}(\mathbf M_{2,\mathbf z})/(\sqrt{2}h^2))$. The following proposition gives its sensitivity and the resulting privacy guarantee.

\begin{proposition}
\label{prop:local-moment-sensitivity}
The sensitivity satisfies $\Delta_2(\mathbf F_{\mathbf z})\le\frac{2}{n\sqrt e}$. Consequently, releasing $\mathbf F_{\mathbf z}(\mathcal D)$ via the Gaussian mechanism with $\varsigma=\frac{2}{n\mu\sqrt e}$ is $(\varepsilon,\delta)$-DP.
\end{proposition}

\begin{remark}
A notable feature is that, after rescaling, the sensitivity bound is entirely independent of the query location, bandwidth, and ambient dimension. Furthermore, the bound is sharp, as equality can be attained when the data domain covers $\mathbb R^D$.
\end{remark}

\begin{algorithm}[htbp]
\caption{\textsc{Private-Local-Moments}}
\label{alg:private-local-moments}
\begin{algorithmic}[1]
\Require Private reference data $\mathcal D$, public query $\mathbf z$, bandwidth $h$, privacy parameters $(\varepsilon,\delta)$
\State Compute $m_{0,\mathbf z}$, $\mathbf m_{1,\mathbf z}$, and $\mathbf M_{2,\mathbf z}$
\State Add Gaussian noise
\[
\tilde m_{0,\mathbf z}:=m_{0,\mathbf z}+\xi_0,
\qquad
\tilde{\mathbf m}_{1,\mathbf z}:=\mathbf m_{1,\mathbf z}+h\boldsymbol\xi_1,
\qquad
\widetilde{\mathbf M}_{2,\mathbf z}:=\mathbf M_{2,\mathbf z}+\sqrt2 h^2\operatorname{smat}(\boldsymbol\xi_2),
\]
\Statex where $(\xi_0,\boldsymbol\xi_1,\boldsymbol\xi_2)\sim\mathcal N(\mathbf 0,\varsigma^2\mathbf I_p)$, $p=1+D+D(D+1)/2$, and $\varsigma=2/(n\mu\sqrt e)$ with $\mu$ satisfying \eqref{eq:gaussian-calibration}
\State \Return $(\tilde m_{0,\mathbf z},\tilde{\mathbf m}_{1,\mathbf z},\widetilde{\mathbf M}_{2,\mathbf z})$
\end{algorithmic}
\end{algorithm}

Algorithm~\ref{alg:private-local-moments} summarizes the resulting $(\varepsilon,\delta)$-DP release of the local moments. From these privatized moments, we define the private local mean $\tilde{\boldsymbol\mu}_{\mathbf z}:=\mathbf z+\tilde{\mathbf m}_{1,\mathbf z}/\tilde q_{\mathbf z}$ and covariance
\begin{equation}
\label{eq:private-local-cov}
\widetilde{\boldsymbol\Sigma}_{\mathbf z}
:=\frac{\widetilde{\mathbf M}_{2,\mathbf z}}{\tilde q_{\mathbf z}}
-\frac{\tilde{\mathbf m}_{1,\mathbf z}\tilde{\mathbf m}_{1,\mathbf z}^{\top}}{\tilde q_{\mathbf z}^{2}},
\end{equation}
where $\tilde q_{\mathbf z}:=\tilde m_{0,\mathbf z}\vee(\eta/n)$ for a public threshold $\eta>0$. Let $\widetilde{\mathbf V}_{\mathbf z,d}$ be an orthonormal basis matrix for the leading $d$-dimensional eigenspace of $\widetilde{\boldsymbol\Sigma}_{\mathbf z}$. We estimate $T_{\mathbf z^\star}\mathcal M$ by its column space $\widetilde{T_{\mathbf z^\star}\mathcal M}$, with projector $\widetilde{\mathbf P}_{\mathbf z}=\widetilde{\mathbf V}_{\mathbf z,d}\widetilde{\mathbf V}_{\mathbf z,d}^{\top}$. By post-processing, this construction yields jointly $(\varepsilon,\delta)$-DP estimates of the local mean, covariance, and tangent space.

We next quantify the accuracy of the DP tangent estimator at a query within the localization scale of $\mathcal M$.

\begin{theorem}
\label{thm:private-tangent-accuracy}
Let $c_0>0$ be a fixed absolute constant. There exists a constant $C>0$, depending only on $d,\tau^d f_{\max},\tau^d f_{\min}$, such that the following holds. Suppose $\gamma\le c_0h$. Define the principal-angle distance
$\operatorname{dist}_\angle\big(T_{\mathbf{z}^\star}\mathcal{M},\,\widetilde{T_{\mathbf{z}^\star}\mathcal{M}}\big):=\|\sin\Theta\big(\mathbf{V}^\star_{\mathbf z},\,\widetilde{\mathbf V}_{\mathbf z,d}\big)\|=\|\boldsymbol\Pi_{\mathbf z^\star}-\widetilde{\mathbf{P}}_{\mathbf{z}}\|$, 
where $\boldsymbol\Pi_{\mathbf z^\star}=\mathbf V^\star_{\mathbf z}\mathbf V_{\mathbf z}^{\star\top}$ projects onto $T_{\mathbf z^\star}\mathcal M$, with $\mathbf V^\star_{\mathbf z}$ an orthonormal basis. For $t\ge1$, if $n(h/\tau)^d\ge C(\eta+t)$, then, with probability at least $1-e^{-t}$,
\[
\operatorname{dist}_\angle\big(T_{\mathbf{z}^\star}\mathcal{M},\,\widetilde{T_{\mathbf{z}^\star}\mathcal{M}}\big)
\le
\min\biggl[
1,\,
C\biggl(
\underbrace{\frac h\tau}_{\text{curvature}}
+\underbrace{\frac{\sigma}{h}}_{\text{noise}}
+\underbrace{\frac{\varsigma(\sqrt D+\sqrt t)}{(h/\tau)^d}}_{\text{privacy}}
\biggr)
\biggr].
\]
\end{theorem}

The error bound in Theorem~\ref{thm:private-tangent-accuracy} decomposes into three parts. First, since a $C^2$ manifold with reach at least $\tau$ locally deviates from its affine tangent space by $O(h^2/\tau)$ at scale $h$, linearizing it induces a curvature bias of order $h/\tau$. Second, measurement noise perturbs the local eigenspace relative to the localization scale, contributing $O(\sigma/h)$. Third, the Gaussian moment perturbations, normalized by the local mass of order $(h/\tau)^d$, propagate to the local covariance and, together with the tangential eigengap of order $h^2$, yield the privacy contribution $\varsigma(\sqrt D+\sqrt t)/(h/\tau)^d$. Here and throughout, the noise scale $\varsigma=2/(n\mu\sqrt e)$ captures the dependence on the sample size and privacy parameters. In particular, the Gaussian calibration \eqref{eq:gaussian-calibration} yields $\varsigma\lesssim\sqrt{\log(1/\delta)}/(n\varepsilon)$ for bounded $\varepsilon$ and $\delta\le1/2$.

\begin{remark}
\label{rem:tangent-properties}
Two points warrant emphasis. First, the manifold structure yields a sampling requirement governed by the intrinsic dimension, while the ambient dimension only enters the error bound through the $\sqrt D$ factor in the privacy term, reflecting the cost of Gaussian noise injection in $\mathbb R^D$. Second, the bound holds for measurement errors satisfying $\|\boldsymbol\epsilon_i\|\le\sigma$, which may be dependent across observations as well as on the latent sample, and may have nonzero means (see the proof of Lemma~\ref{lem:weighted-local-geometry} in Supplementary Appendix~\ref{app:geometric-probabilistic-preliminaries}).
\end{remark}

\subsection{Hard correction}
\label{subsec:hard_correction}

We now denoise a query point $\mathbf z\in\mathcal M_\gamma$ by estimating its normal displacement $\mathbf z-\mathbf z^\star$. Since the normal projector $\mathbf I-\boldsymbol\Pi_{\mathbf z^\star}$ preserves $\mathbf z-\mathbf z^\star$ while canceling tangential offsets, any point on the affine tangent space $\mathbf z^\star+T_{\mathbf z^\star}\mathcal M$ can serve as a base point, which we approximate by the local mean $\bar{\mathbf y}_{\mathbf z}$ with normal deviation of order $h^2/\tau+\sigma$. Normal projection of $\mathbf z-\bar{\mathbf y}_{\mathbf z}$ recovers the target displacement up to an error of the same order \citep{yao2025manifold}.

Replacing the non-private quantities with their private counterparts from Section~\ref{subsec:tangent_space_estimation}, we define the hard correction
\begin{equation}
\label{eq:hard-correction}
\tilde{\mathbf z}_{\mathrm{hard}}
:=\mathbf z-(\mathbf I-\widetilde{\mathbf P}_{\mathbf z})
(\mathbf z-\tilde{\boldsymbol\mu}_{\mathbf z})
=\tilde{\boldsymbol\mu}_{\mathbf z}
+\widetilde{\mathbf P}_{\mathbf z}
(\mathbf z-\tilde{\boldsymbol\mu}_{\mathbf z}).
\end{equation}
Geometrically, $\tilde{\mathbf z}_{\mathrm{hard}}$ is the orthogonal projection of $\mathbf z$ onto the estimated affine tangent space at $\tilde{\boldsymbol\mu}_{\mathbf z}$. Since both $\tilde{\boldsymbol\mu}_{\mathbf z}$ and $\widetilde{\mathbf P}_{\mathbf z}$ are derived from Algorithm~\ref{alg:private-local-moments}, computing $\tilde{\mathbf z}_{\mathrm{hard}}$ is post-processing and incurs no additional privacy cost.

We next quantify the estimation error of $\tilde{\mathbf z}_{\mathrm{hard}}$.

\begin{theorem}
\label{thm:private-hard-correction}
Let $c_0>0$ be a fixed absolute constant. There exist constants $C,c>0$, depending only on $d,\tau^d f_{\max},\tau^d f_{\min}$, such that the following holds. Suppose $\gamma\le c_0h$. For $t\ge1$, if $n(h/\tau)^d\ge C(\eta+t)$ and $h/\tau + \sigma/h + \varsigma(\sqrt{D}+\sqrt{t})/(h/\tau)^d \le c$, then, with probability at least $1-e^{-t}$,
\[
\|\tilde{\mathbf z}_{\mathrm{hard}}-\mathbf z^\star\|
\le C\biggl(
\frac{h^2}{\tau}+\sigma
+\frac{h\varsigma(\sqrt D+\sqrt t)}{(h/\tau)^d}
\biggr).
\]
\end{theorem}

\begin{remark}
The condition $h/\tau+\sigma/h+\varsigma(\sqrt D+\sqrt t)/(h/\tau)^d\le c$ ensures a non-trivial tangent-space estimate in Theorem~\ref{thm:private-tangent-accuracy}. The positional error in Theorem~\ref{thm:private-hard-correction} scales the angular error by $h$, reflecting the projector's action on the displacement. As the privacy budget tends to infinity ($\varepsilon\to\infty$, so $\varsigma\to0$), the bound recovers the non-private rate.
\end{remark}

\begin{remark}
\label{rem:interior-boundary}
The theoretical guarantees extend directly to compact $C^2$ manifolds with boundary. Under the condition $\operatorname{dist}(\mathbf z^\star,\partial\mathcal M)\ge\kappa\tau$ for a fixed constant $\kappa>0$, the tangent-space and denoising error bounds retain their stated rates. The corresponding justification is provided in Supplementary Appendix~\ref{app:boundary-extension}.
\end{remark}
\section{Soft correction}
\label{sec:soft_correction}

We now present the soft correction, which eliminates the need to specify or pre-estimate the intrinsic dimension $d$. Our key observation from \eqref{eq:hard-correction} is that the hard normal projector acts as a binary spectral filter on the local-mean displacement $\tilde{\mathbf r}_{\mathbf z}:=\tilde{\boldsymbol\mu}_{\mathbf z}-\mathbf z$, assigning zero weight to the top $d$ eigendirections of $\widetilde{\boldsymbol\Sigma}_{\mathbf z}$ and unit weight to their orthogonal complement. Using the same private moments, the soft correction replaces the discrete subspace truncation with smoothly varying, eigenvalue-dependent weights.

Specifically, for $\lambda>0$, define the quadratic spectral filter $\mathbf W_{\lambda,\mathbf z}:=\lambda^2\bigl((\widetilde{\boldsymbol\Sigma}_{\mathbf z})_+^2+\lambda^2\mathbf I\bigr)^{-1}$, and the resulting soft correction
\begin{equation}
\label{eq:soft-correction}
\tilde{\mathbf z}_{\mathrm{soft}}(\lambda)
:=\mathbf z-\mathbf W_{\lambda,\mathbf z}
(\mathbf z-\tilde{\boldsymbol\mu}_{\mathbf z})=\mathbf z+\mathbf W_{\lambda,\mathbf z}\tilde{\mathbf r}_{\mathbf z}.
\end{equation}
For $\lambda=0$, we use $\mathbf W_{0,\mathbf z}:=\lim_{\lambda\rightarrow0^+}\mathbf W_{\lambda,\mathbf z}$, the orthogonal projector onto $\ker((\widetilde{\boldsymbol\Sigma}_{\mathbf z})_+)$. The matrix $\mathbf W_{\lambda,\mathbf z}$ assigns weight $\lambda^2/(s^2+\lambda^2)$ to each eigendirection of $(\widetilde{\boldsymbol\Sigma}_{\mathbf z})_+$ with eigenvalue $s$. Here $\lambda$ sets the transition scale, suppressing tangential components ($s\gg\lambda$) while retaining normal displacements ($s\ll\lambda$). The following theorem quantifies the estimation error and shows that an appropriate choice of $\lambda$ recovers the hard-correction rate.

\begin{theorem}
\label{thm:private-soft-correction}
Let $c_0>0$ be a fixed absolute constant. There exist constants $C,c>0$, depending only on $d,\tau^d f_{\max},\tau^d f_{\min}$, such that the following holds. Suppose $\gamma\le c_0h$. For $t\ge1$, if $n(h/\tau)^d\ge C(\eta+t)$ and $h/\tau + \sigma/h + \varsigma(\sqrt{D}+\sqrt{t})/(h/\tau)^d \le c$, then, with probability at least $1-e^{-t}$, simultaneously for all $\lambda>0$,
\[
\|\tilde{\mathbf z}_{\mathrm{soft}}(\lambda)-\mathbf z^\star\|
\le C\biggl[
\frac{h^2}{\tau}+\sigma+\frac{h\varsigma(\sqrt D+\sqrt t)}{(h/\tau)^d}
+h\biggl(\frac{\lambda}{h^2}\biggr)^2
+h\biggl(\frac{\Delta_{h,t}}{\lambda}\biggr)^2
\biggr],
\]
where $\Delta_{h,t}:=(h^2/\tau+\sigma)^2+h^2\varsigma(\sqrt D+\sqrt t)/(h/\tau)^d$. In particular, choosing $\lambda_{\mathrm{or}}:=h\sqrt{\Delta_{h,t}}$ yields
\[
\|\tilde{\mathbf z}_{\mathrm{soft}}(\lambda_{\mathrm{or}})-\mathbf z^\star\|
\le C\biggl(
\frac{h^2}{\tau}+\sigma
+\frac{h\varsigma(\sqrt D+\sqrt t)}{(h/\tau)^d}
\biggr).
\]
\end{theorem}

\begin{remark}
The bound in Theorem~\ref{thm:private-soft-correction} consists of the hard-correction error bound and two regularization terms. The terms $h(\lambda/h^2)^2$ and $h(\Delta_{h,t}/\lambda)^2$ capture the tangential shrinkage bias and the normal residual, respectively, where $h^2$ represents the tangential variance scale and $\Delta_{h,t}$ bounds the effective normal variance. Choosing $\lambda_{\mathrm{or}}$ balances these terms and recovers the hard-correction rate.
\end{remark}

\medskip
\noindent\textbf{Practical choice of $\lambda$.}
In practice, we replace the oracle scale $\lambda_{\mathrm{or}}=h\sqrt{\Delta_{h,t}}$ with a data-driven $\widehat\lambda_{\mathbf z}$ from the private release. Let $\tilde s_1\ge\dots\ge\tilde s_D\ge0$ be the eigenvalues of $(\widetilde{\boldsymbol\Sigma}_{\mathbf z})_+$. While the leading eigenvalue $\tilde s_1$ naturally reflects the tangential variance scale, choosing $\lambda$ also requires accounting for normal variation and privacy-induced covariance perturbations.

To estimate the privacy-induced perturbation from the released moments, we linearize the covariance expression $\mathbf M_2/q-\mathbf m_1\mathbf m_1^\top/q^2$ at $(\tilde q_{\mathbf z},\tilde{\mathbf m}_{1,\mathbf z},\widetilde{\mathbf M}_{2,\mathbf z})$ with respect to the Gaussian moment noise. For an independent auxiliary draw $\mathbf g:=(g_0,\mathbf g_1,\mathbf g_2)\sim\mathcal N(\mathbf 0,\mathbf I_p)$, the resulting first-order perturbation is $(\varsigma/\tilde q_{\mathbf z})\mathbf L_{\mathbf z}$, where
\[
\mathbf L_{\mathbf z}
:=(\tilde{\mathbf r}_{\mathbf z}\tilde{\mathbf r}_{\mathbf z}^{\top}
-\widetilde{\boldsymbol\Sigma}_{\mathbf z})g_0
-h(\mathbf g_1\tilde{\mathbf r}_{\mathbf z}^{\top}
+\tilde{\mathbf r}_{\mathbf z}\mathbf g_1^{\top})
+\sqrt2 h^2\operatorname{smat}(\mathbf g_2).
\]
We quantify the perturbation scale by $(\varsigma/\tilde q_{\mathbf z})(\mathbb E_{\mathbf g}\|\mathbf L_{\mathbf z}\|^2)^{1/2}$ and estimate it using $B\ge1$ independent copies $\mathbf g^{(1)},\dots,\mathbf g^{(B)}$ of $\mathbf g$ as
\begin{equation}
\label{eq:soft-noise-scale}
\nu_{\mathbf z}:=(\varsigma/\tilde q_{\mathbf z})\Bigl(B^{-1}\sum_{b=1}^B\|\mathbf L_{\mathbf z}^{(b)}\|^2\Bigr)^{1/2},
\end{equation}
where $\mathbf L_{\mathbf z}^{(b)}$ denotes $\mathbf L_{\mathbf z}$ evaluated at $\mathbf g^{(b)}$.

To incorporate the privacy perturbation scale $\nu_{\mathbf z}$ into the choice of $\lambda$, we set $b_j:=\tilde s_j\vee\nu_{\mathbf z}$ and define
\begin{equation}
\label{eq:soft-practical-scale}
\widehat\lambda_{\mathbf z}:=\sqrt{b_1\widehat b_{*,\mathbf z}},
\end{equation}
where $\widehat b_{*,\mathbf z}:=\max\{b_j:b_j\le\sqrt{b_1b_D}\vee\alpha_0b_1\}$ for public $\alpha_0\in(0,1)$.

\begin{remark}
The choice of $\widehat\lambda_{\mathbf z}$ exploits the spectral separation between tangential and normal directions. The scale $b_1$ captures tangential variation, while the threshold $\sqrt{b_1b_D}\vee\alpha_0 b_1$ identifies the lower spectral group associated with normal variation and privacy noise.
\end{remark}

\begin{corollary}
\label{cor:adaptive-soft-correction}
Let $c_0>0$ be a fixed absolute constant. There exist constants $\alpha_*,C>0$, depending only on $d,\tau^d f_{\max},\tau^d f_{\min}$, such that the following holds. Suppose $\gamma\le c_0h$. For any fixed $\alpha_0\in(0,\alpha_*)$, there exists a constant $c>0$, depending additionally on $\alpha_0$, such that for $t\ge1$ and $B\ge1$, if $n(h/\tau)^d\ge C(\eta+t)$ and $h/\tau + \sigma/h + \varsigma(\sqrt{D}+\sqrt{t})/(h/\tau)^d \le c$, then, with probability at least $1-e^{-t}$,
\[
\|\tilde{\mathbf z}_{\mathrm{soft}}(\widehat\lambda_{\mathbf z})-\mathbf z^\star\|
\le C\biggl(
\frac{h^2}{\tau}+\sigma
+\frac{h\varsigma(\sqrt D+\sqrt t)}{(h/\tau)^d}
\biggr).
\]
\end{corollary}

Corollary~\ref{cor:adaptive-soft-correction} shows that the practical choice recovers the oracle rate without matching $\lambda_{\mathrm{or}}$. With $\widehat\lambda_{\mathbf z}^2=b_1\widehat b_{*,\mathbf z}$, the proof in Supplementary Appendix~\ref{app:adaptive-soft-proof} retains the observed eigenvalues to bound the difference between soft and hard corrections by $C\widehat b_{*,\mathbf z}/h$. The threshold isolates the lower spectral group to obtain the required upper bound $\widehat b_{*,\mathbf z}\le C\Delta_{h,t}$, with no matching lower bound needed.

Algorithm~\ref{alg:dp-manifold-denoising} summarizes the unified denoising procedure.

\begin{algorithm}[htbp]
\caption{\textsc{DP-Manifold-Denoiser}}
\label{alg:dp-manifold-denoising}
\begin{algorithmic}[1]
\Require Private reference data $\mathcal D$, public query $\mathbf z$, bandwidth $h$, threshold $\eta$, privacy parameters $(\varepsilon,\delta)$, mode $\in\{\text{hard},\text{soft}\}$ ($d$ for hard or $B\ge1$ and $\alpha_0\in(0,1)$ for soft)
\State $(\tilde m_{0,\mathbf z},\tilde{\mathbf m}_{1,\mathbf z},\widetilde{\mathbf M}_{2,\mathbf z})\leftarrow\textsc{Private-Local-Moments}(\mathcal D,\mathbf z,h,\varepsilon,\delta)$ (Algorithm~\ref{alg:private-local-moments})
\State Set $\tilde q_{\mathbf z}=\tilde m_{0,\mathbf z}\vee(\eta/n)$, $\tilde{\mathbf r}_{\mathbf z}=\tilde{\mathbf m}_{1,\mathbf z}/\tilde q_{\mathbf z}$, and $\widetilde{\boldsymbol\Sigma}_{\mathbf z}=\widetilde{\mathbf M}_{2,\mathbf z}/\tilde q_{\mathbf z}-\tilde{\mathbf r}_{\mathbf z}\tilde{\mathbf r}_{\mathbf z}^\top$
\If{mode is hard}
 \State Compute the top-$d$ eigenspace projector $\widetilde{\mathbf P}_{\mathbf z}$ of $\widetilde{\boldsymbol\Sigma}_{\mathbf z}$
 \State \Return $\mathbf z+(\mathbf I-\widetilde{\mathbf P}_{\mathbf z}){\tilde{\mathbf r}_{\mathbf z}}$
\Else
 \State Compute eigenvalues $\tilde s_1\ge\dots\ge\tilde s_D\ge0$ of $(\widetilde{\boldsymbol\Sigma}_{\mathbf z})_+$, draw $\mathbf g^{(1)},\dots,\mathbf g^{(B)}\sim\mathcal N(\mathbf 0,\mathbf I_p)$, and evaluate $\nu_{\mathbf z}$ via \eqref{eq:soft-noise-scale}
 \State Set $\widehat\lambda_{\mathbf z}=\sqrt{b_1\widehat b_{*,\mathbf z}}$ with $b_j=\tilde s_j\vee\nu_{\mathbf z}$ and $\widehat b_{*,\mathbf z}=\max\{b_j:b_j\le\sqrt{b_1b_D}\vee\alpha_0b_1\}$
 \State Set $\mathbf W_{\widehat\lambda_{\mathbf z},\mathbf z}=\widehat\lambda_{\mathbf z}^{\,2}\bigl((\widetilde{\boldsymbol\Sigma}_{\mathbf z})_+^2+\widehat\lambda_{\mathbf z}^{\,2}\mathbf I\bigr)^{-1}$ if $\widehat\lambda_{\mathbf z}>0$, and $\mathbf W_{0,\mathbf z}$ otherwise
 \State \Return $\mathbf z+\mathbf W_{\widehat\lambda_{\mathbf z},\mathbf z}{\tilde{\mathbf r}_{\mathbf z}}$
\EndIf
\end{algorithmic}
\end{algorithm}
\section{Shared private fields}
\label{sec:shared_private_field}
In practice, one needs to denoise distinct public queries $\mathbf z_1,\dots,\mathbf z_m$. Since the geometry is unknown and may vary locally, one also needs to evaluate the denoiser across multiple bandwidths, for example to adapt to local curvature or to select $h$ adaptively (see Section~\ref{sec:sim}). A natural idea is to divide the privacy budget across queries with independent noise, but this inflates the noise scale by $\sqrt m$ at a fixed bandwidth and further across bandwidths.

The key lies in exploiting the correlations across queries and bandwidths. We first introduce the construction at a fixed bandwidth and then extend it to multiple bandwidths.

\subsection{Sharing across query locations}
\label{subsec:shared-locations-field}
Fix a bandwidth $h>0$. The normalized first- and second-order local moments can be expressed using spatial derivatives of the zeroth-order mass field $m_{0,\mathbf z,h}=n^{-1}\sum_{i=1}^n k_h(\mathbf z-\mathbf y_i)$. Specifically, differentiating the Gaussian kernel weights with respect to $\mathbf z$ yields
\[
\frac{\mathbf m_{1,\mathbf z,h}}{h}
=h\nabla_{\mathbf z}m_{0,\mathbf z,h},
\qquad
\frac{\mathbf M_{2,\mathbf z,h}}{h^2}
=m_{0,\mathbf z,h}\mathbf I_D
+h^2\nabla_{\mathbf z}^2m_{0,\mathbf z,h}.
\]
This derivative relationship shows that each record's contributions across query locations and moment orders share a common kernel structure. We leverage this structure to construct the joint noise covariance.

To illustrate the construction, consider first the zeroth-order moment $m_{0,\mathbf z,h}$. In the RKHS $\mathbb H_h$ associated with the kernel $k_h$, the reproducing property yields $\langle k_h(\mathbf z-\cdot),k_h(\mathbf z'-\cdot)\rangle_{\mathbb H_h}=k_h(\mathbf z-\mathbf z')$. This representation motivates choosing the Gram matrix of the functions $k_h(\mathbf z_j-\cdot)$, with entries $[\mathbf K_h^{(0)}]_{j\ell}:=k_h(\mathbf z_j-\mathbf z_\ell)$ for $j,\ell\in[m]$, as the covariance of the Gaussian noise added to these moments. The resulting whitened sensitivity bound is independent of the number of queries, as established in Proposition~\ref{prop:shared-fixed-bandwidth} below.

For the full moment vector, we define a linear differential operator acting on a twice-differentiable scalar function $\psi$ by
\[
\mathcal A_{\mathbf z,h}\psi
:=\biggl(\psi(\mathbf z),\,
h\nabla\psi(\mathbf z),\,
\frac{\operatorname{svec}\bigl(\psi(\mathbf z)\mathbf I_D+h^2\nabla^2\psi(\mathbf z)\bigr)}{\sqrt2}\biggr).
\]
Then $\mathbf F_{\mathbf z,h}(\mathcal D)=\mathcal A_{\mathbf z,h}m_{0,\cdot,h}$. To account for cross-correlations across locations and derivative orders, we define the joint noise covariance $\mathbf K_h\in\mathbb R^{mp\times mp}$. The $p\times p$ cross-covariance between moment perturbations at $\mathbf z_j$ and $\mathbf z_\ell$ is given by
\[
[\mathbf K_h]_{j\ell}
:=\left.
\mathcal A_{\mathbf z,h}\mathcal A_{\mathbf z',h}^{\top}
 k_h(\mathbf z-\mathbf z')
\right|_{\mathbf z=\mathbf z_j,\,\mathbf z'=\mathbf z_\ell},
\qquad j,\ell\in[m].
\]

\begin{remark}
Equivalently, $\mathbf K_h$ is the covariance matrix of the vector field $\mathbf G_{\mathbf z,h}:=\mathcal A_{\mathbf z,h}G_h$ evaluated at $\mathbf z_1,\dots,\mathbf z_m$, where $G_h$ is a centered Gaussian process on $\mathbb R^D$ with covariance  kernel $k_h$. For zeroth-order moments, this corresponds to the GP mechanism for DP kernel density estimation \citep{hall2013functions}. In our setting, applying $\mathcal A_{\mathbf z,h}$ yields joint perturbations for all three moment orders.
\end{remark}

Given joint covariance $\mathbf K_h$, we draw the perturbation $(\mathbf G_{\mathbf z_1,h}^\top,\dots,\mathbf G_{\mathbf z_m,h}^\top)^\top\sim\mathcal N(\mathbf 0,\mathbf K_h)$ and release
\begin{equation}
\label{eq:shared-moment-release}
\widetilde{\mathbf F}_{\mathbf z_j,h}
:=\mathbf F_{\mathbf z_j,h}(\mathcal D)
+\varsigma_{\mathrm{sh}}\mathbf G_{\mathbf z_j,h},
\qquad j\in[m],
\end{equation}
where $\varsigma_{\mathrm{sh}}$ is the shared-release noise scale. More generally, $\widetilde{\mathbf F}_{\cdot,h}$ can be viewed as a function-valued release on $\mathbb R^D$, with \eqref{eq:shared-moment-release} giving its evaluations at the query locations.

Let $\mathbf F_h(\mathcal D):=\bigl(\mathbf F_{\mathbf z_1,h}(\mathcal D)^\top,\dots,\mathbf F_{\mathbf z_m,h}(\mathcal D)^\top\bigr)^\top$. The following proposition provides the joint privacy guarantee, where the marginal noise scale matches the single-query scale up to a constant factor.

\begin{proposition}
\label{prop:shared-fixed-bandwidth}
The covariance matrix $\mathbf K_h$ is positive definite and satisfies $[\mathbf K_h]_{jj}=\mathbf I_p$ for every $j\in[m]$. For any adjacent reference datasets $\mathcal D,\mathcal D'$,
\[
\bigl\|\mathbf K_h^{-1/2}
\bigl(\mathbf F_h(\mathcal D)-\mathbf F_h(\mathcal D')\bigr)\bigr\|
\le\frac{\sqrt2}{n}.
\]
Consequently, the release \eqref{eq:shared-moment-release} with $\varsigma_{\mathrm{sh}}=\sqrt2/(n\mu)$ is $(\varepsilon,\delta)$-DP, and the guarantee extends to the function-valued release $\widetilde{\mathbf F}_{\cdot,h}:=\mathbf F_{\cdot,h}(\mathcal D)+\varsigma_{\mathrm{sh}}\mathbf G_{\cdot,h}$ on $\mathbb R^D$.
\end{proposition}

\subsection{Sharing across bandwidths}
\label{subsec:bandwidth-sharing}

We now consider a public bandwidth grid $\mathcal H_J=\{h_1,\dots,h_J\}$ with $0<h_1<\dots<h_J$. Dividing the privacy budget equally across releases inflates the marginal noise scale by $\sqrt J$. Instead, we couple the fields across bandwidths by exploiting the spectral overlap.

Note that the Gaussian kernel $k_h$ has spectral density $\rho_h(\boldsymbol\omega):=(h/\sqrt{2\pi})^D e^{-h^2\|\boldsymbol\omega\|^2/2}$. Let $\bar k:=J^{-1}\sum_{\ell=1}^J k_{h_\ell}$ be the mixture kernel with spectral density $\bar\rho:=J^{-1}\sum_{\ell=1}^J\rho_{h_\ell}$, and let $\mathbb H_{\bar\rho}$ be its RKHS. For each $h\in\mathcal H_J$, $k_h(\mathbf z-\cdot)$ lies in $\mathbb H_{\bar\rho}$, where the reproducing property gives $k_h(\mathbf z-\mathbf y_i)=\langle k_h(\mathbf z-\cdot),\bar k(\mathbf y_i-\cdot)\rangle_{\mathbb H_{\bar\rho}}$. We define a centered Gaussian field $G(\mathbf z,h)$ on $\mathbb R^D\times\mathcal H_J$ with cross-covariance $\langle k_h(\mathbf z-\cdot),k_{h'}(\mathbf z'-\cdot)\rangle_{\mathbb H_{\bar\rho}}$, given by
\begin{equation}
\label{eq:shared-scale-covariance}
\mathbb E[G(\mathbf z,h)G(\mathbf z',h')]
=\int_{\mathbb R^D}
\frac{\rho_h(\boldsymbol\omega)\rho_{h'}(\boldsymbol\omega)}
{\bar\rho(\boldsymbol\omega)}
\cos\bigl(\boldsymbol\omega^\top(\mathbf z-\mathbf z')\bigr)
\,\mathrm d\boldsymbol\omega,
\qquad h,h'\in\mathcal H_J.
\end{equation}
This allows joint sensitivity control across locations and bandwidths.

Set $\mathbf G_{\mathbf z,h}:=\mathcal A_{\mathbf z,h}G(\cdot,h)$ and release
\begin{equation}
\label{eq:finite-moment-release}
\widetilde{\mathbf F}_{\mathbf z,h}
:=\mathbf F_{\mathbf z,h}(\mathcal D)
+\varsigma_{\mathrm{sh}}\mathbf G_{\mathbf z,h},
\qquad (\mathbf z,h)\in\mathbb R^D\times\mathcal H_J.
\end{equation}
By stationarity, the marginal covariance $\mathbf C_k:=\operatorname{Cov}(\mathbf G_{\mathbf z,h_k})$ is independent of $\mathbf z$. We define the worst-case variance factor $V:=\max_{k\in[J]}\lambda_{\max}(\mathbf C_k)$, so that $\mathbf C_k\preceq V\mathbf I_p$ for all $k\in[J]$. The following proposition provides the privacy guarantee and bounds the variance factor.

\begin{proposition}
\label{prop:shared-moment-field}
The function-valued release \eqref{eq:finite-moment-release} with $\varsigma_{\mathrm{sh}}=\sqrt2/(n\mu)$ is $(\varepsilon,\delta)$-DP. Moreover, $V\le J$. If $h_k=h_1q^{k-1}$ for some $q>1$, then $V\asymp\min\bigl(J,\,1+\sqrt D\log(h_J/h_1)\bigr)$. For $J=1$, this recovers Proposition~\ref{prop:shared-fixed-bandwidth}.
\end{proposition}

\begin{remark}
As shown in the proof of Proposition~\ref{prop:shared-moment-field}, on geometric grids (as in Section~\ref{sec:sim}), the equal-weight mixture $\bar\rho$ attains the optimal order of $V$ among admissible finite Gaussian spectral mixtures with arbitrary scales and weights. Whereas $V$ captures the worst-case variance, the full covariance $\mathbf C_k$ depends only on $D$ and $\mathcal H_J$ and admits a one-dimensional integral representation for comparison with $J\mathbf I_p$. For a fixed grid, as $D$ grows, spectral overlap vanishes and $\mathbf C_k$ approaches $J\mathbf I_p$, so independent releases across bandwidths provide comparable utility at lower computational cost while retaining spatial sharing. See Supplementary Appendix~\ref{app:shared-private-fields} for the proof and spectral-overlap analysis.
\end{remark}

Given the released moments, both corrections in Algorithm~\ref{alg:dp-manifold-denoising} apply at each $(\mathbf z_j,h_k)$, with the soft correction using noise scale $\varsigma_{\mathrm{sh}}$ and drawing auxiliary noise from $\mathcal N(\mathbf 0,\mathbf C_k)$ (or $\mathcal N(\mathbf 0,J\mathbf I_p)$ for per-bandwidth independent releases). The single-query accuracy guarantees then extend simultaneously to all $mJ$ query--bandwidth pairs, with effective noise factor $\sqrt V$ (respectively, $\sqrt J$) and an additive $\log(mJ)$ tail cost. See Supplementary Appendices~\ref{app:shared-release-details} and~\ref{app:shared-utility} for the release procedure and simultaneous utility guarantees.

\section{Simulations}
\label{sec:sim}

In this section, we assess the empirical performance of the proposed method. As shown in Theorem~\ref{thm:private-hard-correction} and Corollary~\ref{cor:adaptive-soft-correction}, the bandwidth $h$ balances local curvature bias against privacy-induced variance. Over the candidate grid $\mathcal H_J$, we select $h$ by minimizing the penalized criterion
\begin{equation}
\label{eq:bandwidth-score}
\mathcal L(h):=-\mathcal E(h)+\frac1m\sum_{j=1}^{m}\frac{V_j(h)}{h^2},
\end{equation}
{\emergencystretch=2em where $\mathcal E(h):=(mJ)^{-1}\sum_{j=1}^{m}\sum_{h'\in\mathcal H_J}\log(\tilde q_{\tilde{\mathbf z}_{\mathrm{soft},j}(h),h'}/\tilde q_{\mathbf z_j,h'})$, with $\tilde{\mathbf z}_{\mathrm{soft},j}(h)$ denoting the soft-corrected output for query $\mathbf z_j$ at bandwidth $h$. The term $\mathcal E(h)$ captures geometric fit via the gain in local mass across evaluation scales. Here $V_j(h)$ measures the empirical variance of $\tilde{\mathbf z}_{\mathrm{soft},j}(h)$ under auxiliary perturbations of the released moments matching the release noise, and $V_j(h)/h^2$ penalizes instability relative to the localization scale. Hard correction selects its bandwidth analogously. In practice, we stabilize released moments by convex projection before denoising; full details and algorithms appear in Supplementary Appendix~\ref{app:algorithms}.\par}
 
We consider four connected manifolds: a trefoil knot ($d=1$), a torus and a Swiss roll ($d=2$), and a bowl ($d=3$). We also include two linked tori ($d=2$) to assess performance on disconnected manifolds. Manifolds with $d\le 2$ are embedded in $\mathbb R^5$ and those with $d=3$ in $\mathbb R^{10}$. Reference observations follow \eqref{eq:model} with noise $\boldsymbol\epsilon_i\sim\mathrm{Unif}(B_D(0,\sigma))$. We sample $m$ points $\mathbf{z}_j^\star$ uniformly on the manifold, away from the boundary when present. From each $\mathbf{z}_j^\star$, we generate a public query at distance $a$ along a random unit normal vector.

To evaluate denoising accuracy, we compute the RMSE $(m^{-1}\sum_{j=1}^m \|\tilde{\mathbf z}_j-\mathbf z_j^\star\|^2)^{1/2}$ for each reference dataset and report the average over 100 independent datasets. We use the bandwidth grid $\mathcal H_9=\{0.025\cdot 2^{j/2}:j=0,\ldots,8\}$. Unless otherwise specified, we set $m=20$, $a=0.08$, $\sigma=0.04$, and $\delta=n^{-1.1}$. Default reference sample sizes are $n=2000$, $8000$, and $16000$ for $d=1$, $2$, and $3$, respectively, and $n=16000$ for the linked tori.

We compare the proposed private soft correction (DP-MD-S) with its non-private counterpart (NP-MD-S) and the uncorrected queries (RAW). To assess the cost of adaptivity, we also include private (DP-MD-H) and non-private (NP-MD-H) hard corrections with known $d$.

As a visual illustration, Figure~\ref{fig:reconstruction} shows reconstructions of 100 public queries on one trefoil dataset, where DP-MD-S and NP-MD-S reduce the latent RMSE from $0.08$ (RAW) to $0.038$ and $0.009$.

\begin{figure}[htbp]
\centering
\includegraphics[width=0.90\linewidth]{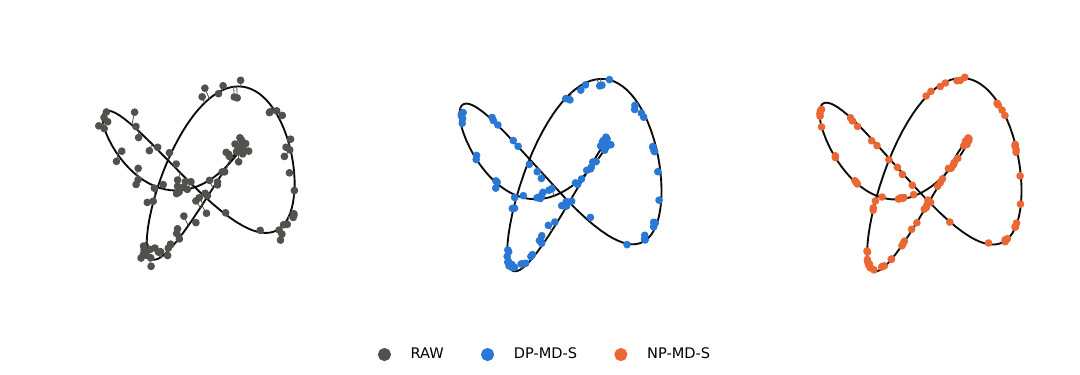}
\caption{Reconstructions of 100 public queries on a trefoil dataset under default settings with $\varepsilon=1$. Solid black curve indicates the true manifold, and thin segments link outputs to their latent targets $\mathbf z_j^\star$.}
\label{fig:reconstruction}
\end{figure}

For quantitative comparisons, we begin by considering the following settings:
\begin{itemize}
    \item [Case] 1: $\varepsilon\in\{1,2,4,8\}$ at the default $n$.
    \item [Case] 2: Varying $n$ across the four geometries, with $\varepsilon=1$.
\end{itemize}

\begin{figure*}[htbp]
\centering
\begin{minipage}{0.48\textwidth}
    \centering
    \includegraphics[width=\linewidth]{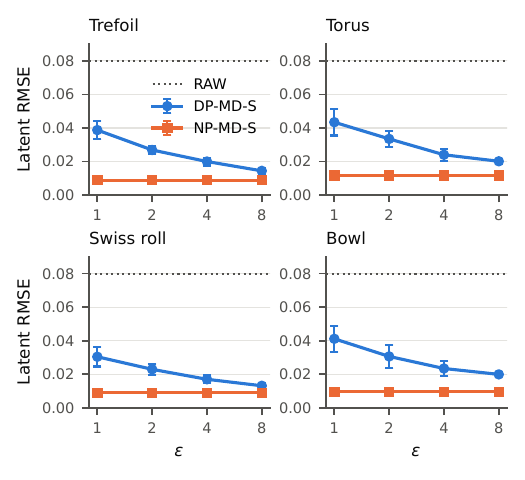}
    \centerline{\small (a) Case 1: Varying $\varepsilon$}
\end{minipage}\hfill
\begin{minipage}{0.48\textwidth}
    \centering
    \includegraphics[width=\linewidth]{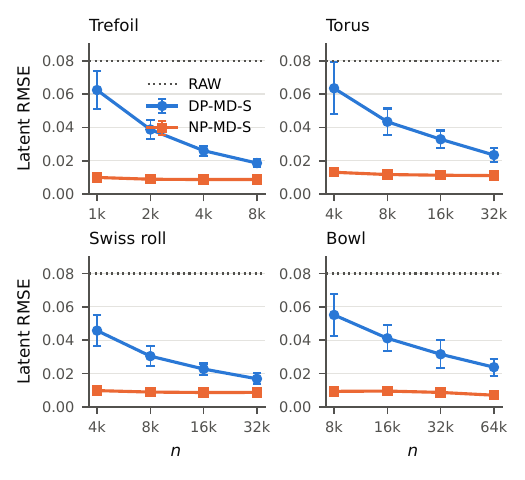}
    \centerline{\small (b) Case 2: Varying $n$}
\end{minipage}
\caption{Latent RMSE across four geometries.}
\label{fig:main-tradeoffs}
\end{figure*}

In Case~1, Figure~\ref{fig:main-tradeoffs}(a) shows that DP-MD-S substantially reduces latent RMSE across all geometries, and the error decreases monotonically as $\varepsilon$ increases. In Case~2, as shown in Figure~\ref{fig:main-tradeoffs}(b), the error drops steadily with $n$ and narrows the gap to NP-MD-S.

We next compare soft correction against hard correction with known $d$.
\begin{itemize}
    \item [Case] 3: Varying $a\in\{0.08, 0.10, 0.12, 0.14\}$ across geometries, with $\varepsilon=1$.
\end{itemize}

\begin{table}[htbp]
\centering
\small
\caption{Latent RMSE (standard deviation in parentheses) across varying $a$ under default settings with $\varepsilon=1$.}
\label{tab:hard-soft-a}
\begin{tabular}{lcccccc}
\toprule
Geometry & $a$ & RAW & NP-MD-H & NP-MD-S & DP-MD-H & DP-MD-S \\
\midrule
Trefoil $(1,5)$ & $0.08$ & $0.080$ & $0.007$ \scriptsize(0.001) & $0.009$ \scriptsize(0.001) & $0.036$ \scriptsize(0.006) & $0.039$ \scriptsize(0.005) \\
 & $0.10$ & $0.100$ & $0.008$ \scriptsize(0.001) & $0.010$ \scriptsize(0.001) & $0.040$ \scriptsize(0.005) & $0.044$ \scriptsize(0.006) \\
 & $0.12$ & $0.120$ & $0.009$ \scriptsize(0.001) & $0.011$ \scriptsize(0.001) & $0.048$ \scriptsize(0.007) & $0.053$ \scriptsize(0.007) \\
 & $0.14$ & $0.140$ & $0.009$ \scriptsize(0.001) & $0.012$ \scriptsize(0.002) & $0.060$ \scriptsize(0.008) & $0.062$ \scriptsize(0.006) \\
\midrule
Torus $(2,5)$ & $0.08$ & $0.080$ & $0.010$ \scriptsize(0.001) & $0.012$ \scriptsize(0.001) & $0.042$ \scriptsize(0.008) & $0.043$ \scriptsize(0.008) \\
 & $0.10$ & $0.100$ & $0.011$ \scriptsize(0.001) & $0.013$ \scriptsize(0.001) & $0.044$ \scriptsize(0.007) & $0.048$ \scriptsize(0.009) \\
 & $0.12$ & $0.120$ & $0.012$ \scriptsize(0.001) & $0.015$ \scriptsize(0.001) & $0.054$ \scriptsize(0.012) & $0.055$ \scriptsize(0.012) \\
 & $0.14$ & $0.140$ & $0.013$ \scriptsize(0.001) & $0.016$ \scriptsize(0.001) & $0.067$ \scriptsize(0.013) & $0.066$ \scriptsize(0.012) \\
\midrule
Swiss roll $(2,5)$ & $0.08$ & $0.080$ & $0.007$ \scriptsize(0.001) & $0.009$ \scriptsize(0.001) & $0.028$ \scriptsize(0.006) & $0.030$ \scriptsize(0.006) \\
 & $0.10$ & $0.100$ & $0.008$ \scriptsize(0.001) & $0.010$ \scriptsize(0.001) & $0.032$ \scriptsize(0.008) & $0.035$ \scriptsize(0.007) \\
 & $0.12$ & $0.120$ & $0.009$ \scriptsize(0.001) & $0.012$ \scriptsize(0.001) & $0.039$ \scriptsize(0.009) & $0.042$ \scriptsize(0.008) \\
 & $0.14$ & $0.140$ & $0.010$ \scriptsize(0.001) & $0.013$ \scriptsize(0.001) & $0.050$ \scriptsize(0.010) & $0.050$ \scriptsize(0.009) \\
\midrule
Bowl $(3,10)$ & $0.08$ & $0.080$ & $0.008$ \scriptsize(0.002) & $0.010$ \scriptsize(0.002) & $0.041$ \scriptsize(0.009) & $0.041$ \scriptsize(0.008) \\
 & $0.10$ & $0.100$ & $0.008$ \scriptsize(0.000) & $0.008$ \scriptsize(0.000) & $0.042$ \scriptsize(0.006) & $0.043$ \scriptsize(0.009) \\
 & $0.12$ & $0.120$ & $0.008$ \scriptsize(0.000) & $0.009$ \scriptsize(0.001) & $0.044$ \scriptsize(0.009) & $0.045$ \scriptsize(0.009) \\
 & $0.14$ & $0.140$ & $0.009$ \scriptsize(0.000) & $0.010$ \scriptsize(0.001) & $0.045$ \scriptsize(0.009) & $0.048$ \scriptsize(0.010) \\
\bottomrule
\end{tabular}

\end{table}

In Table~\ref{tab:hard-soft-a}, the errors of both DP-MD-S and DP-MD-H grow slowly as $a$ increases. Across all settings, DP-MD-S stays close to DP-MD-H and in some cases achieves slightly lower error. Continuous spectral filtering works well without knowing $d$, even for queries farther from the manifold.

\begin{itemize}
    \item [Case] 4: Joint shared release versus independent releases (per-scale and per-query) at $\varepsilon=1$.
\end{itemize}

\begin{figure}[htbp]
\centering
\includegraphics[width=0.65\linewidth]{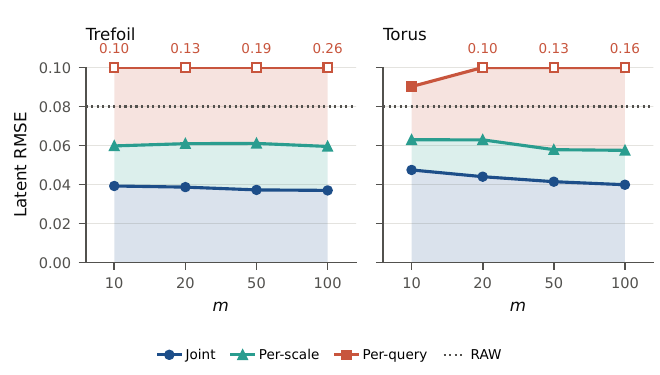}
\caption{Latent RMSE of DP-MD-S under joint, per-scale, and per-query releases at $\varepsilon=1$.}
\label{fig:sim-release}
\end{figure}

In Case~4, Figure~\ref{fig:sim-release} compares joint, per-scale, and per-query releases as $m$ varies. As expected, the joint release stays nearly flat in $m$ and achieves the lowest error on both geometries. By contrast, per-scale releases incur higher error throughout, while per-query error grows rapidly with $m$ and exceeds RAW once $m$ is moderate.

\begin{itemize}
    \item [Case] 5: On the two linked tori, varying $\varepsilon$ at the default $n$ and $n$ at $\varepsilon=2$.
\end{itemize}

\begin{figure}[htbp]
\centering
\includegraphics[width=0.65\linewidth]{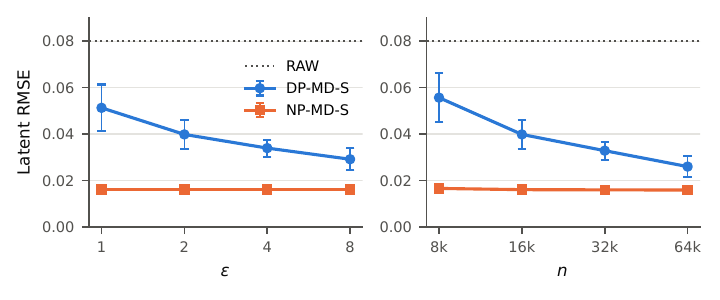}
\caption{Latent RMSE on disconnected manifolds.}
\label{fig:disconnected}
\end{figure}

In Case~5, DP-MD-S behaves much as on connected manifolds, as shown in Figure~\ref{fig:disconnected}. Every corrected output remains closer to its own component, with no cross-component misassignment observed. See Supplementary Figure~\ref{fig:app-disconnected} for an illustration.

\begin{itemize}
    \item [Case] 6: Varying ambient dimension $D$ for the trefoil knot ($n=8000$) and the bowl ($n=16000$; Supplementary Appendix~\ref{app:sim-settings}), with $\varepsilon=2$.
\end{itemize}

\begin{figure}[htbp]
\centering
\includegraphics[width=0.65\linewidth]{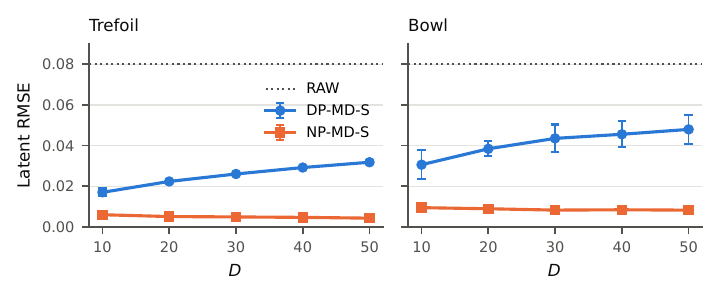}
\caption{Latent RMSE across ambient dimensions.}
\label{fig:highd}
\end{figure}

In Case~6, Figure~\ref{fig:highd} shows that the error of DP-MD-S grows with $D$, as expected from the privacy noise in $\mathbb R^D$, whereas NP-MD-S stays nearly flat since it depends primarily on $d$. Still, DP-MD-S maintains substantial error reduction over RAW across all evaluated dimensions.  

\begin{itemize}
    \item [Case] 7: Comparison with global DP-PCA at $\varepsilon=1$ across four geometries, including the Fourier torus (Supplementary Appendix~\ref{app:sim-settings}).
\end{itemize}

\begin{table}[htbp]
\centering
\small
\caption{Latent RMSE (standard deviation in parentheses) comparing DP-MD-S with NP-MD-S and global DP-PCA under default settings with $\varepsilon=1$, and $n=16000$ for the Fourier torus.}
\label{tab:sim-dp-pca}
\begin{tabular}{lccc}
\toprule
Geometry & NP-MD-S & DP-MD-S & DP-PCA \\
\midrule
Trefoil $(1,5)$ & $0.009$ (0.001) & $0.039$ (0.005) & $0.125$ (0.086) \\
Swiss roll $(2,5)$ & $0.009$ (0.001) & $0.030$ (0.006) & $0.046$ (0.005) \\
Bowl $(3,10)$ & $0.010$ (0.002) & $0.041$ (0.008) & $0.095$ (0.014) \\
Fourier torus $(2,10)$ & $0.010$ (0.001) & $0.044$ (0.007) & $0.166$ (0.002) \\
\bottomrule
\end{tabular}
\end{table}

In Case~7, we compare DP-MD-S with global DP-PCA to assess the benefit of local geometric correction over global linear projection. Details of DP-PCA are given in Supplementary Appendix~\ref{app:sim-baselines}. As shown in Table~\ref{tab:sim-dp-pca}, DP-MD-S achieves lower error than DP-PCA on all four geometries. The gap is largest on the Fourier torus, whose nonlinear embedding spans the full ambient space and cannot be captured by a global linear subspace. We also evaluated subsample-and-aggregate versions of two manifold-fitting methods \citep{yao2023manifold,yao2025manifold}. Those comparisons are deferred to Supplementary Appendix~\ref{app:sim-baselines}, as they incur substantially larger errors.

\section{Applications}
\label{sec:application}

\label{sec:paper-ukb}

We evaluate the proposed method in UK Biobank \citep{Sudlow2015UKBiobank}, using a private reference cohort to denoise held-out participants via local geometry. Since the intrinsic dimension is unknown, we apply DP-MD-S with private preprocessing for end-to-end $(\varepsilon,\delta)$-DP and compare it against NP-MD-S and private global baselines. As noted in Section~\ref{sec:shared_private_field}, we share the private field across query locations and release each bandwidth independently for computational efficiency. All private methods use $(\varepsilon,\delta)=(2,n^{-1.1})$ per repetition, splitting the GDP budget $\mu^2$ as 20\% for preprocessing and 80\% for the main release.

We consider two complementary applications: accelerometry uses repeated recording days to assess denoising under natural day-to-day variability, whereas spirometry uses forced-expiration curves under controlled perturbations to evaluate reconstruction and downstream disease associations.

\subsection{Accelerometry}
\label{sec:paper-ukb-acc}

In its accelerometry substudy, UK Biobank collected week-long wrist-accelerometer recordings to measure physical activity \citep{Doherty2017UKBAccelerometry}. Although commonly used as proxies for habitual behavior, short-term summaries suffer from day-to-day within-person variability that can potentially attenuate associations with health outcomes \citep{zisou2026reproducibility}. For each day, we use $D=28$ coupled features capturing intensity, duration, and timing, and test whether local geometric correction reduces day-to-day noise while preserving individual differences.

Since there is no ground truth for habitual profiles, we evaluate denoising via repeat-day agreement. In each of $20$ repetitions, we draw $m=100$ held-out participants, average their odd and even days separately, and use a private reference of $n=20000$ to correct the odd-day profiles and compare them with the even-day measurements. We report the mean Euclidean distance between the corrected odd-day and even-day profiles (discrepancy), and the proportion of participants whose discrepancy decreases after correction. We also report the proportion of between-participant variance retained, since discrepancy can also be lowered simply by pulling profiles closer together.

To assess the privacy cost and the benefit of local geometry, we compare DP-MD-S with NP-MD-S, DP mean shrinkage toward a private mean, and global DP-PCA. For DP mean shrinkage, we test an empirical-Bayes James--Stein weight calibrated to odd-day variability and a scalar weight tuned on $200$ validation participants. For global DP-PCA, we include the spectrum-based rank rule (see Supplementary Appendix~\ref{app:sim-dp-baseline}) and a rank tuned on the same validation set. DP-MD-S requires no validation data. Further implementation details and privacy accounting are given in Supplementary Appendix~\ref{app:paper-ukb-acc}.

\begin{table}[t]
\centering\small
\setlength{\tabcolsep}{3pt}
\caption{UK Biobank accelerometry. Mean (standard deviation) in percent over $20$ repetitions. Distance reduction and variance retained are relative to the uncorrected profiles, and participants improved is the proportion with lower distance.}
\label{tab:paper-ukb-main}
\begin{tabular}{@{}lccc@{}}
\toprule
Method & Distance reduction & Variance retained & Participants improved \\
\midrule
NP-MD-S & 6.17 (0.76) & 81.66 (1.57) & 78.7 (3.6) \\
DP-MD-S & 4.81 (1.13) & 72.88 (3.74) & 67.5 (3.6) \\
DP mean shrinkage & 4.24 (2.40) & 62.67 (3.90) & 63.3 (4.4) \\
DP mean shrinkage (val.) & 6.77 (1.61) & 51.84 (0.00) & 66.7 (4.7) \\
DP-PCA & $-9.62$ (3.32) & 56.82 (3.80) & 39.7 (4.6) \\
DP-PCA (val.) & 1.47 (0.61) & 94.46 (1.00) & 65.7 (4.6) \\
\bottomrule
\end{tabular}
\end{table}

Table~\ref{tab:paper-ukb-main} summarizes the results over $20$ repetitions. DP-MD-S reduces discrepancy by 4.81\%, capturing about 78\% of the NP-MD-S gain while retaining 72.88\% of between-participant variance. It matches James--Stein mean shrinkage in discrepancy reduction while preserving more individual variation, whereas validation-tuned mean shrinkage lowers discrepancy further at the cost of retaining only 51.84\% of the variance. Global DP-PCA either worsens agreement under the spectrum-based rank rule or yields much smaller gains when validation-tuned. Overall, local geometric correction balances repeat-day noise reduction and individual variation more effectively than global baselines.

\subsection{Spirometry}
\label{sec:paper-ukb-spiro}

In spirometry, UK Biobank recorded forced-expiration volume--time curves to measure lung function. We represent each curve by $D=10$ features: cumulative exhaled volume at nine time points and peak expiratory flow. Since these coordinates trace expiration trajectories that vary in shape across individuals, we test whether local geometric correction reconstructs perturbed curves while preserving clinically informative morphology.

We perturb each coordinate with independent scale-proportional Gaussian noise at relative standard deviation $\sigma_{\mathrm{rel}}=0.2$. In each of $20$ repetitions, a private reference of $n=20000$ corrects $m=500$ perturbed queries forming $250$ matched case--control pairs for downstream chronic obstructive pulmonary disease (COPD) analysis. Reconstruction accuracy is measured by percentage Euclidean error reduction relative to the perturbed curves. We compare DP-MD-S with NP-MD-S and global DP-PCA, with rank chosen by the spectrum-based rule or on $200$ validation curves as in Section~\ref{sec:paper-ukb-acc}.

Table~\ref{tab:paper-ukb-spiro-recon} summarizes the results over $20$ repetitions. DP-MD-S reduces reconstruction error by 54.1\%, retaining about 96\% of the NP-MD-S reduction and outperforming DP-PCA under both rank rules.

\begin{table}[t]
\centering\small
\setlength{\tabcolsep}{3pt}
\caption{UK Biobank spirometry. Mean (standard deviation) over $20$ repetitions. Error reduction is relative to perturbed curves, while the regression coefficient ratio and misclassification rate are relative to the recorded curves.}
\label{tab:paper-ukb-spiro-recon}
\begin{tabular}{@{}lccc@{}}
\toprule
Method & Error reduction (\%) & Coefficient ratio & Misclassified (\%) \\
\midrule
Perturbed & --- & 0.15 (0.05) & 36.2 (2.5) \\
NP-MD-S & 56.6 (0.7) & 0.83 (0.15) & 23.7 (1.6) \\
DP-MD-S & 54.1 (0.9) & 0.82 (0.15) & 26.8 (1.6) \\
DP-PCA & 43.4 (0.9) & 5.08 (1.51) & 41.1 (2.4) \\
DP-PCA (val.) & 51.0 (0.8) & 0.76 (0.16) & 27.8 (2.0) \\
\bottomrule
\end{tabular}
\end{table}

We next evaluate whether improved curve reconstruction recovers downstream clinical associations. From each curve, we compute the ratio of 1-second to 6-second exhaled volume ($\mathrm{FEV}_1/\mathrm{FEV}_6$), a standard airflow-limitation measure linked to incident COPD \citep{vandevoorde2006fev6,zou2026copd}. On the $250$ matched pairs, we fit conditional logistic regression of incident COPD on standardized $\mathrm{FEV}_1/\mathrm{FEV}_6$ and report the estimated log odds ratio relative to its unperturbed value. We also classify airflow limitation via $\mathrm{FEV}_1/\mathrm{FEV}_6 < 0.73$ \citep{vandevoorde2006fev6} and report the misclassification rate against unperturbed labels.

In Table~\ref{tab:paper-ukb-spiro-recon}, noise attenuates the coefficient ratio to $0.15$, whereas DP-MD-S recovers $0.82$, closely matching NP-MD-S. By contrast, validation-tuned DP-PCA leaves the estimate more attenuated, while spectrum-based DP-PCA retains one principal component and inflates the log-odds coefficient roughly fivefold. For threshold classification, DP-MD-S lowers misclassification from $36.2\%$ to $26.8\%$, outperforming DP-PCA under both rank rules. Overall, local geometric correction reconstructs expiratory curves more accurately than global DP-PCA and better preserves downstream clinical associations. Further details are given in Supplementary Appendix~\ref{app:paper-ukb-spiro}.
\section{Discussion}
\label{sec:discussion}
We developed a framework for differentially private local geometric learning combining kernel smoothing with manifold geometry. Gaussian kernel decay yields a moment sensitivity bound independent of ambient dimension, bandwidth, and query location, while local tangency and spectral separation enable hard and noise-adaptive soft corrections to recover latent structure. Exploiting spatial and multi-scale kernel correlations, a shared Gaussian field extends the mechanism across locations and bandwidths with a moderate increase in privacy noise. Our approach provides a practical way to use the latent structure of sensitive cohorts, such as genomic or medical records, to calibrate and denoise external queries, enabling cross-institutional collaboration under formal privacy guarantees.

Several theoretical and practical challenges remain. First, our utility bounds retain an $O(\sigma)$ term matching the reference noise level. This reflects our bounded-error assumption $\|\boldsymbol\epsilon_i\|\le\sigma$, which accommodates general, potentially biased perturbations (see Remark~\ref{rem:tangent-properties}). With more structured noise assumptions, such as zero-mean isotropic errors, local averaging can lead to higher-order noise cancellation \citep{yao2023manifold}. A related direction is extending the analysis to unbounded sub-Gaussian or heavy-tailed noise.

Second, bandwidth selection becomes costly for large $m$ and $D$, as conditional sampling from the shared Gaussian field involves a large joint covariance matrix. Developing fast approximations with formal privacy guarantees is an important direction. Third, since the shared private field provides a continuous function-valued release, turning it into an explicit global manifold reconstruction with theoretical guarantees is a natural next step. Finally, we currently treat queries as public; when queries are sensitive, two-sided privacy mechanisms are needed to protect both the reference cohort and the evaluated individuals.

\subsection*{Declaration of generative AI use}
The authors used OpenAI Codex (GPT-6 family) and Anthropic Claude Code (Opus 5.5) for language polishing, coding assistance, bibliographic checking, \LaTeX{} formatting, and reviewing and refining mathematical proofs for clarity. The authors verified the accuracy and originality of the manuscript and code, and take full responsibility for all content, including references.

\subsection*{Data, Materials, and Software Availability}
Code implementing the proposed method and reproducing the synthetic experiments is available at {\begingroup\urlstyle{same}\url{https://github.com/zhigang-yao/DP-Manifold-Denoising}\endgroup}. UK Biobank data are available to approved researchers through the UK Biobank Access Management System ({\begingroup\urlstyle{same}\url{https://www.ukbiobank.ac.uk}\endgroup}). This study was conducted under UK Biobank application 146760.

\bibliographystyle{apalike}
\bibliography{reference}

\newpage
\appendix
\setcounter{figure}{0}
\setcounter{table}{0}
\renewcommand{\theHfigure}{appendix.\arabic{figure}}
\renewcommand{\theHtable}{appendix.\arabic{table}}
\renewcommand{\thefigure}{S.\arabic{figure}}
\renewcommand{\thetable}{S.\arabic{table}}

\begin{center}
    {\Large\bfseries Supplementary Appendix for}\\[0.6em]
    {\LARGE\bfseries Differentially Private Manifold Denoising}
\end{center}
\vspace{1.5em}

\begingroup
\renewcommand{\contentsname}{Table of Contents}
\setcounter{tocdepth}{2}
\tableofcontents
\endgroup
\addtocontents{toc}{\protect\setcounter{tocdepth}{2}}
\newpage

The supplement is organized as follows:
\begin{itemize}[leftmargin=2em,itemsep=0.3em]
    \item \textbf{Appendix~\ref{app:algorithms}}: Algorithmic details for the shared-field release, conditional evaluation, and bandwidth selection for the soft, hard, and non-private estimators.
    \item \textbf{Appendix~\ref{app:geometric-probabilistic-preliminaries}}: Geometric properties of manifolds with positive reach and concentration bounds for empirical local moments and Gaussian noise.
    \item \textbf{Appendix~\ref{app:local-estimation-proofs}}: Proofs of private moment sensitivity bounds, tangent-space estimation bounds, and accuracy guarantees for both hard and soft corrections, including boundary extensions.
    \item \textbf{Appendix~\ref{app:shared-private-fields}}: Spectral representation, privacy guarantees, radial covariance, and simultaneous utility bounds for the shared Gaussian field.
    \item \textbf{Appendix~\ref{app:additional_simulation}}: Simulation settings and additional results.
    \item \textbf{Appendix~\ref{app:paper-ukb}}: Data, private preprocessing, privacy accounting, and evaluation details for the UK Biobank applications.
\end{itemize}

\paragraph{Additional notation.}
We use $\lesssim$, $\gtrsim$, and $\asymp$ to denote upper, lower, and two-sided bounds up to positive constant factors, respectively. For a finite set $A$, let $|A|$ denote its cardinality. We write $\mathbb I(\cdot)$ for the indicator function and $a\wedge b=\min(a,b)$.

The norm $\|\cdot\|$ denotes the Euclidean norm for vectors and the spectral norm for matrices; $\|\cdot\|_{\mathrm F}$ denotes the Frobenius norm. For a set $A\subseteq\mathbb R^D$, let $\operatorname{dist}(\mathbf u,A):=\inf_{\mathbf a\in A}\|\mathbf u-\mathbf a\|$. Let $\mathbb S^D$ denote the space of symmetric $D\times D$ matrices, with $\lambda_j(\mathbf A)$ denoting the $j$th largest eigenvalue of $\mathbf A\in\mathbb S^D$. For a differentiable map $f$, we write $\mathrm D f(\mathbf x)$ for its derivative at $\mathbf x$.

We write $\mathcal H^d$ for the $d$-dimensional Hausdorff measure and $N_{\mathbf x}\mathcal M$ for the normal space at $\mathbf x\in\mathcal M$, with orthogonal projector $\boldsymbol\Pi_{\mathbf x}^{\perp}:=\mathbf I-\boldsymbol\Pi_{\mathbf x}$. For a fixed query $\mathbf z$ with $\mathbf z^\star=\pi_{\mathcal M}(\mathbf z)$, we abbreviate $\boldsymbol\Pi_\star:=\boldsymbol\Pi_{\mathbf z^\star}$ and $\boldsymbol\Pi_\star^\perp:=\boldsymbol\Pi_{\mathbf z^\star}^\perp$.

\section{Algorithmic and implementation details}
\label{app:algorithms}

\subsection{Shared-field release and conditional evaluation}
\label{app:shared-release-details}

On $\mathcal H_J$, we consider either the coupled field in \eqref{eq:shared-scale-covariance} or independent fixed-bandwidth fields with
\begin{equation}
\label{eq:independent-scale-covariance}
\begin{aligned}
\mathbb E[G^{\mathrm{ind}}(\mathbf z,h_k)G^{\mathrm{ind}}(\mathbf z',h_r)]
&=J\mathbb I(k=r)k_{h_k}(\mathbf z-\mathbf z'),\\
[\mathbf K^{\mathrm{ind}}]_{(j,k),(\ell,r)}
&=J\mathbb I(k=r)[\mathbf K_{h_k}]_{j\ell}.
\end{aligned}
\end{equation}
Let $\mathbf K$ denote the joint moment covariance for the chosen mode. Algorithm~\ref{alg:shared-private-moments} summarizes the initial release.

\begin{algorithm}[htbp]
\caption{\textsc{Shared-Private-Moments}}
\label{alg:shared-private-moments}
\begin{algorithmic}[1]
\Require Private reference data $\mathcal D$, public queries $\mathbf z_1,\dots,\mathbf z_m$, bandwidth grid $\mathcal H_J$, privacy parameters $(\varepsilon,\delta)$, mode $\in\{\text{shared},\text{independent}\}$
\State Compute $m_{0,\mathbf z_j,h_k}$, $\mathbf m_{1,\mathbf z_j,h_k}$, and $\mathbf M_{2,\mathbf z_j,h_k}$ for each $(j,k)\in[m]\times[J]$
\State Form $\mathbf K$ from Section~\ref{subsec:bandwidth-sharing} for the shared mode and \eqref{eq:independent-scale-covariance} for the independent mode
\State Add Gaussian noise
\[
\begin{aligned}
\tilde m_{0,\mathbf z_j,h_k}
  &:=m_{0,\mathbf z_j,h_k}+\xi_{0,jk},\\
\tilde{\mathbf m}_{1,\mathbf z_j,h_k}
  &:=\mathbf m_{1,\mathbf z_j,h_k}+h_k\boldsymbol\xi_{1,jk},\\
\widetilde{\mathbf M}_{2,\mathbf z_j,h_k}
  &:=\mathbf M_{2,\mathbf z_j,h_k}
    +\sqrt2h_k^2\operatorname{smat}(\boldsymbol\xi_{2,jk}),
\end{aligned}
\]
\Statex where  $\{(\xi_{0,jk},\boldsymbol\xi_{1,jk},\boldsymbol\xi_{2,jk})\}_{(j,k)\in[m]\times[J]}\sim\mathcal N(\mathbf 0,\varsigma_{\mathrm{sh}}^2\mathbf K)$, and $\varsigma_{\mathrm{sh}}=\sqrt2/(n\mu)$ with $\mu$ satisfying \eqref{eq:gaussian-calibration}
\State Set
\[\widetilde{\mathbf F}_{\mathbf z_j,h_k}:=(\tilde m_{0,\mathbf z_j,h_k},\tilde{\mathbf m}_{1,\mathbf z_j,h_k}/h_k,\operatorname{svec}(\widetilde{\mathbf M}_{2,\mathbf z_j,h_k})/(\sqrt{2}h_k^2)),\] 
\Statex and stack the blocks over $(j,k)\in[m]\times[J]$ into $\widetilde{\mathbf F}$
\State \Return $\widetilde{\mathbf F}$ and $\{(\tilde m_{0,\mathbf z_j,h_k},\tilde{\mathbf m}_{1,\mathbf z_j,h_k},\widetilde{\mathbf M}_{2,\mathbf z_j,h_k})\}_{(j,k)\in[m]\times[J]}$
\end{algorithmic}
\end{algorithm}

\paragraph{Conditional evaluation for bandwidth selection.}
{\emergencystretch=2em Evaluating \eqref{eq:bandwidth-score} requires noisy kernel masses $\tilde m_{0,\cdot,h_r}$ at the candidate outputs. Let $\mathcal A$ index the initial release $\widetilde{\mathbf F}$, $\mathcal B$ index $\{(\tilde{\mathbf z}_{\mathrm{soft},j}(h_k),h_r):(j,k,r)\in[m]\times[J]^2\}$, and $U$ collect the auxiliary randomness used to form the candidate points. Conditional on $\widetilde{\mathbf F}_{\mathcal A}$ and $U$,
\begin{equation}
\label{eq:shared-conditional-mass}
\begin{aligned}
\widetilde{\mathbf F}_{\mathcal B}\,\big|\,\widetilde{\mathbf F}_{\mathcal A},U
&\sim\mathcal N\!\bigl(
\mathbf m_{\mathcal B\mid\mathcal A},\,
\varsigma_{\mathrm{sh}}^2\mathbf K_{\mathcal B\mid\mathcal A}
\bigr),\\
\mathbf m_{\mathcal B\mid\mathcal A}
&:=\mathbf F_{\mathcal B}(\mathcal D)
+[\mathbf K]_{\mathcal B\mathcal A}[\mathbf K]_{\mathcal A\mathcal A}^{-1}
\bigl(\widetilde{\mathbf F}_{\mathcal A}-\mathbf F_{\mathcal A}(\mathcal D)\bigr),\\
\mathbf K_{\mathcal B\mid\mathcal A}
&:=[\mathbf K]_{\mathcal B\mathcal B}
-[\mathbf K]_{\mathcal B\mathcal A}[\mathbf K]_{\mathcal A\mathcal A}^{-1}
[\mathbf K]_{\mathcal A\mathcal B},
\end{aligned}
\end{equation}}
where $\mathbf F_{\mathcal A}(\mathcal D)$ stacks the empirical moment vectors $\mathbf F_{\mathbf z_j,h_k}(\mathcal D)$ from \eqref{eq:finite-moment-release} over $\mathcal A$, and $\mathbf F_{\mathcal B}(\mathcal D)$ collects the empirical kernel masses $m_{0,\cdot,h_r}$ over $\mathcal B$.

\subsection{Bandwidth selection procedure}
\label{app:bandwidth-selection-details}

\paragraph{Moment stabilization.}
Given a normalized released block $\mathbf u=(u_0,\mathbf u_1,\mathbf u_2)$ at query $\mathbf z$ and bandwidth $h_k$, we project onto the feasible moment set by solving
\begin{equation}
\label{eq:moment-stabilization}
\Bigl(\tilde m_{0,\mathbf z},\,\frac{\tilde{\mathbf m}_{1,\mathbf z}}{h_k},\,\frac{\widetilde{\mathbf M}_{2,\mathbf z}}{h_k^2}\Bigr)
:=\operatorname*{argmin}_{\substack{0\le q\le1,\\ \left(\begin{smallmatrix}q&\mathbf v^\top\\\mathbf v&\mathbf A\end{smallmatrix}\right)\succeq\mathbf 0}}
\Bigl((q-u_0)^2+\|\mathbf v-\mathbf u_1\|^2+\tfrac12\bigl\|\mathbf A-\sqrt2\operatorname{smat}(\mathbf u_2)\bigr\|_{\mathrm F}^2\Bigr),
\end{equation}
and, as in Algorithm~\ref{alg:dp-manifold-denoising}, set
\begin{equation}
\label{eq:stabilized-summaries}
\tilde q_{\mathbf z}:=\tilde m_{0,\mathbf z}\vee(\eta/n),\qquad
\tilde{\mathbf r}_{\mathbf z}:=\frac{\tilde{\mathbf m}_{1,\mathbf z}}{\tilde q_{\mathbf z}},\qquad
\widetilde{\boldsymbol\Sigma}_{\mathbf z}:=\frac{\widetilde{\mathbf M}_{2,\mathbf z}}{\tilde q_{\mathbf z}}-\tilde{\mathbf r}_{\mathbf z}\tilde{\mathbf r}_{\mathbf z}^\top.
\end{equation}
Since $\tilde q_{\mathbf z}\ge\tilde m_{0,\mathbf z}$ and $\tilde q_{\mathbf z}>0$, the Schur complement ensures $\widetilde{\boldsymbol\Sigma}_{\mathbf z}\succeq\mathbf 0$.

\paragraph{Soft correction from a released block.}
Let $\textsc{Soft-From-Moments}(\mathbf u;\mathbf z,h_k)$ apply the soft-correction branch of Algorithm~\ref{alg:dp-manifold-denoising} to $(\tilde q_{\mathbf z},\tilde{\mathbf r}_{\mathbf z},\widetilde{\boldsymbol\Sigma}_{\mathbf z})$ from \eqref{eq:moment-stabilization}--\eqref{eq:stabilized-summaries}, with noise scale $\varsigma_{\mathrm{sh}}$ in \eqref{eq:soft-noise-scale} and auxiliary draws from $\mathcal N(\mathbf 0,\mathbf C_k)$ as in Section~\ref{subsec:bandwidth-sharing}. For per-bandwidth independent releases, replace $\mathbf C_k$ by $J\mathbf I_p$.

\paragraph{Output stability and selection.}
For $R\ge2$, the empirical output variance in Algorithm~\ref{alg:private-bandwidth-selection} is given by
\begin{equation}
\label{eq:auxiliary-output-variance}
\overline{\mathbf z}_j^{\mathrm{aux}}(h_k)
:=\frac1R\sum_{\ell=1}^R\tilde{\mathbf z}_{\mathrm{soft},j}^{(\ell)}(h_k),\qquad
V_j(h_k):=\frac1{R-1}\sum_{\ell=1}^R
\bigl\|\tilde{\mathbf z}_{\mathrm{soft},j}^{(\ell)}(h_k)-\overline{\mathbf z}_j^{\mathrm{aux}}(h_k)\bigr\|^2.
\end{equation}
Algorithm~\ref{alg:private-bandwidth-selection} summarizes the selection procedure.

\begin{algorithm}[htbp]
\caption{\textsc{Private-Bandwidth-Selection}}
\label{alg:private-bandwidth-selection}
\begin{algorithmic}[1]
\Require Private reference data $\mathcal D$, public queries $\mathbf z_1,\dots,\mathbf z_m$, candidate grid $\mathcal H_J=\{h_1,\dots,h_J\}$, parameters $(\eta,B,\alpha_0,R)$, noise scale $\varsigma_{\mathrm{sh}}$, release $\widetilde{\mathbf F}$ from Algorithm~\ref{alg:shared-private-moments}, and marginal covariances $\mathbf C_1,\dots,\mathbf C_J$
\For{$j\in[m]$ and $k\in[J]$}
  \State $\tilde{\mathbf z}_{\mathrm{soft},j}(h_k)\leftarrow\textsc{Soft-From-Moments}(\widetilde{\mathbf F}_{\mathbf z_j,h_k};\mathbf z_j,h_k)$
  \State Draw independent $\boldsymbol\xi_{j,k}^{(\ell)}\sim\mathcal N(\mathbf 0,\varsigma_{\mathrm{sh}}^2\mathbf C_k)$ and compute
  \[
  \tilde{\mathbf z}_{\mathrm{soft},j}^{(\ell)}(h_k)
  :=\textsc{Soft-From-Moments}(\widetilde{\mathbf F}_{\mathbf z_j,h_k}+\boldsymbol\xi_{j,k}^{(\ell)};\mathbf z_j,h_k),
  \qquad \ell\in[R]
  \]
  \State Compute $V_j(h_k)$ via \eqref{eq:auxiliary-output-variance}
\EndFor
\State Draw the conditional mass evaluations $\{\tilde m_{0,\tilde{\mathbf z}_{\mathrm{soft},j}(h_k),h_r}\}$ for $(j,k,r)\in[m]\times[J]^2$ via \eqref{eq:shared-conditional-mass}
\State Compute $\mathcal L(h_k)$ for each $k\in[J]$ via \eqref{eq:bandwidth-score}
\State Set $\widehat h=\min\operatorname*{argmin}_{h\in\mathcal H_J}\mathcal L(h)$
\State \Return $\widehat h$ and $\{\tilde{\mathbf z}_{\mathrm{soft},j}(\widehat h)\}_{j=1}^m$
\end{algorithmic}
\end{algorithm}

\subsection{Hard-correction bandwidth selection}
\label{app:hard-bandwidth-selection}

For the private hard correction, given the intrinsic dimension $d$, we modify Algorithm~\ref{alg:private-bandwidth-selection} by replacing $\textsc{Soft-From-Moments}$ with the hard-correction branch of Algorithm~\ref{alg:dp-manifold-denoising}. We apply it to the stabilized summaries $(\tilde q_{\mathbf z},\tilde{\mathbf r}_{\mathbf z},\widetilde{\boldsymbol\Sigma}_{\mathbf z})$ from \eqref{eq:moment-stabilization}--\eqref{eq:stabilized-summaries}. The output variance $V_j(h_k)$ in \eqref{eq:auxiliary-output-variance}, conditional mass evaluations in \eqref{eq:shared-conditional-mass}, and score \eqref{eq:bandwidth-score} are then computed from the resulting hard-corrected outputs.

\subsection{Non-private counterparts}
\label{app:nonprivate-counterparts}

For the non-private counterparts, we use the empirical moments from $\mathcal D$. We apply the hard and soft correction rules of Algorithm~\ref{alg:dp-manifold-denoising} (with $\nu_{\mathbf z}=0$ in the soft branch) over the grid $\mathcal H_J$. For bandwidth selection, kernel masses are evaluated directly on $\mathcal D$ and the bandwidth is chosen by maximizing $\mathcal E(h)$ in \eqref{eq:bandwidth-score}.

\section{Geometric and concentration lemmas}
\label{app:geometric-probabilistic-preliminaries}

\subsection{Local manifold geometry} 

\begin{lemma}
\label{lemma:reach}
$\mathrm{reach}(\mathcal{M})\ge \tau$ if and only if
\[
\|\boldsymbol\Pi_{\mathbf{x}}^{\perp}(\mathbf{y}-\mathbf{x})\|
\le\frac{1}{2\tau}\,\|\mathbf{y}-\mathbf{x}\|^2,
\qquad \forall\,\mathbf{x},\mathbf{y}\in\mathcal{M}.
\]
\end{lemma}

\begin{proof}
See for example Theorem 4.18 in \cite{federer1959curvature}.
\end{proof}

\begin{lemma}
\label{lem:projection-lipschitz}
For $0<r<\tau$ and $\mathbf u,\mathbf v\in\mathbb R^D$ satisfying $\operatorname{dist}(\mathbf u,\mathcal M)\vee \operatorname{dist}(\mathbf v,\mathcal M)\le r$,
\[
\|\pi_{\mathcal M}(\mathbf u)-\pi_{\mathcal M}(\mathbf v)\|
\le\frac{\tau}{\tau-r}\|\mathbf u-\mathbf v\|.
\]
\end{lemma}

\begin{proof}
This follows from Theorem~4.8(8) of
\cite{federer1959curvature}.
\end{proof}

\begin{lemma}
\label{lem:local-graph}
There exist absolute constants $\rho\in(0,1)$ and $C>0$ such that, for every $\mathbf x\in\mathcal M$, each $\mathbf z\in\mathcal M\cap B_D(\mathbf x,\rho\tau/2)$ can be uniquely written as
\[
\mathbf z=\mathbf x+\mathbf u+\mathbf g_{\mathbf x}(\mathbf u),
\]
where $\mathbf u=\boldsymbol\Pi_{\mathbf x}(\mathbf z-\mathbf x)$ and $\mathbf g_{\mathbf x}:T_{\mathbf x}\mathcal M\cap B_D(0,\rho\tau)\to N_{\mathbf x}\mathcal M$ is a $C^1$ local graph map of $\mathcal M$ with $\mathbf g_{\mathbf x}(\mathbf 0)=\mathbf 0$ satisfying
\[
\|\mathrm D\mathbf g_{\mathbf x}(\mathbf u)\|\le1,
\qquad
\|\mathbf g_{\mathbf x}(\mathbf u)\|\le\frac{C}{\tau}\|\mathbf u\|^2.
\]
Moreover, for $0<r\le\rho\tau/2$,
\[
\mathcal H^d\bigl(\mathcal M\cap B_D(\mathbf x,r)\bigr)\le C_dr^d,
\]
where $C_d>0$ depends only on $d$.
\end{lemma}

\begin{proof}
Fix $\mathbf x\in\mathcal M$.

\medskip
\noindent\emph{Step 1. Reparametrization.} To express $\mathcal M$ locally as a graph over $T_{\mathbf x}\mathcal M$, we invert the tangential projection of the exponential map. By Lemma~1 of \cite{aamari2019non}, the exponential map can be written as $\exp_{\mathbf x}(\mathbf v)=\mathbf x+\mathbf v+\mathbf R_{\mathbf x}(\mathbf v)$ for $\mathbf v\in T_{\mathbf x}\mathcal M\cap B_D(0,\tau/4)$, where $\mathbf R_{\mathbf x}(\mathbf 0)=\mathbf 0$ and $\|\mathrm D\mathbf R_{\mathbf x}(\mathbf v)\|\le5\|\mathbf v\|/(4\tau)$. Define the $C^1$ map $\mathbf F(\mathbf v):=\boldsymbol\Pi_{\mathbf x}\bigl(\exp_{\mathbf x}(\mathbf v)-\mathbf x\bigr)=\mathbf v+\boldsymbol\Pi_{\mathbf x}\mathbf R_{\mathbf x}(\mathbf v)$. For $\mathbf v\in T_{\mathbf x}\mathcal M\cap B_D(0,\tau/8)$,
\[
\|\mathrm D\mathbf F(\mathbf v)-\mathbf I\|
=\|\mathrm D(\boldsymbol\Pi_{\mathbf x}\mathbf R_{\mathbf x})(\mathbf v)\|
\le\frac{5\|\mathbf v\|}{4\tau}\le\frac14,
\]
where $\mathbf I$ is the identity map on $T_{\mathbf x}\mathcal M$. Thus $\mathrm D\mathbf F(\mathbf v)$ is invertible with $\|[\mathrm D\mathbf F(\mathbf v)]^{-1}\|\le4/3$, and $\|\boldsymbol\Pi_{\mathbf x}\mathbf R_{\mathbf x}(\mathbf v)\|\le\|\mathbf v\|/4$.

For $\mathbf u\in T_{\mathbf x}\mathcal M\cap B_D(0,\tau/16)$, the map $\mathbf v\mapsto\mathbf u-\boldsymbol\Pi_{\mathbf x}\mathbf R_{\mathbf x}(\mathbf v)$ is therefore a $1/4$-contraction of $T_{\mathbf x}\mathcal M\cap B_D(0,\tau/8)$ into its interior, as
\[
\|\mathbf u-\boldsymbol\Pi_{\mathbf x}\mathbf R_{\mathbf x}(\mathbf v)\|
\le\frac{\tau}{16}+\frac{\|\mathbf v\|}{4}
\le\frac{3\tau}{32}
<\frac{\tau}{8}.
\]
Thus $\mathbf F^{-1}$ is well-defined on $T_{\mathbf x}\mathcal M\cap B_D(0,\tau/16)$ by the contraction mapping theorem. By the inverse function theorem, $\mathbf F^{-1}$ is $C^1$ and satisfies
\[
\mathrm D(\mathbf F^{-1})(\mathbf u)=\bigl[\mathrm D\mathbf F(\mathbf F^{-1}(\mathbf u))\bigr]^{-1},
\qquad
\|\mathrm D(\mathbf F^{-1})(\mathbf u)\|\le\frac43.
\]

\medskip
\noindent\emph{Step 2. Local graph representation.} Set $\rho:=1/16$. Define the $C^1$ normal graph map $\mathbf g_{\mathbf x}:T_{\mathbf x}\mathcal M\cap B_D(0,\rho\tau)\to N_{\mathbf x}\mathcal M$ by
\[
\mathbf g_{\mathbf x}(\mathbf u):=\boldsymbol\Pi_{\mathbf x}^\perp\bigl(\exp_{\mathbf x}(\mathbf F^{-1}(\mathbf u))-\mathbf x\bigr)=\boldsymbol\Pi_{\mathbf x}^\perp\mathbf R_{\mathbf x}\bigl(\mathbf F^{-1}(\mathbf u)\bigr),
\]
and the corresponding graph parametrization $\boldsymbol\Psi(\mathbf u):=\mathbf x+\mathbf u+\mathbf g_{\mathbf x}(\mathbf u)$. Clearly $\mathbf g_{\mathbf x}(\mathbf 0)=\mathbf 0$. For $\mathbf u\in T_{\mathbf x}\mathcal M\cap B_D(0,\rho\tau)$, we have $\|\mathbf F^{-1}(\mathbf u)\|\le\frac43\|\mathbf u\|<\tau/8$. By the chain rule,
\[
\|\mathrm D\mathbf g_{\mathbf x}(\mathbf u)\|
\le\bigl\|\mathrm D\mathbf R_{\mathbf x}\bigl(\mathbf F^{-1}(\mathbf u)\bigr)\bigr\|\,\|\mathrm D(\mathbf F^{-1})(\mathbf u)\|
\le\frac{5\|\mathbf F^{-1}(\mathbf u)\|}{4\tau}\cdot\frac43
\le\frac{20}{9\tau}\|\mathbf u\|\le1,
\]
which upon integration yields $\|\mathbf g_{\mathbf x}(\mathbf u)\|\le\frac{10}{9\tau}\|\mathbf u\|^2$.

Now fix any $\mathbf z\in\mathcal M\cap B_D(\mathbf x,\rho\tau/2)$. Every point on the line segment joining $\mathbf x$ and $\mathbf z$ lies within distance $\|\mathbf z-\mathbf x\|/2<\tau/2$ of $\mathcal M$. By Lemma~\ref{lem:projection-lipschitz} with $r=\tau/2$, $\pi_{\mathcal M}$ is $2$-Lipschitz on this segment, so its image is a curve in $\mathcal M$ connecting $\mathbf x$ and $\mathbf z$ of length at most $2\|\mathbf z-\mathbf x\|$. Thus the geodesic distance satisfies
\[
d_{\mathcal M}(\mathbf x,\mathbf z)
\le2\|\mathbf z-\mathbf x\|
\le\rho\tau<\frac{\tau}{8}.
\]
By compactness, a minimizing geodesic joins $\mathbf x$ to $\mathbf z$, so $\mathbf z=\exp_{\mathbf x}(\mathbf v)$ for some $\mathbf v\in T_{\mathbf x}\mathcal M$ with $\|\mathbf v\|=d_{\mathcal M}(\mathbf x,\mathbf z)<\tau/8$.

Setting $\mathbf u:=\boldsymbol\Pi_{\mathbf x}(\mathbf z-\mathbf x)\in T_{\mathbf x}\mathcal M\cap B_D(0,\rho\tau/2)$, we have $\mathbf F(\mathbf v)=\mathbf u$. Since $\mathbf F$ is injective on
$T_{\mathbf x}\mathcal M\cap B_D(0,\tau/8)$ from Step~1, we obtain $\mathbf v=\mathbf F^{-1}(\mathbf u)$ and hence $\mathbf z=\boldsymbol\Psi(\mathbf u)$.

\medskip
\noindent\emph{Step 3. Volume bound.} Fix $0<r\le\rho\tau/2$. By Step~2, every $\mathbf z\in\mathcal M\cap B_D(\mathbf x,r)$ has tangent coordinate $\mathbf u=\boldsymbol\Pi_{\mathbf x}(\mathbf z-\mathbf x)$ satisfying $\|\mathbf u\|\le\|\mathbf z-\mathbf x\|\le r$, so $\mathcal M\cap B_D(\mathbf x,r)\subset\boldsymbol\Psi\bigl(T_{\mathbf x}\mathcal M\cap B_D(0,r)\bigr)$.

Since $T_{\mathbf x}\mathcal M\perp N_{\mathbf x}\mathcal M$, the metric tensor pulled back by $\boldsymbol\Psi$ is $\mathbf I+[\mathrm D\mathbf g_{\mathbf x}(\mathbf u)]^\top\mathrm D\mathbf g_{\mathbf x}(\mathbf u)$, with area Jacobian
\[
J(\mathbf u)
:=\sqrt{\det\bigl(\mathbf I+[\mathrm D\mathbf g_{\mathbf x}(\mathbf u)]^\top\mathrm D\mathbf g_{\mathbf x}(\mathbf u)\bigr)}.
\]
Because $\|\mathrm D\mathbf g_{\mathbf x}(\mathbf u)\|\le1$, we have $1\le J(\mathbf u)\le(1+\|\mathrm D\mathbf g_{\mathbf x}(\mathbf u)\|^2)^{d/2}\le2^{d/2}$. The area formula then yields
\[
\mathcal H^d\bigl(\mathcal M\cap B_D(\mathbf x,r)\bigr)
\le\int_{T_{\mathbf x}\mathcal M\cap B_D(0,r)}J(\mathbf u)\,\mathrm d\mathbf u
\le 2^{d/2}\int_{T_{\mathbf x}\mathcal M\cap B_D(0,r)}\mathrm d\mathbf u
=2^{d/2}\omega_d r^d,
\]
where $\omega_d$ is the volume of the $d$-dimensional unit Euclidean ball.
\end{proof}

\subsection{Concentration of empirical local moments}

\begin{lemma}
\label{lem:weighted-local-geometry}
Let $c_0>0$ be a fixed absolute constant. There exist constants $C,c>0$, depending only on $d,\tau^d f_{\max},\tau^d f_{\min}$, such that if $\|\mathbf z-\mathbf z^\star\|\le c_0h$, $h/\tau+\sigma/h\le c$, and $n(h/\tau)^d\ge Ct$ for $t\ge1$, then with probability at least $1-e^{-t}$,
\[
c(h/\tau)^d\le m_{0,\mathbf z}\le C(h/\tau)^d,\qquad
\|\mathbf m_{1,\mathbf z}\|\le Ch\,m_{0,\mathbf z},\qquad
\|\mathbf M_{2,\mathbf z}\|\le Ch^2 m_{0,\mathbf z}.
\]
Moreover, on the same event,
$\mathbf A:=\boldsymbol\Pi_{\star}\widehat{\boldsymbol\Sigma}_{\mathbf z}\boldsymbol\Pi_{\star}$
satisfies $\operatorname{range}(\mathbf A)=T_{\mathbf z^\star}\mathcal M$,
$\lambda_d(\mathbf A)\ge ch^2$, and
\[
\|\widehat{\boldsymbol\Sigma}_{\mathbf z}-\mathbf A\|
\le Ch^2\biggl(\frac h\tau+\frac\sigma h\biggr).
\]
\end{lemma}

\begin{proof}
Let $r_i:=\|\mathbf x_i-\mathbf z^\star\|$.

\medskip
\noindent\emph{Step 1. Upper bounds on weighted moments.}
Since $\|\boldsymbol\epsilon_i\|\le\sigma\le ch$ and $\|\mathbf z-\mathbf z^\star\|\le c_0h$, the triangle inequality gives $\|\mathbf y_i-\mathbf z\|\ge r_i-(\sigma+c_0h)$. Hence $\|\mathbf y_i-\mathbf z\|^2\ge\frac12 r_i^2-(\sigma+c_0h)^2$, which yields
\[
w_{i,\mathbf z}
=\exp\bigl(-\|\mathbf y_i-\mathbf z\|^2/(2h^2)\bigr)
\le C\exp\bigl(-r_i^2/(4h^2)\bigr).
\]
To bound the moments $w_{i,\mathbf z}r_i^k$ simultaneously for $k\in\{0,1,2,4\}$, note that $(r_i/h)^k\le1+(r_i/h)^4$. We therefore introduce the envelope variables
\[
Z_i:=\exp\bigl(-r_i^2/(4h^2)\bigr)\bigl(1+(r_i/h)^4\bigr),
\]
which are i.i.d.\ and uniformly bounded by an absolute constant.

Let $\rho\in(0,1)$ be the constant from Lemma~\ref{lem:local-graph}. We decompose $\mathcal M\cap B_D(\mathbf z^\star,\rho\tau/2)$ into shells $jh\le r_i<(j+1)h$ truncated at $r_i=\rho\tau/2$. Applying Lemma~\ref{lem:local-graph} yields
\[
\begin{aligned}
\mathbb E[Z_i\mathbb I(r_i\le\rho\tau/2)]
&\le\sum_{j=0}^{\infty}f_{\max}
\mathcal H^d\Bigl(\mathcal M\cap B_D\bigl(\mathbf z^\star,((j+1)h)\wedge(\rho\tau/2)\bigr)\Bigr)
\bigl(1+(j+1)^4\bigr)e^{-j^2/4}\\
&\le C(h/\tau)^d.
\end{aligned}
\]
For the outer tail $r_i>\rho\tau/2$, we have $Z_i\le C\exp(-c_1\tau^2/h^2)\le C(h/\tau)^d$ whenever $h/\tau\le c$ for sufficiently small $c$. Consequently, 
\[
\mathbb E[Z_i]\le C(h/\tau)^d,
\qquad
\mathbb E[Z_i^2]\le C\mathbb E[Z_i]\le C(h/\tau)^d.
\]

Applying Bernstein's inequality to the bounded variables $Z_i$, if $n(h/\tau)^d\ge Ct$, then with probability at least $1-\frac12 e^{-t}$,
\[
\frac1n\sum_{i=1}^n Z_i\le C(h/\tau)^d.
\]
Since $w_{i,\mathbf z}r_i^k\le Ch^kZ_i$ by construction, on the same event we obtain
\[
\frac1n\sum_{i=1}^n w_{i,\mathbf z}r_i^k
\le Ch^k(h/\tau)^d,
\qquad k\in\{0,1,2,4\}.
\]

\medskip
\noindent\emph{Step 2. Lower bounds on mass and tangential covariance.}
By Corollary~4.2.13 of \cite{vershynin2018high}, there exists a $1/8$-net $\mathcal N$ of the unit sphere in $T_{\mathbf z^\star}\mathcal M$ with $|\mathcal N|\le17^d$. Using the local graph parametrization $\boldsymbol\Psi(\mathbf u):=\mathbf z^\star+\mathbf u+\mathbf g_{\mathbf z^\star}(\mathbf u)$ from Lemma~\ref{lem:local-graph}, define for each $\mathbf v\in\mathcal N$ the two opposite caps
\[
\mathcal C_{\mathbf v}^\pm:=\Bigl\{\boldsymbol\Psi(\mathbf u):\mathbf u\in T_{\mathbf z^\star}\mathcal M,\ \bigl\|\mathbf u\mp\tfrac h2\mathbf v\bigr\|\le\tfrac h8\Bigr\}\subset\mathcal M.
\]
For $h/\tau\le c$, every point $\mathbf x_i=\boldsymbol\Psi(\mathbf u_i)\in\mathcal C_{\mathbf v}^\pm$ satisfies $\|\mathbf u_i\|\le\frac{5h}{8}\le\rho\tau/2$ and $\|\mathbf g_{\mathbf z^\star}(\mathbf u_i)\|\le\frac{C}{\tau}h^2\le\frac{3h}{8}$, which yields $r_i\le h$. This implies $\|\mathbf y_i-\mathbf z\|\le r_i+\sigma+c_0h\le Ch$, and hence $w_{i,\mathbf z}\ge c$ on $\mathcal C_{\mathbf v}^\pm$. 

Moreover, since the area Jacobian of $\boldsymbol\Psi$ satisfies $J(\mathbf u)\ge1$, each cap has sampling probability
\[
\mathbb P(\mathbf x_i\in\mathcal C_{\mathbf v}^\pm)
\ge f_{\min}\mathcal H^d(\mathcal C_{\mathbf v}^\pm)
\ge f_{\min}\omega_d(h/8)^d
\ge c(h/\tau)^d.
\]
By Chernoff's inequality and a union bound over all $2|\mathcal N|\le2\cdot17^d$ caps, when $n(h/\tau)^d\ge Ct$, with probability at least $1-\frac12 e^{-t}$ every cap contains at least $cn(h/\tau)^d$ sample points:
\[
\sum_{i=1}^n\mathbb I(\mathbf x_i\in\mathcal C_{\mathbf v}^\pm)\ge cn(h/\tau)^d,
\qquad\forall\,\mathbf v\in\mathcal N.
\]

Work on the intersection of the events from Steps~1 and~2, which holds with probability at least $1-e^{-t}$. For any $\mathbf v\in\mathcal N$, the count bound and $w_{i,\mathbf z}\ge c$ on $\mathcal C_{\mathbf v}^+$ give
\[
m_{0,\mathbf z}
\ge\frac1n\sum_{\mathbf x_i\in\mathcal C_{\mathbf v}^+}w_{i,\mathbf z}
\ge c(h/\tau)^d,
\]
which together with Step~1 establishes $c(h/\tau)^d\le m_{0,\mathbf z}\le C(h/\tau)^d$. Using $\|\mathbf y_i-\mathbf z\|\le r_i+Ch$ and the moment bounds from Step~1, we obtain
\[
\begin{aligned}
\|\mathbf m_{1,\mathbf z}\|
&\le\frac1n\sum_{i=1}^n w_{i,\mathbf z}(r_i+Ch)
\le Ch(h/\tau)^d
\le Ch\,m_{0,\mathbf z},\\
\|\mathbf M_{2,\mathbf z}\|
&\le\frac2n\sum_{i=1}^n w_{i,\mathbf z}(r_i^2+C^2h^2)
\le Ch^2(h/\tau)^d
\le Ch^2 m_{0,\mathbf z}.
\end{aligned}
\]

To bound $\lambda_d(\mathbf A)$ from below, fix a unit vector $\mathbf v_0\in T_{\mathbf z^\star}\mathcal M$ and choose $\mathbf v\in\mathcal N$ with $\|\mathbf v-\mathbf v_0\|\le1/8$. Since $\mathbf v_0\perp\mathbf g_{\mathbf z^\star}(\mathbf u_i)$, every point $\mathbf x_i\in\mathcal C_{\mathbf v}^+$ satisfies, for $\sigma/h\le1/16$,
\[
\mathbf v_0^\top(\mathbf y_i-\mathbf z^\star)
=\mathbf v_0^\top\mathbf u_i+\mathbf v_0^\top\boldsymbol\epsilon_i
\ge\frac h2\mathbf v_0^\top\mathbf v-\bigl\|\mathbf u_i-\tfrac h2\mathbf v\bigr\|-\sigma
\ge\frac{5h}{16}-\sigma\ge\frac h4,
\]
and symmetrically $\mathbf v_0^\top(\mathbf y_i-\mathbf z^\star)\le-h/4$ for every $\mathbf x_i\in\mathcal C_{\mathbf v}^-$. 

Set $a:=\mathbf v_0^\top(\bar{\mathbf y}_{\mathbf z}-\mathbf z^\star)$, so that $\mathbf v_0^\top(\mathbf y_i-\bar{\mathbf y}_{\mathbf z})=\mathbf v_0^\top(\mathbf y_i-\mathbf z^\star)-a$, and let $\mathcal C$ be $\mathcal C_{\mathbf v}^-$ if $a\ge0$ and $\mathcal C_{\mathbf v}^+$ otherwise. Each index $i$ with $\mathbf x_i\in\mathcal C$ then satisfies $|\mathbf v_0^\top(\mathbf y_i-\mathbf z^\star)-a|\ge h/4$. Using $w_{i,\mathbf z}\ge c$ on $\mathcal C$, the cap-count lower bound above, and $m_{0,\mathbf z}\le C(h/\tau)^d$ yields
\[
\begin{aligned}
\mathbf v_0^\top\mathbf A\mathbf v_0
=\mathbf v_0^\top\widehat{\boldsymbol\Sigma}_{\mathbf z}\mathbf v_0
&=\frac{1}{nm_{0,\mathbf z}}\sum_{i=1}^n w_{i,\mathbf z}\bigl(\mathbf v_0^\top(\mathbf y_i-\mathbf z^\star)-a\bigr)^2\\
&\ge\frac{1}{nm_{0,\mathbf z}}\sum_{\mathbf x_i\in\mathcal C} w_{i,\mathbf z}\bigl(\mathbf v_0^\top(\mathbf y_i-\mathbf z^\star)-a\bigr)^2\\
&\ge\frac{cn(h/\tau)^d\cdot c\cdot(h/4)^2}{nm_{0,\mathbf z}}
\ge ch^2.
\end{aligned}
\]
Since this holds uniformly over all unit vectors $\mathbf v_0\in T_{\mathbf z^\star}\mathcal M$, it follows that $\operatorname{range}(\mathbf A)=T_{\mathbf z^\star}\mathcal M$ and $\lambda_d(\mathbf A)\ge ch^2$.

\medskip
\noindent\emph{Step 3. Normal and cross covariance blocks.}
By Lemma~\ref{lemma:reach}, $\|\boldsymbol\Pi_\star^\perp(\mathbf x_i-\mathbf z^\star)\|\le r_i^2/(2\tau)$, which together with $\|\boldsymbol\epsilon_i\|\le\sigma$ gives
\[
\|\boldsymbol\Pi_\star^\perp(\mathbf y_i-\mathbf z^\star)\|
\le\frac{r_i^2}{2\tau}+\sigma.
\]
Since centering at $\bar{\mathbf y}_{\mathbf z}$ minimizes the weighted second moment,
\[
\widehat{\boldsymbol\Sigma}_{\mathbf z}
\preceq\frac{1}{nm_{0,\mathbf z}}\sum_{i=1}^n w_{i,\mathbf z}(\mathbf y_i-\mathbf z^\star)(\mathbf y_i-\mathbf z^\star)^\top.
\]
Projecting onto the tangent and normal spaces and using Step~1, $m_{0,\mathbf z}\ge c(h/\tau)^d$, and $\|\mathbf y_i-\mathbf z^\star\|\le r_i+\sigma$ yields
\[
\|\mathbf A\|
\le\frac{1}{nm_{0,\mathbf z}}\sum_{i=1}^n w_{i,\mathbf z}\|\mathbf y_i-\mathbf z^\star\|^2
\le Ch^2,
\]
and
\[
\|\boldsymbol\Pi_\star^\perp\widehat{\boldsymbol\Sigma}_{\mathbf z}\boldsymbol\Pi_\star^\perp\|
\le\frac{1}{nm_{0,\mathbf z}}\sum_{i=1}^n w_{i,\mathbf z}\|\boldsymbol\Pi_\star^\perp(\mathbf y_i-\mathbf z^\star)\|^2
\le\frac{2}{nm_{0,\mathbf z}}\sum_{i=1}^n w_{i,\mathbf z}\biggl(\frac{r_i^4}{4\tau^2}+\sigma^2\biggr)
\le C\biggl(\frac{h^2}{\tau}+\sigma\biggr)^2.
\]

Since $\widehat{\boldsymbol\Sigma}_{\mathbf z}\succeq\mathbf 0$, Cauchy--Schwarz gives
\[
\|\boldsymbol\Pi_\star\widehat{\boldsymbol\Sigma}_{\mathbf z}\boldsymbol\Pi_\star^\perp\|
\le\|\mathbf A\|^{1/2}\|\boldsymbol\Pi_\star^\perp\widehat{\boldsymbol\Sigma}_{\mathbf z}\boldsymbol\Pi_\star^\perp\|^{1/2}\le Ch\biggl(\frac{h^2}{\tau}+\sigma\biggr).
\]
Combining the normal and cross blocks and using $h/\tau+\sigma/h\le c$ gives
\[
\begin{aligned}
\|\widehat{\boldsymbol\Sigma}_{\mathbf z}-\mathbf A\|
&\le2\|\boldsymbol\Pi_\star\widehat{\boldsymbol\Sigma}_{\mathbf z}\boldsymbol\Pi_\star^\perp\|+\|\boldsymbol\Pi_\star^\perp\widehat{\boldsymbol\Sigma}_{\mathbf z}\boldsymbol\Pi_\star^\perp\|\\
&\le Ch\biggl(\frac{h^2}{\tau}+\sigma\biggr)+C\biggl(\frac{h^2}{\tau}+\sigma\biggr)^2\\
&\le Ch^2\biggl(\frac h\tau+\frac\sigma h\biggr).
\end{aligned}
\]
\end{proof}

\subsection{Tail bounds for Gaussian moment noise}

\begin{lemma}
\label{lem:moment-gaussian-noise}
Let $(\xi_0,\boldsymbol\xi_1,\boldsymbol\xi_2)\sim\mathcal N(\mathbf 0,\varsigma^2\mathbf I_p)$ with $\xi_0\in\mathbb R$, $\boldsymbol\xi_1\in\mathbb R^D$, and $\boldsymbol\xi_2\in\mathbb R^{D(D+1)/2}$, where $p=1+D+D(D+1)/2$. There exists an absolute constant $C>0$ such that, for any $t\ge1$, with probability at least $1-e^{-t}$,
\[
|\xi_0|\le C\varsigma\sqrt t,\qquad
\|\boldsymbol\xi_1\|\le C\varsigma(\sqrt D+\sqrt t),\qquad
\|\operatorname{smat}(\boldsymbol\xi_2)\|\le C\varsigma(\sqrt D+\sqrt t).
\]
\end{lemma}

\begin{proof}
Since $\xi_0\sim\mathcal N(0,\varsigma^2)$, the standard normal tail bound gives $\mathbb P(|\xi_0|>\varsigma\sqrt{2r})\le2e^{-r}$ for any $r>0$. Next, for any unit vector $\mathbf u\in\mathbb R^D$, we have $\mathbf u^\top\boldsymbol\xi_1\sim\mathcal N(0,\varsigma^2)$. Moreover, since $\operatorname{svec}$ preserves the Frobenius inner product,
\[
\mathbf u^\top\operatorname{smat}(\boldsymbol\xi_2)\mathbf u
=\operatorname{tr}\bigl(\operatorname{smat}(\boldsymbol\xi_2)\mathbf u\mathbf u^\top\bigr)
=\langle\boldsymbol\xi_2,\operatorname{svec}(\mathbf u\mathbf u^\top)\rangle
\sim\mathcal N(0,\varsigma^2),
\]
where we used $\|\operatorname{svec}(\mathbf u\mathbf u^\top)\|=\|\mathbf u\mathbf u^\top\|_{\mathrm F}=\|\mathbf u\|^2=1$.

Applying the Gaussian tail bounds and a standard sphere-covering argument \citep[see, e.g.,][Corollary~4.2.13 and Exercise~4.4.2]{vershynin2018high}, followed by a union bound, gives, with probability at least $1-e^{-t}$,
\[
|\xi_0|\le C\varsigma\sqrt t,
\qquad
\|\boldsymbol\xi_1\|\vee\|\operatorname{smat}(\boldsymbol\xi_2)\|
\le C\varsigma(\sqrt D+\sqrt t).
\]
\end{proof}

\section{Proofs for private local estimation and denoising}
\label{app:local-estimation-proofs}

Throughout this section, write
\[
R_t:=\frac h\tau+\frac\sigma h+\frac{\varsigma(\sqrt D+\sqrt t)}{(h/\tau)^d}.
\]

\subsection{Proof of Proposition~\ref{prop:local-moment-sensitivity}}

Let $\mathcal D'$ denote a replace-one neighboring dataset obtained by replacing $\mathbf y_i$ in $\mathcal D$ with $\mathbf y_i'$.

\medskip
\noindent\emph{Step 1. Single-record reduction.}
Define
\[
\boldsymbol\psi(\mathbf u)
:=e^{-\|\mathbf u\|^2/2}
\bigg(1,\mathbf u,
\frac{\operatorname{svec}(\mathbf u\mathbf u^\top)}{\sqrt2}\bigg).
\]
Thus $\mathbf F_{\mathbf z}(\mathcal D)=n^{-1}\sum_{j=1}^n\boldsymbol\psi\big((\mathbf y_j-\mathbf z)/h\big)$ and
\[
\|\mathbf F_{\mathbf z}(\mathcal D)-\mathbf F_{\mathbf z}(\mathcal D')\|
=\frac1n\|\boldsymbol\psi(\mathbf u)-\boldsymbol\psi(\mathbf v)\|,
\]
where $\mathbf u=(\mathbf y_i-\mathbf z)/h$ and $\mathbf v=(\mathbf y_i'-\mathbf z)/h$.

\medskip
\noindent\emph{Step 2. Bounding the difference.}
Note that
\begin{align*}
\|\boldsymbol\psi(\mathbf u)-\boldsymbol\psi(\mathbf v)\|^2
&=e^{-\|\mathbf u\|^2}\big(1+\|\mathbf u\|^2+\tfrac12\|\mathbf u\|^4\big)
+e^{-\|\mathbf v\|^2}\big(1+\|\mathbf v\|^2+\tfrac12\|\mathbf v\|^4\big)
\\
&\quad-2e^{-(\|\mathbf u\|^2+\|\mathbf v\|^2)/2}
\big(1+\mathbf u^\top\mathbf v+\tfrac12(\mathbf u^\top\mathbf v)^2\big)
\\
&=\frac12e^{-\|\mathbf u\|^2}(1+\|\mathbf u\|^2)^2
+\frac12e^{-\|\mathbf v\|^2}(1+\|\mathbf v\|^2)^2
\\
&\quad+\frac12\big(e^{-\|\mathbf u\|^2/2}-e^{-\|\mathbf v\|^2/2}\big)^2
-e^{-(\|\mathbf u\|^2+\|\mathbf v\|^2)/2}(1+\mathbf u^\top\mathbf v)^2.
\end{align*}
Using
\[
\frac12\big(e^{-\|\mathbf u\|^2/2}-e^{-\|\mathbf v\|^2/2}\big)^2
\le\big(e^{-\|\mathbf u\|^2/2}-e^{-1/2}\big)^2
+\big(e^{-\|\mathbf v\|^2/2}-e^{-1/2}\big)^2,
\]
we have
\[
\|\boldsymbol\psi(\mathbf u)-\boldsymbol\psi(\mathbf v)\|^2
\le\frac12H(\|\mathbf u\|^2)+\frac12H(\|\mathbf v\|^2),
\]
where $H(x):=e^{-x}(1+x)^2+2\big(e^{-x/2}-e^{-1/2}\big)^2$.

To show that $H(x)\le4/e$ for all $x\ge0$, define $G(x):=2e^{x-1}+4e^{(x-1)/2}-(1+x)^2-2$, so that $H(x)\le4/e$ is equivalent to $G(x)\ge0$. Since $G'$ is strictly convex with $G'(0)<0$ and $G'(1)=0$, $x=1$ is the unique global minimizer of $G$ on $[0,\infty)$, which gives $G(x)\ge G(1)=0$. Consequently, $\sup_{x\ge0}H(x)=4/e$.

\medskip
\noindent\emph{Step 3. Privacy guarantee and sharpness.}
By Step~2, $\|\boldsymbol\psi(\mathbf u)-\boldsymbol\psi(\mathbf v)\|^2 \le4/e$, so $\Delta_2(\mathbf F_{\mathbf z})\le2/(n\sqrt e)$. The $(\varepsilon,\delta)$-DP guarantee then follows from \eqref{eq:gaussian-calibration}.

Equality is attained when replacing a record from $\mathbf z+h\mathbf u$ to $\mathbf z-h\mathbf u$, with $\mathbf u=-\mathbf v$ and $\|\mathbf u\|=1$. Hence, the sensitivity bound is sharp when the data domain covers $\mathbb R^D$.
\qed

\subsection{Proof of Theorem~\ref{thm:private-tangent-accuracy}}

Since $\|\widetilde{\mathbf P}_{\mathbf z}-\boldsymbol\Pi_\star\|\le1$, the bound holds trivially for $R_t>c$ once $C\ge1/c$ and we may assume $R_t\le c$ for a sufficiently small constant $c>0$.

\medskip
\noindent\emph{Step 1. Non-private local moments and covariance.}
Since $\mathbf z\in\mathcal M_\gamma$ and $\gamma\le c_0h$, we have $\|\mathbf z-\mathbf z^\star\|\le\gamma\le c_0h$, and $R_t\le c$ implies $h/\tau+\sigma/h\le c$. Applying Lemma~\ref{lem:weighted-local-geometry} with $t+\log2$ in place of $t$, when $n(h/\tau)^d\ge C(\eta+t)$ for a sufficiently large constant $C$, we obtain that with probability at least $1-\frac12 e^{-t}$,
\[
c(h/\tau)^d\le m_{0,\mathbf z}\le C(h/\tau)^d,\qquad
\|\mathbf m_{1,\mathbf z}\|\le Ch\,m_{0,\mathbf z},\qquad
\|\mathbf M_{2,\mathbf z}\|\le Ch^2 m_{0,\mathbf z},
\]
and the tangent covariance block $\mathbf A:=\boldsymbol\Pi_\star\widehat{\boldsymbol\Sigma}_{\mathbf z}\boldsymbol\Pi_\star$ satisfies $\operatorname{range}(\mathbf A)=T_{\mathbf z^\star}\mathcal M$, $\lambda_d(\mathbf A)\ge ch^2$, and
\[
\|\widehat{\boldsymbol\Sigma}_{\mathbf z}-\mathbf A\|
\le Ch^2\biggl(\frac h\tau+\frac\sigma h\biggr).
\]

\medskip
\noindent\emph{Step 2. Perturbation of the private local covariance.}
Applying Lemma~\ref{lem:moment-gaussian-noise} with $t+\log2$ in place of $t$, we have with probability at least $1-\frac12 e^{-t}$,
\[
|\xi_0|\le C\varsigma\sqrt t,\qquad
\|h\boldsymbol\xi_1\|\le Ch\varsigma(\sqrt D+\sqrt t),\qquad
\|\sqrt2 h^2\operatorname{smat}(\boldsymbol\xi_2)\|\le Ch^2\varsigma(\sqrt D+\sqrt t).
\]
On the intersection of the events from Steps~1 and~2, which holds with probability at least $1-e^{-t}$, choosing $c$ sufficiently small and $C$ sufficiently large ensures that $|\xi_0|\le\frac12 m_{0,\mathbf z}$ and $\frac12 m_{0,\mathbf z}\ge\eta/n$. Consequently, 
\[
\tilde q_{\mathbf z}
=\tilde m_{0,\mathbf z}\vee(\eta/n)
=m_{0,\mathbf z}+\xi_0
\ge\frac12 m_{0,\mathbf z}
\ge c(h/\tau)^d.
\]

Using $\tilde{\mathbf m}_{1,\mathbf z}=\mathbf m_{1,\mathbf z}+h\boldsymbol\xi_1$ and $\tilde q_{\mathbf z}=m_{0,\mathbf z}+\xi_0$, we have
\[
\frac{\tilde{\mathbf m}_{1,\mathbf z}}{\tilde q_{\mathbf z}}-\frac{\mathbf m_{1,\mathbf z}}{m_{0,\mathbf z}}
=\frac{h\boldsymbol\xi_1-\xi_0(\mathbf m_{1,\mathbf z}/m_{0,\mathbf z})}{m_{0,\mathbf z}+\xi_0},
\]
which, together with $\|\mathbf m_{1,\mathbf z}\|/m_{0,\mathbf z}\le Ch$ and $\tilde q_{\mathbf z}\ge c(h/\tau)^d$, yields
\[
\biggl\|\frac{\tilde{\mathbf m}_{1,\mathbf z}}{\tilde q_{\mathbf z}}-\frac{\mathbf m_{1,\mathbf z}}{m_{0,\mathbf z}}\biggr\|
\le\frac{\|h\boldsymbol\xi_1\|+|\xi_0|\,\|\mathbf m_{1,\mathbf z}\|/m_{0,\mathbf z}}{\tilde q_{\mathbf z}}
\le Ch\frac{\varsigma(\sqrt D+\sqrt t)}{(h/\tau)^d}.
\]
Similarly, expanding $\widetilde{\mathbf M}_{2,\mathbf z}=\mathbf M_{2,\mathbf z}+\sqrt2 h^2\operatorname{smat}(\boldsymbol\xi_2)$ with $\|\mathbf M_{2,\mathbf z}\|/m_{0,\mathbf z}\le Ch^2$ gives
\[
\biggl\|\frac{\widetilde{\mathbf M}_{2,\mathbf z}}{\tilde q_{\mathbf z}}-\frac{\mathbf M_{2,\mathbf z}}{m_{0,\mathbf z}}\biggr\|
\le\frac{\|\sqrt2 h^2\operatorname{smat}(\boldsymbol\xi_2)\|+|\xi_0|\,\|\mathbf M_{2,\mathbf z}\|/m_{0,\mathbf z}}{\tilde q_{\mathbf z}}
\le Ch^2\frac{\varsigma(\sqrt D+\sqrt t)}{(h/\tau)^d}.
\]
The preceding bounds and $R_t\le c$ give $\|\tilde{\mathbf m}_{1,\mathbf z}/\tilde q_{\mathbf z}\|+\|\mathbf m_{1,\mathbf z}/m_{0,\mathbf z}\|\le Ch$. Using $\|\mathbf u\mathbf u^\top-\mathbf v\mathbf v^\top\|\le(\|\mathbf u\|+\|\mathbf v\|)\|\mathbf u-\mathbf v\|$, we obtain
\[
\begin{aligned}
\|\widetilde{\boldsymbol\Sigma}_{\mathbf z}-\widehat{\boldsymbol\Sigma}_{\mathbf z}\|
&\le\biggl\|\frac{\widetilde{\mathbf M}_{2,\mathbf z}}{\tilde q_{\mathbf z}}-\frac{\mathbf M_{2,\mathbf z}}{m_{0,\mathbf z}}\biggr\|
+Ch\biggl\|\frac{\tilde{\mathbf m}_{1,\mathbf z}}{\tilde q_{\mathbf z}}-\frac{\mathbf m_{1,\mathbf z}}{m_{0,\mathbf z}}\biggr\|\\
&\le Ch^2\frac{\varsigma(\sqrt D+\sqrt t)}{(h/\tau)^d}.
\end{aligned}
\]
Combining this bound with Step~1 via the triangle inequality gives
\[
\|\widetilde{\boldsymbol\Sigma}_{\mathbf z}-\mathbf A\|
\le\|\widetilde{\boldsymbol\Sigma}_{\mathbf z}-\widehat{\boldsymbol\Sigma}_{\mathbf z}\|+\|\widehat{\boldsymbol\Sigma}_{\mathbf z}-\mathbf A\|
\le Ch^2 R_t.
\]

\medskip
\noindent\emph{Step 3. Eigenspace perturbation.}
Let $\widetilde{\mathbf V}_{\mathbf z,d}$ contain the leading $d$ orthonormal eigenvectors of $\widetilde{\boldsymbol\Sigma}_{\mathbf z}$, so that
\[
\widetilde{\boldsymbol\Sigma}_{\mathbf z}\widetilde{\mathbf V}_{\mathbf z,d}
=\widetilde{\mathbf V}_{\mathbf z,d}\boldsymbol\Lambda,
\]
where $\boldsymbol\Lambda\in\mathbb R^{d\times d}$ is diagonal. Since $\lambda_d(\mathbf A)\ge ch^2$ and $\|\widetilde{\boldsymbol\Sigma}_{\mathbf z}-\mathbf A\|\le Ch^2 R_t$, Weyl's inequality ensures that for $R_t\le c$ with $c$ sufficiently small,
\[
\lambda_{\min}(\boldsymbol\Lambda)
=\lambda_d(\widetilde{\boldsymbol\Sigma}_{\mathbf z})
\ge\lambda_d(\mathbf A)-\|\widetilde{\boldsymbol\Sigma}_{\mathbf z}-\mathbf A\|
\ge\frac c2 h^2,
\]
so $\boldsymbol\Lambda$ is invertible with $\|\boldsymbol\Lambda^{-1}\|\le C h^{-2}$. As $\operatorname{range}(\mathbf A)=T_{\mathbf z^\star}\mathcal M$, we have $\boldsymbol\Pi_\star^\perp\mathbf A=\mathbf 0$, and therefore
\[
\boldsymbol\Pi_\star^\perp\widetilde{\mathbf V}_{\mathbf z,d}
=\boldsymbol\Pi_\star^\perp\widetilde{\boldsymbol\Sigma}_{\mathbf z}\widetilde{\mathbf V}_{\mathbf z,d}\boldsymbol\Lambda^{-1}
=\boldsymbol\Pi_\star^\perp(\widetilde{\boldsymbol\Sigma}_{\mathbf z}-\mathbf A)\widetilde{\mathbf V}_{\mathbf z,d}\boldsymbol\Lambda^{-1}.
\]
Since $\widetilde{\mathbf P}_{\mathbf z}$ and $\boldsymbol\Pi_\star$ are rank-$d$ orthogonal projectors, by the principal-angle identity we have
\[
\|\widetilde{\mathbf P}_{\mathbf z}-\boldsymbol\Pi_\star\|
=\|\boldsymbol\Pi_\star^\perp\widetilde{\mathbf V}_{\mathbf z,d}\|
\le\|\widetilde{\boldsymbol\Sigma}_{\mathbf z}-\mathbf A\|\,\|\boldsymbol\Lambda^{-1}\|
\le CR_t,
\]
completing the proof.
\qed

\subsection{Proof of Theorem~\ref{thm:private-hard-correction}}

Apply Lemmas~\ref{lem:weighted-local-geometry} and~\ref{lem:moment-gaussian-noise} with $t+\log2$ in place of $t$. Under $n(h/\tau)^d\ge C(\eta+t)$ and $R_t\le c$, we work on the intersection of their events, which holds with probability at least $1-e^{-t}$.

\medskip
\noindent\emph{Step 1. Normal deviation of the non-private local mean.}
Let $r_i:=\|\mathbf x_i-\mathbf z^\star\|$. By Lemma~\ref{lemma:reach} and $\|\boldsymbol\epsilon_i\|\le\sigma$, 
\[
\|\boldsymbol\Pi_\star^\perp(\mathbf y_i-\mathbf z^\star)\|
\le\|\boldsymbol\Pi_\star^\perp(\mathbf x_i-\mathbf z^\star)\|+\|\boldsymbol\epsilon_i\|
\le\frac{r_i^2}{2\tau}+\sigma.
\]
Using $\bar{\mathbf y}_{\mathbf z}-\mathbf z^\star=(nm_{0,\mathbf z})^{-1}\sum_{i=1}^n w_{i,\mathbf z}(\mathbf y_i-\mathbf z^\star)$, the moment bounds $n^{-1}\sum_{i=1}^n w_{i,\mathbf z}r_i^2\le Ch^2(h/\tau)^d$ and $m_{0,\mathbf z}\ge c(h/\tau)^d$ from Steps~1--2 of the proof of Lemma~\ref{lem:weighted-local-geometry} yield
\[
\|\boldsymbol\Pi_\star^\perp(\bar{\mathbf y}_{\mathbf z}-\mathbf z^\star)\|
\le\frac{1}{nm_{0,\mathbf z}}\sum_{i=1}^n w_{i,\mathbf z}\biggl(\frac{r_i^2}{2\tau}+\sigma\biggr)
\le C\biggl(\frac{h^2}{\tau}+\sigma\biggr).
\]

\medskip
\noindent\emph{Step 2. Private mean and tangent-space perturbations.}
Step~2 of the proof of Theorem~\ref{thm:private-tangent-accuracy} gives
\[
\|\tilde{\boldsymbol\mu}_{\mathbf z}-\bar{\mathbf y}_{\mathbf z}\|
=\biggl\|\frac{\tilde{\mathbf m}_{1,\mathbf z}}{\tilde q_{\mathbf z}}-\frac{\mathbf m_{1,\mathbf z}}{m_{0,\mathbf z}}\biggr\|
\le Ch\frac{\varsigma(\sqrt D+\sqrt t)}{(h/\tau)^d},
\]
while Step~3 of the same proof yields $\|\widetilde{\mathbf P}_{\mathbf z}-\boldsymbol\Pi_\star\|\le CR_t$.

\medskip
\noindent\emph{Step 3. One-step correction error.}
Since $\mathbf z^\star=\pi_{\mathcal M}(\mathbf z)$, we have $\mathbf z-\mathbf z^\star\in N_{\mathbf z^\star}\mathcal M$. Using \eqref{eq:hard-correction}, we obtain
\[
\tilde{\mathbf z}_{\mathrm{hard}}-\mathbf z^\star
=\boldsymbol\Pi_\star^\perp(\bar{\mathbf y}_{\mathbf z}-\mathbf z^\star)
+(\widetilde{\mathbf P}_{\mathbf z}-\boldsymbol\Pi_\star)(\mathbf z-\bar{\mathbf y}_{\mathbf z})+(\mathbf I-\widetilde{\mathbf P}_{\mathbf z})(\tilde{\boldsymbol\mu}_{\mathbf z}-\bar{\mathbf y}_{\mathbf z}).
\]
Since $\mathbf z-\bar{\mathbf y}_{\mathbf z}=-\mathbf m_{1,\mathbf z}/m_{0,\mathbf z}$, Lemma~\ref{lem:weighted-local-geometry} guarantees $\|\mathbf z-\bar{\mathbf y}_{\mathbf z}\|\le Ch$. Furthermore, $\|\mathbf I-\widetilde{\mathbf P}_{\mathbf z}\|\le1$. Applying the bounds from Steps~1 and~2 gives
\[
\begin{aligned}
\|\tilde{\mathbf z}_{\mathrm{hard}}-\mathbf z^\star\|
&\le\|\boldsymbol\Pi_\star^\perp(\bar{\mathbf y}_{\mathbf z}-\mathbf z^\star)\|
+\|\widetilde{\mathbf P}_{\mathbf z}-\boldsymbol\Pi_\star\|\,\|\mathbf z-\bar{\mathbf y}_{\mathbf z}\|
+\|\tilde{\boldsymbol\mu}_{\mathbf z}-\bar{\mathbf y}_{\mathbf z}\|\\
&\le C\biggl(\frac{h^2}{\tau}+\sigma\biggr)+ChR_t
+Ch\frac{\varsigma(\sqrt D+\sqrt t)}{(h/\tau)^d}\\
&\le C\biggl(\frac{h^2}{\tau}+\sigma+\frac{h\varsigma(\sqrt D+\sqrt t)}{(h/\tau)^d}\biggr).
\end{aligned}
\]
\qed

\subsection{Proof of Theorem~\ref{thm:private-soft-correction}}

We work on the same joint event of Lemmas~\ref{lem:weighted-local-geometry} and~\ref{lem:moment-gaussian-noise} with $t+\log2$ as in the proof of Theorem~\ref{thm:private-hard-correction}, which holds with probability at least $1-e^{-t}$ under $n(h/\tau)^d\ge C(\eta+t)$ and $R_t\le c$. This event depends only on the reference sample and Gaussian perturbations, and does not depend on $\lambda$.

\medskip
\noindent\emph{Step 1. Spectral and displacement bounds.}
By Lemma~\ref{lem:weighted-local-geometry} and Step~3 of its proof, the non-private covariance $\widehat{\boldsymbol\Sigma}_{\mathbf z}$ satisfies
\[
\lambda_d\bigl(\boldsymbol\Pi_\star\widehat{\boldsymbol\Sigma}_{\mathbf z}\boldsymbol\Pi_\star\bigr)\ge ch^2,
\qquad
\|\boldsymbol\Pi_\star^\perp\widehat{\boldsymbol\Sigma}_{\mathbf z}\boldsymbol\Pi_\star^\perp\|
\le C\biggl(\frac{h^2}{\tau}+\sigma\biggr)^2.
\]
Since $\operatorname{dim}(T_{\mathbf z^\star}\mathcal M)=d$ and $\operatorname{dim}(N_{\mathbf z^\star}\mathcal M)=D-d$, the Courant--Fischer min--max theorem implies
\[
\lambda_d(\widehat{\boldsymbol\Sigma}_{\mathbf z})
\ge\min_{\mathbf v\in T_{\mathbf z^\star}\mathcal M,\,\|\mathbf v\|=1}\mathbf v^\top\widehat{\boldsymbol\Sigma}_{\mathbf z}\mathbf v
\ge ch^2,
\qquad
\lambda_{d+1}(\widehat{\boldsymbol\Sigma}_{\mathbf z})
\le\max_{\mathbf v\in N_{\mathbf z^\star}\mathcal M,\,\|\mathbf v\|=1}\mathbf v^\top\widehat{\boldsymbol\Sigma}_{\mathbf z}\mathbf v
\le C\biggl(\frac{h^2}{\tau}+\sigma\biggr)^2.
\]
Combining this with the bound $\|\widetilde{\boldsymbol\Sigma}_{\mathbf z}-\widehat{\boldsymbol\Sigma}_{\mathbf z}\|\le Ch^2\varsigma(\sqrt D+\sqrt t)/(h/\tau)^d$ from Step~2 of the proof of Theorem~\ref{thm:private-tangent-accuracy}, Weyl's inequality gives, for $R_t\le c$ with $c$ sufficiently small,
\[
\lambda_d(\widetilde{\boldsymbol\Sigma}_{\mathbf z})\ge ch^2,
\qquad
\lambda_{d+1}(\widetilde{\boldsymbol\Sigma}_{\mathbf z})\le C\Delta_{h,t}.
\]
Let $\tilde s_1\ge\cdots\ge\tilde s_D\ge0$ denote the eigenvalues of $(\widetilde{\boldsymbol\Sigma}_{\mathbf z})_+$, where $\tilde s_j=\lambda_j(\widetilde{\boldsymbol\Sigma}_{\mathbf z})\vee 0$. Since $\lambda_d(\widetilde{\boldsymbol\Sigma}_{\mathbf z})\ge ch^2>0$, taking the positive semidefinite part preserves the leading $d$ eigenvalues and eigenspace. Hence
\[
\tilde s_d\ge ch^2,
\qquad
\tilde s_{d+1}\le C\Delta_{h,t}.
\]
Moreover, by Lemma~\ref{lem:weighted-local-geometry} and Step~2 of the proof of Theorem~\ref{thm:private-hard-correction}, the private local displacement satisfies
\[
\|\tilde{\mathbf r}_{\mathbf z}\|
\le\|\bar{\mathbf y}_{\mathbf z}-\mathbf z\|+\|\tilde{\boldsymbol\mu}_{\mathbf z}-\bar{\mathbf y}_{\mathbf z}\|
=\frac{\|\mathbf m_{1,\mathbf z}\|}{m_{0,\mathbf z}}+\|\tilde{\boldsymbol\mu}_{\mathbf z}-\bar{\mathbf y}_{\mathbf z}\|
\le Ch+ChR_t\le Ch.
\]

\medskip
\noindent\emph{Step 2. Spectral comparison with the hard correction.}
By \eqref{eq:hard-correction} and \eqref{eq:soft-correction}, we have, for every $\lambda>0$,
\[
\tilde{\mathbf z}_{\mathrm{soft}}(\lambda)-\tilde{\mathbf z}_{\mathrm{hard}}
=[\mathbf W_{\lambda,\mathbf z}-(\mathbf I-\widetilde{\mathbf P}_{\mathbf z})]\tilde{\mathbf r}_{\mathbf z}.
\]
Since $\mathbf W_{\lambda,\mathbf z}$ and $\widetilde{\mathbf P}_{\mathbf z}$ share an eigenbasis of $\widetilde{\boldsymbol\Sigma}_{\mathbf z}$,
\[
\mathbf W_{\lambda,\mathbf z}-(\mathbf I-\widetilde{\mathbf P}_{\mathbf z})
=\mathbf W_{\lambda,\mathbf z}\widetilde{\mathbf P}_{\mathbf z}-(\mathbf I-\mathbf W_{\lambda,\mathbf z})(\mathbf I-\widetilde{\mathbf P}_{\mathbf z}).
\]
On the leading $d$ eigendirections ($1\le j\le d$), $\tilde s_j\ge\tilde s_d\ge ch^2$, so
\[
\|\mathbf W_{\lambda,\mathbf z}\widetilde{\mathbf P}_{\mathbf z}\|
=\max_{1\le j\le d}\frac{\lambda^2}{\tilde s_j^2+\lambda^2}
=\frac{\lambda^2}{\tilde s_d^2+\lambda^2}
\le\biggl(\frac{\lambda}{\tilde s_d}\biggr)^2
\le C\biggl(\frac{\lambda}{h^2}\biggr)^2.
\]
On the remaining $D-d$ eigendirections ($d+1\le j\le D$), $\tilde s_j\le\tilde s_{d+1}\le C\Delta_{h,t}$, so
\[
\|(\mathbf I-\mathbf W_{\lambda,\mathbf z})(\mathbf I-\widetilde{\mathbf P}_{\mathbf z})\|
=\max_{d+1\le j\le D}\frac{\tilde s_j^2}{\tilde s_j^2+\lambda^2}
\le\frac{\tilde s_{d+1}^2}{\tilde s_{d+1}^2+\lambda^2}
\le C\biggl(\frac{\Delta_{h,t}}{\lambda}\biggr)^2.
\]
Combining these two spectral estimates with $\|\tilde{\mathbf r}_{\mathbf z}\|\le Ch$ from Step~1 gives, for every $\lambda>0$,
\[
\|\tilde{\mathbf z}_{\mathrm{soft}}(\lambda)-\tilde{\mathbf z}_{\mathrm{hard}}\|
\le Ch[(\lambda/h^2)^2+(\Delta_{h,t}/\lambda)^2].
\]
Step~3 of the proof of Theorem~\ref{thm:private-hard-correction} gives, on this event, $\|\tilde{\mathbf z}_{\mathrm{hard}}-\mathbf z^\star\|\le C\bigl(h^2/\tau+\sigma+h\varsigma(\sqrt D+\sqrt t)/(h/\tau)^d\bigr)$. The triangle inequality yields the claimed bound simultaneously for all $\lambda>0$.

\medskip
{\noindent\emph{Step 3. Oracle parameter choice.}
Taking $\lambda_{\mathrm{or}}:=h\sqrt{\Delta_{h,t}}$ balances the two filter-approximation terms in Step~2:
\[
h\biggl(\frac{\lambda_{\mathrm{or}}}{h^2}\biggr)^2
=h\biggl(\frac{\Delta_{h,t}}{\lambda_{\mathrm{or}}}\biggr)^2
=\frac{\Delta_{h,t}}{h}
\le c\biggl(\frac{h^2}{\tau}+\sigma\biggr)+\frac{h\varsigma(\sqrt D+\sqrt t)}{(h/\tau)^d},
\]}
where the inequality uses $h/\tau+\sigma/h\le R_t\le c$. Combining with Step~2 completes the proof.
\qed

\subsection{Proof of Corollary~\ref{cor:adaptive-soft-correction}}
\label{app:adaptive-soft-proof}

Apply Lemmas~\ref{lem:weighted-local-geometry} and~\ref{lem:moment-gaussian-noise} with $t+\log3$ in place of $t$. By the arguments in the proofs of Theorems~\ref{thm:private-tangent-accuracy}, \ref{thm:private-hard-correction}, and~\ref{thm:private-soft-correction}, their joint event has probability at least $1-\frac23e^{-t}$ and, after adjusting constants, satisfies
\[
\begin{aligned}
&\tilde q_{\mathbf z}\ge c(h/\tau)^d,
\qquad
\|\tilde{\mathbf r}_{\mathbf z}\|\le Ch,
\qquad
\|\widetilde{\boldsymbol\Sigma}_{\mathbf z}\|\le Ch^2,\\
&ch^2\le\tilde s_d\le\tilde s_1\le Ch^2,
\qquad
\tilde s_{d+1}\le C\Delta_{h,t}.
\end{aligned}
\]
The hard-correction bound also holds on this event. We work on it throughout.

\medskip
\noindent\emph{Step 1. Finite-sample control of the auxiliary noise scale.}
Condition on the released summaries. By the definition of $\mathbf L_{\mathbf z}$, $\|\tilde{\mathbf r}_{\mathbf z}\|\le Ch$, $\|\widetilde{\boldsymbol\Sigma}_{\mathbf z}\|\le Ch^2$, and $\|\operatorname{smat}(\mathbf g_2)\|\le\|\operatorname{smat}(\mathbf g_2)\|_{\mathrm F}=\|\mathbf g_2\|$, we have
\[
\|\mathbf L_{\mathbf z}(\mathbf g)-\mathbf L_{\mathbf z}(\mathbf g')\|
\le Ch^2\|\mathbf g-\mathbf g'\|.
\]
Integrating the matrix tail bound in Lemma~\ref{lem:moment-gaussian-noise} gives $\mathbb E\|\operatorname{smat}(\mathbf g_2)\|^2\le CD$. Together with $\mathbb E\|\mathbf g_1\|^2=D$ and $\mathbb E g_0^2=1$, this yields
\[
\mathbb E_{\mathbf g}\|\mathbf L_{\mathbf z}(\mathbf g)\|^2
\le Ch^4\bigl(
\mathbb E\|\operatorname{smat}(\mathbf g_2)\|^2
+\mathbb E\|\mathbf g_1\|^2+\mathbb E g_0^2\bigr)
\le Ch^4D.
\]
For $\mathbf G:=(\mathbf g^{(1)},\dots,\mathbf g^{(B)})$, let $F_B(\mathbf G):=(B^{-1}\sum_{b=1}^B\|\mathbf L_{\mathbf z}(\mathbf g^{(b)})\|^2)^{1/2}$. The Lipschitz bound on $\mathbf L_{\mathbf z}$ implies that $F_B$ is $Ch^2/\sqrt B$-Lipschitz, and Jensen's inequality yields $\mathbb E F_B\le Ch^2\sqrt D$. Gaussian concentration for Lipschitz functions \citep[Theorem~5.2.2]{vershynin2018high} therefore gives, with conditional probability at least $1-\frac13e^{-t}$,
\[
F_B\le Ch^2\bigl(\sqrt D+\sqrt{t/B}\bigr),
\qquad
\nu_{\mathbf z}
\le C\frac{h^2\varsigma(\sqrt D+\sqrt{t/B})}{(h/\tau)^d}
\le C\Delta_{h,t}.
\]
Intersecting this conditional event with the high-probability event above yields probability at least $1-e^{-t}$.

\medskip
\noindent\emph{Step 2. Identification of the lower spectral group.}
Since
\[
\frac{\Delta_{h,t}}{h^2}
=\biggl(\frac h\tau+\frac\sigma h\biggr)^2
+\frac{\varsigma(\sqrt D+\sqrt t)}{(h/\tau)^d}
\le CR_t,
\]
the spectral bounds and Step~1 imply $b_1\asymp h^2$ and, for constants $\kappa\in(0,1)$ and $C_0>0$,
\[
\frac{b_d}{b_1}\ge\kappa,
\qquad
\frac{b_{d+1}}{b_1}\le C_0R_t.
\]
Choose $\alpha_*<\kappa$. For each fixed $\alpha_0\in(0,\alpha_*)$, choose $c_{\alpha_0}$ sufficiently small that
$C_0R_t<\min(\alpha_0,\kappa^2)$ whenever $R_t\le c_{\alpha_0}$.
Because $b_D\le b_{d+1}$, we then have
\[
b_{d+1}<\alpha_0b_1\le\sqrt{b_1b_D}\vee\alpha_0b_1\le b_1\max(\sqrt{C_0R_t},\alpha_0)<\kappa b_1\le b_d.
\]
Hence $b_j<\sqrt{b_1b_D}\vee\alpha_0 b_1$ if and only if $j\ge d+1$, so that
\[
\widehat b_{*,\mathbf z}=b_{d+1}
=\tilde s_{d+1}\vee\nu_{\mathbf z}
\le C\Delta_{h,t}.
\]

\medskip
\noindent\emph{Step 3. Comparison with the hard correction.}
Suppose first that $\widehat b_{*,\mathbf z}>0$. Keeping the actual eigenvalues in the spectral comparison from Step~2 of the proof of Theorem~\ref{thm:private-soft-correction} gives
\[
\begin{aligned}
\|\tilde{\mathbf z}_{\mathrm{soft}}(\widehat\lambda_{\mathbf z})
-\tilde{\mathbf z}_{\mathrm{hard}}\|
&\le Ch\biggl(
\frac{\widehat\lambda_{\mathbf z}^{\,2}}{\tilde s_d^2}
+\frac{\tilde s_{d+1}^{\,2}}{\widehat\lambda_{\mathbf z}^{\,2}}
\biggr)
\le C\biggl(
\frac{\widehat b_{*,\mathbf z}}h
+\frac{\tilde s_{d+1}^{\,2}}{h\widehat b_{*,\mathbf z}}
\biggr)\\
&\le C\frac{\widehat b_{*,\mathbf z}}h
\le C\frac{\Delta_{h,t}}h.
\end{aligned}
\]
We used $\widehat\lambda_{\mathbf z}^{\,2}=b_1\widehat b_{*,\mathbf z}$, $b_1\asymp h^2$, and $\tilde s_d\asymp h^2$ for the second inequality, followed by $\tilde s_{d+1}\le\widehat b_{*,\mathbf z}$ and $\widehat b_{*,\mathbf z}\le C\Delta_{h,t}$. In the degenerate case $\widehat b_{*,\mathbf z}=0$, we have $\nu_{\mathbf z}=\tilde s_{d+1}=0<\tilde s_d$, so $\mathbf W_{0,\mathbf z}=\mathbf I-\widetilde{\mathbf P}_{\mathbf z}$ and the two corrections coincide.

Since $h/\tau+\sigma/h\le R_t\le c_{\alpha_0}$, we have $\frac1h(h^2/\tau+\sigma)^2\le c_{\alpha_0}(h^2/\tau+\sigma)$, so that
\[
\frac{\Delta_{h,t}}h
=\frac1h\biggl(\frac{h^2}{\tau}+\sigma\biggr)^2
+\frac{h\varsigma(\sqrt D+\sqrt t)}{(h/\tau)^d}
\le C\biggl(
\frac{h^2}{\tau}+\sigma
+\frac{h\varsigma(\sqrt D+\sqrt t)}{(h/\tau)^d}
\biggr).
\]
Combining this with the hard-correction bound on the event fixed above proves the claim.
\qed

\subsection{Extension to manifolds with boundary}
\label{app:boundary-extension}

To establish Remark~\ref{rem:interior-boundary}, we retain the model assumptions except that $\mathcal M$ may have a boundary. Suppose $\operatorname{dist}(\mathbf z^\star,\partial\mathcal M)\ge\kappa\tau$ for a fixed $\kappa>0$, and allow constants below to depend on $\kappa$. Set $r_\kappa:=c\min(1,\kappa)\tau$, where $c>0$ is an absolute constant small enough that $4r_\kappa<\kappa\tau$ and $4r_\kappa\le\tau/8$.

\medskip
\noindent\emph{Step 1. Interior local graph and volume bounds.}
Fix $\mathbf x\in\mathcal M\cap B_D(\mathbf z^\star,r_\kappa)$. Every point on the line segment joining $\mathbf z^\star$ and $\mathbf x$ lies within $r_\kappa/2\le\tau/2$ of $\mathcal M$. By Lemma~\ref{lem:projection-lipschitz} with $r=\tau/2$, projecting this segment onto $\mathcal M$ gives a curve of length at most $2r_\kappa$. By compactness, a minimizing path in $\mathcal M$ exists and is no longer than this curve. Since $2r_\kappa<\operatorname{dist}(\mathbf z^\star,\partial\mathcal M)$, it avoids $\partial\mathcal M$ and is therefore an interior geodesic.

Since $4r_\kappa<\operatorname{dist}(\mathbf z^\star,\partial\mathcal M)$, geodesics from $\mathbf z^\star$ remain in the interior up to length $4r_\kappa$. The exponential map is thus well-defined on $T_{\mathbf z^\star}\mathcal M\cap B_D(0,4r_\kappa)$. The argument of Lemma~\ref{lem:local-graph} gives a graph over $T_{\mathbf z^\star}\mathcal M\cap B_D(0,2r_\kappa)$ with the same derivative and quadratic remainder bounds. For $0<r\le r_\kappa$, any $\mathbf x\in\mathcal M\cap B_D(\mathbf z^\star,r)$ is reached by an interior geodesic of length at most $2r_\kappa$. Hence $\mathbf x$ lies on the graph with tangent coordinate in $B_D(0,r)$. By the area formula,
\[
\mathcal H^d\bigl(\mathcal M\cap B_D(\mathbf z^\star,r)\bigr)
\le C_d r^d,
\qquad 0<r\le r_\kappa.
\]

\medskip
\noindent\emph{Step 2. Weighted local moment bounds.}
We adapt the proof of Lemma~\ref{lem:weighted-local-geometry}, replacing $\rho\tau/2$ with $r_\kappa$. By Step~1, the shell argument applies inside $B_D(\mathbf z^\star,r_\kappa)$ and gives $\mathbb E[Z_i\mathbb I(r_i\le r_\kappa)]\le C(h/\tau)^d$. Outside $B_D(\mathbf z^\star,r_\kappa)$, the envelope $Z_i$ satisfies
\[
Z_i\le C\exp(-c r_\kappa^2/h^2)\le C_\kappa(h/\tau)^d.
\]
Hence $\mathbb E Z_i\le C_\kappa(h/\tau)^d$ and $\mathbb E Z_i^2\le C\mathbb E Z_i$. Bernstein's inequality then yields the empirical moment upper bounds.

For the lower bounds, choose $h/\tau\le c_\kappa$ small enough that the opposite caps $\mathcal C_{\mathbf v}^\pm$ lie in $\mathcal M\cap B_D(\mathbf z^\star,r_\kappa)$ on the interior graph from Step~1. The Chernoff argument from Lemma~\ref{lem:weighted-local-geometry} then gives $m_{0,\mathbf z}\ge c_\kappa(h/\tau)^d$. Together with the upper bound on $m_{0,\mathbf z}$, the opposite-cap bound in that proof yields $\lambda_d(\mathbf A)\ge c_\kappa h^2$.

\medskip
\noindent\emph{Step 3. Normal bounds and denoising rates.}
Since $\mathbf z^\star$ is an interior point, its tangent cone is the linear subspace $T_{\mathbf z^\star}\mathcal M$. It follows that $\mathbf z-\mathbf z^\star\in N_{\mathbf z^\star}\mathcal M$. Applying Theorem~4.18 of \cite{federer1959curvature} at $\mathbf z^\star$ gives
\[
\|\boldsymbol\Pi_\star^\perp(\mathbf x-\mathbf z^\star)\|
\le \frac{\|\mathbf x-\mathbf z^\star\|^2}{2\tau},
\qquad \mathbf x\in\mathcal M.
\]
Hence the bounds on $\|\widehat{\boldsymbol\Sigma}_{\mathbf z}-\mathbf A\|$ and the normal deviation of $\bar{\mathbf y}_{\mathbf z}$ remain valid. With Steps~1--3 established, the perturbation and spectral arguments go through with $\kappa$-dependent constants, so Theorems~\ref{thm:private-tangent-accuracy}--\ref{thm:private-soft-correction} and Corollary~\ref{cor:adaptive-soft-correction} continue to hold as stated.
\qed

\section{Shared-field construction, privacy, and utility}
\label{app:shared-private-fields}

\subsection{Spectral representation and sensitivity}

Recall the spectral density $\rho_h$ and mixture $\bar\rho$ from Section~\ref{subsec:bandwidth-sharing}, under which
\[
k_h(\mathbf v)=\int_{\mathbb R^D}\rho_h(\boldsymbol\omega)e^{i\boldsymbol\omega^\top\mathbf v}\,\mathrm d\boldsymbol\omega.
\]

\begin{lemma}
\label{lem:finite-mixture-admissibility}
For each $h\in\mathcal H_J$ and polynomial $P(\boldsymbol\omega)$,
\[
\int_{\mathbb R^D}\bigl(1+|P(\boldsymbol\omega)|\bigr)
\frac{\rho_h(\boldsymbol\omega)^2}
{\bar\rho(\boldsymbol\omega)}
\,\mathrm d\boldsymbol\omega<\infty.
\]
\end{lemma}

\begin{proof}
Since $\bar\rho(\boldsymbol\omega)\ge\rho_{h_1}(\boldsymbol\omega)/J$, we have
\[
\frac{\rho_h(\boldsymbol\omega)^2}{\bar\rho(\boldsymbol\omega)}
\le J\frac{\rho_h(\boldsymbol\omega)^2}{\rho_{h_1}(\boldsymbol\omega)}
\lesssim\exp\bigl(-(h^2-h_1^2/2)\|\boldsymbol\omega\|^2\bigr).
\]
Since $h^2-h_1^2/2\ge h_1^2/2>0$, the right-hand side decays at a Gaussian rate and is therefore integrable against $1+|P(\boldsymbol\omega)|$.
\end{proof}

Applying $\mathcal A_{\mathbf z,h}$ to $\mathbf z\mapsto k_h(\mathbf y-\mathbf z)$ recovers the single-record feature from the proof of Proposition~\ref{prop:local-moment-sensitivity}:
\[
\boldsymbol\phi_{\mathbf z,h}(\mathbf y):=\mathcal A_{\mathbf z,h}[k_h(\mathbf y-\cdot)]\in\mathbb R^p,
\qquad
\mathbf F_{\mathbf z,h}(\mathcal D)=\frac1n\sum_{i=1}^n\boldsymbol\phi_{\mathbf z,h}(\mathbf y_i).
\]

\begin{lemma}
\label{lem:shared-whitened-sensitivity}
Let $\rho$ be a strictly positive, even probability density on $\mathbb R^D$ with nonnegative characteristic function, and fix finitely many distinct pairs $\{(\mathbf z_j,h_k)\}$ such that
\[
\int_{\mathbb R^D}(1+\|\boldsymbol\omega\|^4)
\frac{\rho_{h_k}(\boldsymbol\omega)^2}{\rho(\boldsymbol\omega)}
\,\mathrm d\boldsymbol\omega<\infty
\]
for each $h_k$. Let $G$ be a centered Gaussian field with covariance \eqref{eq:shared-scale-covariance}, where $\rho$ replaces $\bar\rho$, and let $\mathbf K$ be the covariance of the stacked perturbations $\mathbf G_{\mathbf z_j,h_k}:=\mathcal A_{\mathbf z_j,h_k}G(\cdot,h_k)$. If $\mathbf F(\mathcal D)$ stacks $\mathbf F_{\mathbf z_j,h_k}(\mathcal D)$, then $\mathbf K$ is positive definite and, for any adjacent datasets $\mathcal D,\mathcal D'$,
\[
\bigl\|\mathbf K^{-1/2}\bigl(\mathbf F(\mathcal D)-\mathbf F(\mathcal D')\bigr)\bigr\|
\le\frac{\sqrt2}{n}.
\]
\end{lemma}

\begin{proof}
\noindent\emph{Step 1. Hilbert-space factorization.}
Let $\mathbb H$ be the real Hilbert space of square-integrable functions $u:\mathbb R^D\to\mathbb C$ satisfying $u(-\boldsymbol\omega)=\overline{u(\boldsymbol\omega)}$, equipped with $\langle u,v\rangle_{\mathbb H}:=\operatorname{Re}\int_{\mathbb R^D} u(\boldsymbol\omega)\overline{v(\boldsymbol\omega)}\,\mathrm d\boldsymbol\omega$. For each $\mathbf y\in\mathbb R^D$, set $a_{\mathbf y}(\boldsymbol\omega):=\sqrt{\rho(\boldsymbol\omega)}\,e^{i\boldsymbol\omega^\top\mathbf y}\in\mathbb H$. For each pair $(\mathbf z,h)$, define
\[
\begin{aligned}
r_{0,\mathbf z,h}(\boldsymbol\omega)
&:=\frac{\rho_h(\boldsymbol\omega)}{\sqrt{\rho(\boldsymbol\omega)}}e^{i\boldsymbol\omega^\top\mathbf z},\\
r_{a,\mathbf z,h}(\boldsymbol\omega)
&:=ih\omega_a r_{0,\mathbf z,h}(\boldsymbol\omega),
\qquad a\in[D],\\
r_{\mathbf E,\mathbf z,h}(\boldsymbol\omega)
&:=\frac{\operatorname{tr}(\mathbf E)-h^2\boldsymbol\omega^\top\mathbf E\boldsymbol\omega}{\sqrt2}r_{0,\mathbf z,h}(\boldsymbol\omega),
\qquad \mathbf E\in\mathbb S^D.
\end{aligned}
\]
Note that these three coordinate functions lie in $\mathbb H$ by evenness of $\rho,\rho_h$ and the integrability condition. It is easy to see that
\[
\langle r_{0,\mathbf z,h},a_{\mathbf y}\rangle_{\mathbb H}=k_h(\mathbf y-\mathbf z),
\qquad
\langle r_{0,\mathbf z,h},r_{0,\mathbf z',h'}\rangle_{\mathbb H}=\mathbb E[G(\mathbf z,h)G(\mathbf z',h')].
\]

Write $\{(\mathbf z_\ell,h_\ell)\}_{\ell=1}^M$ for the pairs. Applying $\mathcal A_{\mathbf z,h}$ with respect to $\mathbf z$ extends these identities to $r_{a,\mathbf z,h}$ and $r_{\mathbf E,\mathbf z,h}$ and yields $\boldsymbol\phi(\mathbf y)=Ta_{\mathbf y}$ and $\mathbf K=TT^*$, where $T:\mathbb H\to\mathbb R^{Mp}$ stacks the $\mathbb H$-inner products with the $Mp$ coordinate functions and $\boldsymbol\phi(\mathbf y)\in\mathbb R^{Mp}$ stacks $\boldsymbol\phi_{\mathbf z_\ell,h_\ell}(\mathbf y)$.

\medskip
\noindent\emph{Step 2. Positive definiteness of $\mathbf K$.}
Suppose $\mathbf K\mathbf c=\mathbf 0$. Then $\|T^*\mathbf c\|_{\mathbb H}^2=\mathbf c^\top TT^*\mathbf c=0$, so $T^*\mathbf c=0$ and $\mathbf c^\top\boldsymbol\phi(\mathbf y)=\langle T^*\mathbf c,a_{\mathbf y}\rangle_{\mathbb H}=0$ for all $\mathbf y\in\mathbb R^D$. Partition $\mathbf c$ into blocks $(a_\ell,\mathbf b_\ell,\operatorname{svec}(\mathbf E_\ell))$ with $\mathbf E_\ell\in\mathbb S^D$. Expanding $\mathbf c^\top\boldsymbol\phi(\mathbf y)=0$ gives
\[
\sum_{\ell=1}^M e^{-\|\mathbf z_\ell\|^2/(2h_\ell^2)}
Q_\ell\Bigl(\frac{\mathbf y-\mathbf z_\ell}{h_\ell}\Bigr)
\exp\Bigl(-\frac{\|\mathbf y\|^2}{2h_\ell^2}+\frac{\mathbf z_\ell^\top\mathbf y}{h_\ell^2}\Bigr)=0,
\qquad \mathbf y\in\mathbb R^D,
\]
where $Q_\ell(\mathbf u):=a_\ell+\mathbf b_\ell^\top\mathbf u+\mathbf u^\top\mathbf E_\ell\mathbf u/\sqrt2$. Along a generic ray $\mathbf y=t\mathbf v$, each $Q_\ell\bigl((t\mathbf v-\mathbf z_\ell)/h_\ell\bigr)$ with $Q_\ell\not\equiv 0$ remains a nonzero polynomial in $t$. Moreover, since the pairs $(\mathbf z_\ell,h_\ell)$ are distinct, the exponents $-\|\mathbf y\|^2/(2h_\ell^2)+\mathbf z_\ell^\top\mathbf y/h_\ell^2$ have distinct asymptotics. The identity therefore forces $Q_\ell\equiv 0$ for every $\ell\in[M]$. Hence $\mathbf c=\mathbf 0$ and $\mathbf K\succ\mathbf 0$.

\medskip
{\emergencystretch=2em \noindent\emph{Step 3. Whitened sensitivity bound.}
Since $T^*(TT^*)^{-1}T$ is the orthogonal projection onto $\operatorname{range}(T^*)\subset\mathbb H$, every $v\in\mathbb H$ satisfies
\[
\|\mathbf K^{-1/2}Tv\|^2
=\langle v,T^*(TT^*)^{-1}Tv\rangle_{\mathbb H}
\le\|v\|_{\mathbb H}^2.
\]}
If $\mathcal D'$ replaces $\mathbf y_i$ in $\mathcal D$ with $\mathbf y_i'$, then $\mathbf F(\mathcal D)-\mathbf F(\mathcal D')=n^{-1}T(a_{\mathbf y_i}-a_{\mathbf y_i'})$. Because $\rho$ integrates to $1$ and has a nonnegative characteristic function, $\|a_{\mathbf y_i}\|_{\mathbb H}^2=\|a_{\mathbf y_i'}\|_{\mathbb H}^2=1$ and $\langle a_{\mathbf y_i},a_{\mathbf y_i'}\rangle_{\mathbb H}\ge 0$. Hence $\|a_{\mathbf y_i}-a_{\mathbf y_i'}\|_{\mathbb H}^2\le 2$, yielding the bound $\sqrt2/n$.
\end{proof}

\begin{lemma}
\label{lem:shared-block-covariance}
For any $t>0$, $a_0\in\mathbb R$, $\mathbf b\in\mathbb R^D$, and $\mathbf E\in\mathbb S^D$, if $\mathbf X\sim\mathcal N(\mathbf 0,t\mathbf I_D)$, then
\[
\mathbb E\biggl|
a_0+i\mathbf b^\top\mathbf X
+\frac{\operatorname{tr}(\mathbf E)-\mathbf X^\top\mathbf E\mathbf X}{\sqrt2}
\biggr|^2
=\biggl(a_0+\frac{(1-t)\operatorname{tr}(\mathbf E)}{\sqrt2}\biggr)^2
+t\|\mathbf b\|^2+t^2\|\mathbf E\|_{\mathrm F}^2.
\]
Consequently, under the setup of Lemma~\ref{lem:shared-whitened-sensitivity}, every pair $(\mathbf z,h)$ satisfies $[\mathbf K]_{(\mathbf z,h),(\mathbf z,h)}=\mathbf I_{p}$ when $\rho=\rho_h$, and $[\mathbf K]_{(\mathbf z,h),(\mathbf z,h)}\preceq Z\mathbf I_{p}$ whenever $\rho_h\le Z\rho$.
\end{lemma}

\begin{proof}
We have $\mathbb E[(\mathbf b^\top\mathbf X)^2]=t\|\mathbf b\|^2$, $\mathbb E[\mathbf X^\top\mathbf E\mathbf X]=t\operatorname{tr}(\mathbf E)$, and $\operatorname{Var}(\mathbf X^\top\mathbf E\mathbf X)=2t^2\|\mathbf E\|_{\mathrm F}^2$. The moment identity follows by separating the real and imaginary parts.

{\emergencystretch=2em For any $(\mathbf z,h)$, the $\mathbb H$-inner product representation in Lemma~\ref{lem:shared-whitened-sensitivity} gives the variance of $(a_0,\mathbf b,\operatorname{svec}(\mathbf E))^\top\mathbf G_{\mathbf z,h}$ as
\begin{equation}
\label{eq:block-variance-integral}
\int_{\mathbb R^D}\frac{\rho_h(\boldsymbol\omega)^2}{\rho(\boldsymbol\omega)}
\biggl|
a_0+ih\mathbf b^\top\boldsymbol\omega
+\frac{\operatorname{tr}(\mathbf E)-h^2\boldsymbol\omega^\top\mathbf E\boldsymbol\omega}{\sqrt2}
\biggr|^2\,\mathrm d\boldsymbol\omega.
\end{equation}}
When $\rho=\rho_h$, the weight in \eqref{eq:block-variance-integral} reduces to the $\mathcal N(\mathbf 0,h^{-2}\mathbf I_D)$ density $\rho_h$, so the moment identity yields $a_0^2+\|\mathbf b\|^2+\|\mathbf E\|_{\mathrm F}^2=\|(a_0,\mathbf b,\operatorname{svec}(\mathbf E))\|^2$. It follows that $[\mathbf K]_{(\mathbf z,h),(\mathbf z,h)}=\mathbf I_p$. When $\rho_h\le Z\rho$, the bound $\rho_h^2/\rho\le Z\rho_h$ yields $[\mathbf K]_{(\mathbf z,h),(\mathbf z,h)}\preceq Z\mathbf I_p$.
\end{proof}

\subsection{Proofs of the privacy guarantees}
\label{app:shared-privacy-proofs}

\medskip
\noindent\textbf{Proof of Proposition~\ref{prop:shared-fixed-bandwidth}.}
Take $\rho=\rho_h$, whose characteristic function is $k_h(\mathbf v)>0$. By Lemmas~\ref{lem:shared-whitened-sensitivity} and~\ref{lem:shared-block-covariance}, $\mathbf K_h\succ\mathbf 0$, $[\mathbf K_h]_{jj}=\mathbf I_p$, and $\|\mathbf K_h^{-1/2}(\mathbf F_h(\mathcal D)-\mathbf F_h(\mathcal D'))\|\le\sqrt2/n$. Whitening \eqref{eq:shared-moment-release} by $\varsigma_{\mathrm{sh}}^{-1}\mathbf K_h^{-1/2}$ reduces it to a standard Gaussian mechanism with $\ell_2$ sensitivity at most $\sqrt2/(n\varsigma_{\mathrm{sh}})=\mu$, so \eqref{eq:gaussian-calibration} gives $(\varepsilon,\delta)$-DP.

For the function-valued extension, $G_h$ admits a version with sample paths in $C^2(\mathbb R^D)$. Its Cameron--Martin space is the RKHS $\mathbb H_h$ of its covariance kernel $k_h$. If $\mathcal D'$ replaces $\mathbf y_i$ in $\mathcal D$ with $\mathbf y_i'$, the reproducing property in $\mathbb H_h$ gives
\[
\|m_{0,\cdot,h}(\mathcal D)-m_{0,\cdot,h}(\mathcal D')\|_{\mathbb H_h}^2
=\frac{2-2k_h(\mathbf y_i-\mathbf y_i')}{n^2}
\le\frac{2}{n^2}.
\]
Let $P$ and $P'$ denote the laws of $m_{0,\cdot,h}(\mathcal D)+\varsigma_{\mathrm{sh}}G_h$ and $m_{0,\cdot,h}(\mathcal D')+\varsigma_{\mathrm{sh}}G_h$, respectively, and set $\mu_0:=\|m_{0,\cdot,h}(\mathcal D)-m_{0,\cdot,h}(\mathcal D')\|_{\mathbb H_h}/\varsigma_{\mathrm{sh}}\le\mu$. If $\mu_0=0$, then $P=P'$ and the privacy inequality is immediate. Suppose $\mu_0>0$. 

By the Cameron--Martin formula \citep[Corollary~2.4.3]{bogachev1998gaussian}, $\frac{\mathrm dP}{\mathrm dP'}=\exp(\mu_0 Z-\mu_0^2/2)$ with $Z\sim\mathcal N(0,1)$ under $P'$, reducing the comparison to $\mathcal N(\mu_0,1)$ versus $\mathcal N(0,1)$. Hence, for every measurable set $S$,
\[
\begin{aligned}
P(S)-e^\varepsilon P'(S)
&=\mathbb E_{P'}\Bigl[\mathbb I(S)\bigl(e^{\mu_0 Z-\mu_0^2/2}-e^\varepsilon\bigr)\Bigr]\\
&\le\mathbb E_{P'}\Bigl[\mathbb I(\mu_0 Z-\mu_0^2/2>\varepsilon)\bigl(e^{\mu_0 Z-\mu_0^2/2}-e^\varepsilon\bigr)\Bigr]\\
&=\mathbb P\Bigl(\mu_0 Z+\frac{\mu_0^2}{2}>\varepsilon\Bigr)
-e^\varepsilon\mathbb P\Bigl(\mu_0 Z-\frac{\mu_0^2}{2}>\varepsilon\Bigr)
\le\delta,
\end{aligned}
\]
where the last inequality follows from $\mu_0\le\mu$ and \eqref{eq:gaussian-calibration}. By symmetry between $\mathcal D$ and $\mathcal D'$, $m_{0,\cdot,h}+\varsigma_{\mathrm{sh}}G_h$ is $(\varepsilon,\delta)$-DP. Applying $\mathcal A_{\mathbf z,h}$ preserves $(\varepsilon,\delta)$-DP for $\widetilde{\mathbf F}_{\cdot,h}$ by post-processing.
\qed

\medskip
\noindent\textbf{Proof of Proposition~\ref{prop:shared-moment-field}.}

\noindent\emph{Step 1. Function-valued privacy.}
By Lemma~\ref{lem:finite-mixture-admissibility}, $\rho_{h_k}^2/\bar\rho$ has finite polynomial moments of all orders for each $k\in[J]$. Hence $(G(\cdot,h_k))_{k=1}^J$ admits a version with sample paths in $C^2(\mathbb R^D)^J$.

If $\mathcal D'$ replaces $\mathbf y_i$ in $\mathcal D$ with $\mathbf y_i'$, set $u:=n^{-1}\bigl(\bar k(\mathbf y_i-\cdot)-\bar k(\mathbf y_i'-\cdot)\bigr)\in\mathbb H_{\bar\rho}$. By the reproducing property in $\mathbb H_{\bar\rho}$,
\[
\|u\|_{\mathbb H_{\bar\rho}}^2
=\frac{2-2\bar k(\mathbf y_i-\mathbf y_i')}{n^2}
\le\frac{2}{n^2},
\qquad
m_{0,\mathbf z,h_k}(\mathcal D)-m_{0,\mathbf z,h_k}(\mathcal D')
=\langle u,k_{h_k}(\mathbf z-\cdot)\rangle_{\mathbb H_{\bar\rho}}.
\]
Since \eqref{eq:shared-scale-covariance} equals $\langle k_h(\mathbf z-\cdot),k_{h'}(\mathbf z'-\cdot)\rangle_{\mathbb H_{\bar\rho}}$ and the averages $J^{-1}\sum_{\ell=1}^J k_{h_\ell}(\mathbf z-\cdot)=\bar k(\mathbf z-\cdot)$ have dense span in $\mathbb H_{\bar\rho}$, the Cameron--Martin norm of the joint mean shift equals $\|u\|_{\mathbb H_{\bar\rho}}\le\sqrt2/n$. Since $\|u\|_{\mathbb H_{\bar\rho}}/\varsigma_{\mathrm{sh}}\le\mu$, the same Cameron--Martin argument as in the proof of Proposition~\ref{prop:shared-fixed-bandwidth} shows that the joint release
\[
\bigl(m_{0,\cdot,h_k}(\mathcal D)+\varsigma_{\mathrm{sh}}G(\cdot,h_k)\bigr)_{k=1}^J
\]
{is $(\varepsilon,\delta)$-DP. Applying $\mathcal A_{\mathbf z,h_k}$ to each component preserves $(\varepsilon,\delta)$-DP for \eqref{eq:finite-moment-release} by post-processing.\par}

\medskip
\noindent\emph{Step 2. Upper bound on $V$.}
Since $\bar\rho\ge\rho_{h_k}/J$ for every $k\in[J]$, Lemma~\ref{lem:shared-block-covariance} gives $V\le J$.

For $h_k=h_1q^{k-1}$ with $q>1$, set $r_0:=\lfloor c/(\sqrt D\log q)\rfloor$ for a sufficiently small constant $c>0$, and let $N_k:=\{\ell\in[J]:|\ell-k|\le r_0\}$. We have $|N_k|\ge\min(J,r_0+1)$. For $\ell\in N_k$, define
\[
\zeta:=1-(h_\ell/h_k)^2,
\qquad
t:=(1+\zeta)^{-1},
\]
so that $|\zeta|\lesssim D^{-1/2}$ and $D(1-t)^2\lesssim1$.

By definition of $\rho_h$,
\[
\frac{\rho_{h_k}(\boldsymbol\omega)^2}{\rho_{h_\ell}(\boldsymbol\omega)}
=(1-\zeta^2)^{-D/2}\biggl(\frac{h_k}{\sqrt{2\pi t}}\biggr)^D
\exp\!\biggl(-\frac{\|h_k\boldsymbol\omega\|^2}{2t}\biggr).
\]
Since $|\zeta|\lesssim D^{-1/2}$, we have $(1-\zeta^2)^{-D/2}\lesssim1$, while the remaining factor is the density of $\boldsymbol\omega$ when $h_k\boldsymbol\omega\sim\mathcal N(\mathbf 0,t\mathbf I_D)$. Hence Lemma~\ref{lem:shared-block-covariance} bounds \eqref{eq:block-variance-integral} at $h=h_k$ and $\rho=\rho_{h_\ell}$ by
\[
C\bigl[\bigl(a_0+(1-t)\operatorname{tr}(\mathbf E)/\sqrt2\bigr)^2
+t\|\mathbf b\|^2+t^2\|\mathbf E\|_{\mathrm F}^2\bigr]
\lesssim a_0^2+\|\mathbf b\|^2+\|\mathbf E\|_{\mathrm F}^2,
\]
using $|\operatorname{tr}(\mathbf E)|\le\sqrt D\|\mathbf E\|_{\mathrm F}$.

Since $\bar\rho\ge J^{-1}\sum_{\ell\in N_k}\rho_{h_\ell}$,
\[
\frac{\rho_{h_k}^2}{\bar\rho}
\le\frac{J}{|N_k|^2}\sum_{\ell\in N_k}\frac{\rho_{h_k}^2}{\rho_{h_\ell}},
\]
by convexity of $x\mapsto1/x$. Consequently,
\[
V\lesssim\max_{k\in[J]}\frac{J}{|N_k|}
\le\frac{J}{\min(J,r_0+1)}
\lesssim 1+\sqrt D\log(h_J/h_1).
\]
Combining with $V\le J$ gives the upper bound.

\medskip
\noindent\emph{Step 3. Lower bound on $V$.}
Consider a mixture with public scales and weights,
\[
\rho=\sum_{\ell=1}^K w_\ell\rho_{g_\ell},
\qquad
 g_\ell>0,\quad w_\ell\ge0,\quad \sum_{\ell=1}^K w_\ell=1,\quad
\min_{\ell:w_\ell>0}g_\ell<\sqrt2h_1,
\]
and let $V(\rho)$ be the resulting variance factor. By \eqref{eq:block-variance-integral} with $(a_0,\mathbf b,\mathbf E)=(1,\mathbf 0,\mathbf 0)$,
\[
V(\rho)\ge\int_{\mathbb R^D}
\frac{\rho_{h_k}(\boldsymbol\omega)^2}{\rho(\boldsymbol\omega)}
\,\mathrm d\boldsymbol\omega,
\qquad k\in[J].
\]

Choose a subgrid $\{\tilde h_i\}_{i=1}^M\subseteq\mathcal H_J$ with consecutive log-ratios at least $C/\sqrt D$ for a sufficiently large constant $C>0$, such that
\[
M\gtrsim\min\bigl(J,\,1+\sqrt D\log(h_J/h_1)\bigr).
\]
For $M\ge2$, set $r_i:=\sqrt{D/(\tilde h_i\tilde h_{i+1})}$ and partition $\mathbb R^D$ into
\[
A_1:=\{\|\boldsymbol\omega\|\ge r_1\},\quad
A_i:=\{r_i\le\|\boldsymbol\omega\|<r_{i-1}\}\quad(2\le i<M),
\qquad
A_M:=\{\|\boldsymbol\omega\|<r_{M-1}\},
\]
with $A_1:=\mathbb R^D$ if $M=1$. Under $\rho_{\tilde h_i}$, we have $\tilde h_i^2\|\boldsymbol\omega\|^2\sim\chi_D^2$. Since $\tilde h_i^2 r_i^2\le De^{-C/\sqrt D}$ for $i<M$ and $\tilde h_i^2 r_{i-1}^2\ge De^{C/\sqrt D}$ for $i>1$, chi-square concentration gives
\[
\int_{A_i}\rho_{\tilde h_i}(\boldsymbol\omega)
\,\mathrm d\boldsymbol\omega\gtrsim1,
\qquad i\in[M].
\]
By Cauchy--Schwarz on each $A_i$ and $\sum_{i=1}^M\int_{A_i}\rho(\boldsymbol\omega)\,\mathrm d\boldsymbol\omega=1$,
\[
V(\rho)\ge
\sum_{i=1}^M\biggl(\int_{A_i}
\rho_{\tilde h_i}(\boldsymbol\omega)\,\mathrm d\boldsymbol\omega
\biggr)^2
\gtrsim M
\gtrsim\min\bigl(J,\,1+\sqrt D\log(h_J/h_1)\bigr).
\]
Thus $\inf_\rho V(\rho)\asymp V\asymp\min\bigl(J,\,1+\sqrt D\log(h_J/h_1)\bigr)$.
\qed

\subsection{Radial covariance and spectral overlap}
\label{app:radial-spectral-overlap}

\begin{lemma}[Radial covariance representation]
\label{lem:finite-spectral-design}
Define
\[
r_{\ell k}(x):=(h_\ell/h_k)^D
\exp\Bigl(-\frac12(h_\ell^2/h_k^2-1)x\Bigr),
\qquad
a_k(x):=\biggl[J^{-1}\sum_{\ell=1}^Jr_{\ell k}(x)\biggr]^{-1},
\]
and $A_{j,k}:=\mathbb E_{X\sim\chi^2_{D+2j}}[a_k(X)]$ for $j\in\{0,1,2\}$. Then $\mathbf C_k$ is orthogonally equivalent to
\[
\mathbf B_k\oplus A_{1,k}\mathbf I_D
\oplus A_{2,k}\mathbf I_{D(D+1)/2-1}.
\]
The three blocks correspond to the joint zeroth-order and second-order trace coordinates, the first-order coordinates, and the zero-trace second-order coordinates. The last block is absent when $D=1$. Here
\begin{equation}
\label{eq:finite-radial-covariance}
\mathbf B_k=
\begin{pmatrix}
A_{0,k}&\sqrt{\frac D2}(A_{0,k}-A_{1,k})\\[3pt]
\sqrt{\frac D2}(A_{0,k}-A_{1,k})&
\frac D2A_{0,k}-DA_{1,k}+\frac{D+2}{2}A_{2,k}
\end{pmatrix}.
\end{equation}
Moreover, $A_{2,k}\le\lambda_{\max}(\mathbf B_k)$ and $V=\max_{k\in[J]}\bigl\{\lambda_{\max}(\mathbf B_k),\,A_{1,k}\bigr\}$.
\end{lemma}

\begin{proof}
Under $\rho_{h_k}$, write $h_k\boldsymbol\omega=\sqrt X\mathbf U$, where $X\sim\chi_D^2$ and $\mathbf U$ is uniform on the unit sphere in $\mathbb R^D$, independent of $X$, so that $\rho_{h_k}(\boldsymbol\omega)/\bar\rho(\boldsymbol\omega)=a_k(X)$. In $\mathbb S^D$, separate the normalized trace direction $\mathbf I_D/\sqrt D$ from its Frobenius-orthogonal traceless subspace $\{\mathbf E\in\mathbb S^D:\operatorname{tr}(\mathbf E)=0\}$. Along $\mathbf I_D/\sqrt D$, the second-order multiplier in \eqref{eq:block-variance-integral} is $t_D(X):=(D-X)/\sqrt{2D}$, giving the zeroth-order and trace block
\[
\mathbf B_k=\mathbb E\!\left[
a_k(X)
\begin{pmatrix}1&t_D(X)\\t_D(X)&t_D(X)^2\end{pmatrix}
\right].
\]

By spherical symmetry, $\mathbb E[\mathbf U\mathbf U^\top]=\mathbf I_D/D$, while for traceless $\mathbf E,\mathbf E'\in\mathbb S^D$,
\[
\mathbb E[\mathbf U^\top\mathbf E\mathbf U]=0,
\qquad
\mathbb E\bigl[(\mathbf U^\top\mathbf E\mathbf U)(\mathbf U^\top\mathbf E'\mathbf U)\bigr]
=\frac{2\operatorname{tr}(\mathbf E\mathbf E')}{D(D+2)}.
\]
Thus the three blocks are mutually uncorrelated, with first-order and traceless second-order covariances
\[
\mathbb E\biggl[\frac{a_k(X)X}{D}\biggr]\mathbf I_D
\qquad\text{and}\qquad
\mathbb E\biggl[\frac{a_k(X)X^2}{D(D+2)}\biggr]\mathbf I_{D(D+1)/2-1}.
\]
Applying the chi-square density identity
\[
\frac{x^j f_{\chi_D^2}(x)}{\mathbb E[X^j]}=f_{\chi_{D+2j}^2}(x),
\qquad j\in\{1,2\},
\]
yields $A_{1,k}$, $A_{2,k}$, and \eqref{eq:finite-radial-covariance}.

{Finally, for the unit vector $\mathbf u_D:=(D+2)^{-1/2}(\sqrt D,-\sqrt2)^\top$, the identity $\sqrt D-\sqrt2\,t_D(X)=X/\sqrt D$ gives
\[
A_{2,k}=\mathbf u_D^\top\mathbf B_k\mathbf u_D\le\lambda_{\max}(\mathbf B_k),
\]}
and hence the formula for $V$.
\end{proof}

\paragraph{Spectral overlap and high-dimensional scaling.}
Since $r_{kk}(x)=1$, we have
\begin{equation}
\label{eq:shared-spectral-leakage}
\frac{A_{j,k}}{J}
=\mathbb E_{X\sim\chi^2_{D+2j}}\biggl[\frac1{1+\sum_{\ell\ne k}r_{\ell k}(X)}\biggr],
\qquad j\in\{0,1,2\}.
\end{equation}
Note that the integrand equals $\rho_{h_k}(\boldsymbol\omega)/\sum_{\ell=1}^J\rho_{h_\ell}(\boldsymbol\omega)$, so $1-A_{j,k}/J$ measures the spectral overlap of $h_k$ with the other bandwidths.

On the geometric grid $h_k=h_1q^{k-1}$ with fixed $J$ and $q>1$, apply $1-(1+\sum_{\ell\ne k}u_\ell)^{-1}\le\frac12\sum_{\ell\ne k}\sqrt{u_\ell}$ for $u_\ell\ge0$. Taking the chi-square Laplace transform gives
\begin{equation}
\label{eq:shared-exponential-decay}
0\le 1-\frac{A_{j,k}}{J}
\le\frac12\sum_{\ell\ne k}
\biggl(\frac{2q^{|\ell-k|}}{1+q^{2|\ell-k|}}\biggr)^{D/2}
\biggl(\frac{2}{1+q^{2(\ell-k)}}\biggr)^j,
\qquad j\in\{0,1,2\}.
\end{equation}
Since $2q/(1+q^2)<1$, the overlap $1-A_{j,k}/J$ decays exponentially in $D$ and dominates the polynomial factors in \eqref{eq:finite-radial-covariance}. By Lemma~\ref{lem:finite-spectral-design}, it follows that $\max_{k\in[J]}\|\mathbf C_k/J-\mathbf I_p\|\to0$ as $D\to\infty$.

\paragraph{Second-order noise comparison.}
Let $\mathbf C_k^{(2)}$ denote the second-order block of $\mathbf C_k$ of dimension $D(D+1)/2$. By Lemma~\ref{lem:finite-spectral-design}, its average coordinate variance is
\begin{equation}
\label{eq:shared-matrix-noise}
M_k:=\frac{\operatorname{tr}(\mathbf C_k^{(2)})}{D(D+1)/2}
=\frac{A_{0,k}-2A_{1,k}+(D+2)A_{2,k}}{D+1},
\end{equation}
compared with $J$ under independent bandwidth releases, giving a maximal relative root-mean-square noise increase of $\Gamma_{\mathrm{mat}}(D,\mathcal H_J):=\max_{k\in[J]}\bigl(\sqrt{J/M_k}-1\bigr)$. For the grid $\mathcal H_9$ in Section~\ref{sec:sim} ($J=9, q=\sqrt2$), one-dimensional numerical integration yields:
\begin{center}
\small\setlength{\tabcolsep}{4pt}
\begin{tabular}{@{}lrrrrrrrr@{}}
\toprule
$D$ & $5$ & $10$ & $20$ & $28$ & $30$ & $40$ & $50$ & $100$ \\
\midrule
$\Gamma_{\mathrm{mat}}(D,\mathcal H_9)$ (\%) & $101.9$ & $60.6$ & $31.5$ & $20.8$ & $18.9$ & $12.1$ & $8.0$ & $1.3$ \\
\bottomrule
\end{tabular}
\end{center}

\subsection{Simultaneous utility guarantees}
\label{app:shared-utility}

Write $\varsigma_{\mathrm{eff}}$ for $\sqrt V\,\varsigma_{\mathrm{sh}}$ under the shared release of Proposition~\ref{prop:shared-moment-field} and $\sqrt J\,\varsigma_{\mathrm{sh}}$ under independent bandwidth releases. 

\begin{corollary}
\label{cor:shared-field-utility}
Under the assumptions and constants of Theorems~\ref{thm:private-tangent-accuracy}--\ref{thm:private-soft-correction} and Corollary~\ref{cor:adaptive-soft-correction}, suppose $\mathbf z_1,\dots,\mathbf z_m\in\mathcal M_\gamma$ with $\gamma\le c_0h_1$. For $t\ge1$ and $B\ge1$, set $u:=t+\log(3mJ)$. If $n(h_1/\tau)^d\ge C(\eta+u)$ and for $k\in[J]$,
\[
\frac{h_k}{\tau}+\frac{\sigma}{h_k}
+\frac{\varsigma_{\mathrm{eff}}(\sqrt D+\sqrt u)}{(h_k/\tau)^d}\le c,
\]
then, with probability at least $1-e^{-t}$, the bounds in Theorems~\ref{thm:private-tangent-accuracy}--\ref{thm:private-soft-correction} and Corollary~\ref{cor:adaptive-soft-correction} hold simultaneously for all $(j,k)\in[m]\times[J]$, with $(\mathbf z_j,h_k,\varsigma_{\mathrm{eff}},u)$ in place of $(\mathbf z,h,\varsigma,t)$.
\end{corollary}

\begin{proof}
\noindent\emph{Step 1. Simultaneous moment and noise bounds.}
The marginal moment perturbation at each $(\mathbf z_j,h_k)$ is centered Gaussian with covariance bounded by $\varsigma_{\mathrm{eff}}^2\mathbf I_p$, so the tail bounds of Lemma~\ref{lem:moment-gaussian-noise} hold with $\varsigma_{\mathrm{eff}}$ in place of $\varsigma$. Applying Lemma~\ref{lem:weighted-local-geometry} and these bounds at each $(\mathbf z_j,h_k)$ with $u$ in place of $t$ and taking a union bound over $[m]\times[J]$ yields a joint event of probability at least $1-2mJe^{-u}=1-\frac23e^{-t}$, on which the proofs of Theorems~\ref{thm:private-tangent-accuracy}--\ref{thm:private-soft-correction} apply at every $(j,k)$.

\medskip
\noindent\emph{Step 2. Uniform control of the adaptive soft correction.}
On the event from Step~1, conditional on the released moments, the auxiliary draws at $(j,k)$ have covariance bounded by $(\varsigma_{\mathrm{eff}}/\varsigma_{\mathrm{sh}})^2\mathbf I_p$. Applying Step~1 of the proof of Corollary~\ref{cor:adaptive-soft-correction} to their whitened representation with $(\varsigma_{\mathrm{eff}},u)$ in place of $(\varsigma,t+\log3)$ gives $\nu_{\mathbf z_j,h_k}\le C\Delta_{h_k,u}$ for each $(j,k)$ with conditional failure probability at most $e^{-u}$. By a union bound over $[m]\times[J]$, this holds simultaneously with conditional probability at least $1-\frac13e^{-t}$, and Steps~2--3 of that proof then apply at every $(j,k)$.
\end{proof}

\section{Additional simulation details and results}
\label{app:additional_simulation}

\newcommand{\fittable}[1]{\sbox0{#1}\ifdim\wd0>\linewidth
  \resizebox{\linewidth}{!}{\usebox0}\else\usebox0\fi}

This appendix collects simulation details and complementary results for Section~\ref{sec:sim}.

\subsection{Simulation setup}
\label{app:sim-settings}
\label{app:sim-central}

\paragraph{Implementation details.}
We use the shared Gaussian field from Section~\ref{subsec:bandwidth-sharing} with Gaussian noise calibrated by \eqref{eq:gaussian-calibration}, and set $\eta=1$, $\alpha_0=0.1$, $B=10$, and $R=100$.

To reduce Monte Carlo variance at $B=10$, we replace the sample average in \eqref{eq:soft-noise-scale} with the control-variate estimator
\[
\frac1B\sum_{b=1}^B\|\mathbf L_{\mathbf z}^{(b)}\|^2
-\frac4D\biggl(\frac1B\sum_{b=1}^B\|\mathbf L_{\mathbf z}^{(b)}\|_{\mathrm F}^2
-\mathbb E_{\mathbf g}\|\mathbf L_{\mathbf z}\|_{\mathrm F}^2\biggr),
\]
where $\mathbb E_{\mathbf g}\|\mathbf L_{\mathbf z}\|_{\mathrm F}^2$ is computed analytically from $\mathbf C_k$, and $4/D$ approximates $\mathbb E\|\mathbf W\|^2/\mathbb E\|\mathbf W\|_{\mathrm F}^2$ for a $D\times D$ Gaussian orthogonal ensemble matrix $\mathbf W$. If this estimate is nonpositive, we revert to the unadjusted sample average.

Hard corrections use the true intrinsic dimension $d$, with bandwidths selected via Appendix~\ref{app:hard-bandwidth-selection}.

Each configuration is evaluated over 100 independent repetitions, each with newly drawn reference data, public queries, and private noise.

\paragraph{Geometries.}
Figure~\ref{fig:app-geometries} shows a reference sample of each geometry. Each manifold is given in a standard parametrization and embedded isometrically into the specified ambient space. The four main geometries are defined as follows:
\begin{itemize}
    \item \textbf{Trefoil knot} $(d,D)=(1,5)$:
    \[
    t\mapsto \frac{1}{3}
    (\sin t+2\sin 2t,\,
     \cos t-2\cos 2t,\,
     -\sin 3t),
    \quad t\in[0,2\pi).
    \]

    \item \textbf{Torus} $(d,D)=(2,5)$:
    \[
    (\theta,\phi)\mapsto
    \bigl(
    (R_{\mathrm{maj}}+r_{\mathrm{min}}\cos\phi)\cos\theta,\,
    (R_{\mathrm{maj}}+r_{\mathrm{min}}\cos\phi)\sin\theta,\,
    r_{\mathrm{min}}\sin\phi
    \bigr),
    \]
    \[
    R_{\mathrm{maj}}=\frac23,\quad r_{\mathrm{min}}=\frac13,
    \qquad (\theta,\phi)\in[0,2\pi)^2.
    \]

    \item \textbf{Swiss roll} $(d,D)=(2,5)$:
    \[
    (\theta,v)\mapsto
    \bigl(
    r(\theta)\cos\theta,\,
    r(\theta)\sin\theta,\,
    v
    \bigr)-\mathbf c,
    \]
    \[
    r(\theta)=0.21+\frac{0.48}{2\pi}\theta,
    \quad
    (\theta,v)\in[0,3\pi]\times[-H,H],
    \]
    where $\mathbf c$ centers the surface and $H\approx 0.581$ is chosen so that it has diameter two.

    \item \textbf{Bowl} $(d,D)=(3,10)$:
    \[
    \mathbf u\mapsto
    (\mathbf u,0.4\|\mathbf u\|^2),
    \qquad \mathbf u\in B_3(0,1).
    \]
\end{itemize}

Additional geometries are:
\begin{itemize}
    \item \textbf{Double-curvature graph}
    $(d,D)=(3,10)$:
    \[
    \mathbf u\mapsto
    \bigl(
    \mathbf u,\,
    \alpha(u_1^2-u_2^2),\,
    \alpha u_2u_3
    \bigr),
    \quad
    \mathbf u\in B_3(0,1),\quad \alpha=0.3.
    \]

    \item \textbf{Fourier torus} $(d,D)=(2,10)$:
    \[
    \boldsymbol\theta\mapsto
    \bigl(
    r_k\cos(\mathbf w_k^\top\boldsymbol\theta),
    r_k\sin(\mathbf w_k^\top\boldsymbol\theta)
    \bigr)_{k=1}^5,
    \quad
    \boldsymbol\theta\in[0,2\pi)^2,
    \]
    \[
    (r_k)_{k=1}^5
    =(0.32,0.32,0.27,0.24,0.15),
    \]
    \[
    (\mathbf w_k)_{k=1}^5
    =((1,0),(0,1),(1,1),(1,-1),(2,0)).
    \]
\end{itemize}

The robustness experiments further consider:
\begin{itemize}
    \item \textbf{Circle} $(d,D)=(1,5)$:
    \[
    t\mapsto(\cos t,\sin t),
    \quad t\in[0,2\pi).
    \]

    \item \textbf{Plane} $(d,D)=(2,5)$:
    \[
    \mathbf u\mapsto\mathbf u,
    \quad \mathbf u\in[-1,1]^2.
    \]

    \item \textbf{Catenoid} $(d,D)=(2,5)$:
    \[
    (u,v)\mapsto\bigl(c\cosh(v/c)\cos u,\,c\cosh(v/c)\sin u,\,v\bigr),
    \quad (u,v)\in[0,2\pi)\times[-0.5,0.5],
    \]
    with $c=0.3$.
\end{itemize}
For the two linked tori in Case~5, see Appendix~\ref{app:sim-disconnected}.

\begin{figure}[htbp]
\centering
\includegraphics[width=\textwidth]{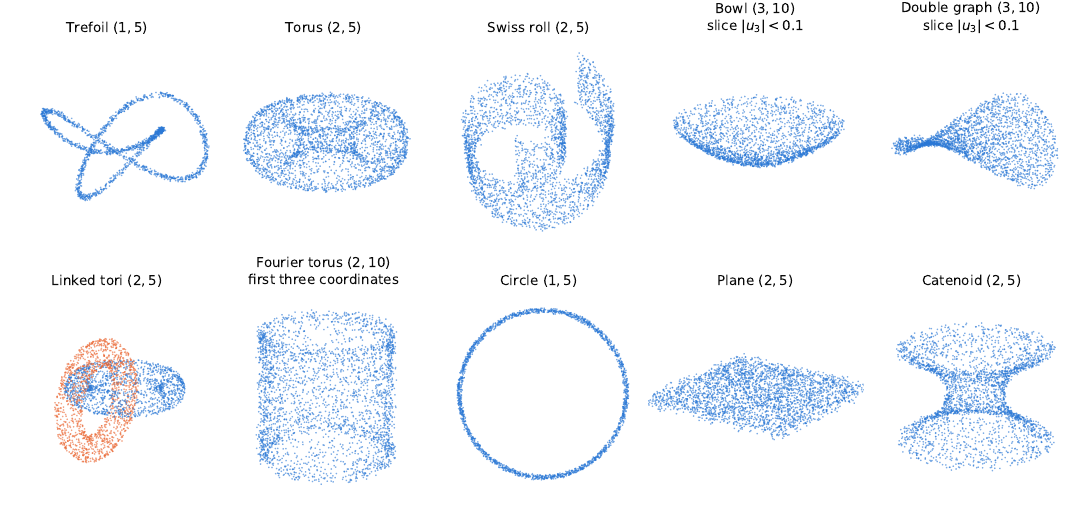}
\caption{One reference sample for each geometry at its default $n$ and $\sigma=0.04$. The bowl and the double-curvature graph are shown on the slice $|u_3|<0.1$, and the Fourier torus is projected onto its first three coordinates.}
\label{fig:app-geometries}
\end{figure}

\subsection{Detailed simulation results}
\label{app:sim-tables}
\label{app:sim-more}

Tables~\ref{tab:app-sweep-eps}--\ref{tab:app-sweep-sigma} report the full results for Cases~1--2 and across measurement-noise levels $\sigma$. Table~\ref{tab:app-hard-soft-a-dg} extends Case~3 to the double-curvature graph, and Table~\ref{tab:app-manifold} replaces latent-point RMSE with distance to the true manifold.

Figure~\ref{fig:app-sigma} and Table~\ref{tab:app-sweep-sigma} vary the noise radius $\sigma\in\{0.02,0.04,0.06,0.08\}$ at the default $n$ and $\varepsilon=1$. 

\begin{table}[htbp]\centering\small\setlength{\tabcolsep}{4pt}
\fittable{\begin{tabular}{lcccccc}
\toprule
Geometry & $\varepsilon$ & RAW & NP-MD-H & NP-MD-S & DP-MD-H & DP-MD-S \\
\midrule
Trefoil $(1,5)$ & $1$ & $0.080$ & $0.007$ \scriptsize(0.001) & $0.009$ \scriptsize(0.001) & $0.036$ \scriptsize(0.006) & $0.039$ \scriptsize(0.005) \\
 & $2$ & $0.080$ & $0.007$ \scriptsize(0.001) & $0.009$ \scriptsize(0.001) & $0.024$ \scriptsize(0.002) & $0.027$ \scriptsize(0.002) \\
 & $4$ & $0.080$ & $0.007$ \scriptsize(0.001) & $0.009$ \scriptsize(0.001) & $0.016$ \scriptsize(0.002) & $0.020$ \scriptsize(0.002) \\
 & $8$ & $0.080$ & $0.007$ \scriptsize(0.001) & $0.009$ \scriptsize(0.001) & $0.012$ \scriptsize(0.001) & $0.014$ \scriptsize(0.002) \\
\midrule
Torus $(2,5)$ & $1$ & $0.080$ & $0.010$ \scriptsize(0.001) & $0.012$ \scriptsize(0.001) & $0.042$ \scriptsize(0.008) & $0.043$ \scriptsize(0.008) \\
 & $2$ & $0.080$ & $0.010$ \scriptsize(0.001) & $0.012$ \scriptsize(0.001) & $0.033$ \scriptsize(0.005) & $0.033$ \scriptsize(0.005) \\
 & $4$ & $0.080$ & $0.010$ \scriptsize(0.001) & $0.012$ \scriptsize(0.001) & $0.022$ \scriptsize(0.003) & $0.024$ \scriptsize(0.004) \\
 & $8$ & $0.080$ & $0.010$ \scriptsize(0.001) & $0.012$ \scriptsize(0.001) & $0.019$ \scriptsize(0.001) & $0.020$ \scriptsize(0.001) \\
\midrule
Swiss roll $(2,5)$ & $1$ & $0.080$ & $0.007$ \scriptsize(0.001) & $0.009$ \scriptsize(0.001) & $0.028$ \scriptsize(0.006) & $0.030$ \scriptsize(0.006) \\
 & $2$ & $0.080$ & $0.007$ \scriptsize(0.001) & $0.009$ \scriptsize(0.001) & $0.021$ \scriptsize(0.002) & $0.023$ \scriptsize(0.003) \\
 & $4$ & $0.080$ & $0.007$ \scriptsize(0.001) & $0.009$ \scriptsize(0.001) & $0.015$ \scriptsize(0.002) & $0.017$ \scriptsize(0.002) \\
 & $8$ & $0.080$ & $0.007$ \scriptsize(0.001) & $0.009$ \scriptsize(0.001) & $0.012$ \scriptsize(0.001) & $0.013$ \scriptsize(0.002) \\
\midrule
Bowl $(3,10)$ & $1$ & $0.080$ & $0.008$ \scriptsize(0.002) & $0.010$ \scriptsize(0.002) & $0.041$ \scriptsize(0.009) & $0.041$ \scriptsize(0.008) \\
 & $2$ & $0.080$ & $0.008$ \scriptsize(0.002) & $0.010$ \scriptsize(0.002) & $0.030$ \scriptsize(0.007) & $0.031$ \scriptsize(0.007) \\
 & $4$ & $0.080$ & $0.008$ \scriptsize(0.002) & $0.010$ \scriptsize(0.002) & $0.023$ \scriptsize(0.004) & $0.023$ \scriptsize(0.005) \\
 & $8$ & $0.080$ & $0.008$ \scriptsize(0.002) & $0.010$ \scriptsize(0.002) & $0.019$ \scriptsize(0.002) & $0.020$ \scriptsize(0.001) \\
\midrule
Double-curvature graph $(3,10)$ & $1$ & $0.080$ & $0.008$ \scriptsize(0.002) & $0.010$ \scriptsize(0.002) & $0.017$ \scriptsize(0.006) & $0.021$ \scriptsize(0.004) \\
 & $2$ & $0.080$ & $0.008$ \scriptsize(0.002) & $0.010$ \scriptsize(0.002) & $0.011$ \scriptsize(0.004) & $0.014$ \scriptsize(0.003) \\
 & $4$ & $0.080$ & $0.008$ \scriptsize(0.002) & $0.010$ \scriptsize(0.002) & $0.008$ \scriptsize(0.004) & $0.011$ \scriptsize(0.005) \\
 & $8$ & $0.080$ & $0.008$ \scriptsize(0.002) & $0.010$ \scriptsize(0.002) & $0.006$ \scriptsize(0.003) & $0.009$ \scriptsize(0.005) \\
\bottomrule
\end{tabular}
}
\caption{Case~1. Latent RMSE (standard deviation in parentheses) across varying $\varepsilon$ under default settings.}
\label{tab:app-sweep-eps}
\end{table}

\begin{table}[htbp]\centering\small\setlength{\tabcolsep}{4pt}
\fittable{\begin{tabular}{lcccccc}
\toprule
Geometry & $n$ & RAW & NP-MD-H & NP-MD-S & DP-MD-H & DP-MD-S \\
\midrule
Trefoil $(1,5)$ & $1000$ & $0.080$ & $0.009$ \scriptsize(0.001) & $0.010$ \scriptsize(0.001) & $0.064$ \scriptsize(0.010) & $0.062$ \scriptsize(0.011) \\
 & $2000$ & $0.080$ & $0.007$ \scriptsize(0.001) & $0.009$ \scriptsize(0.001) & $0.036$ \scriptsize(0.006) & $0.039$ \scriptsize(0.005) \\
 & $4000$ & $0.080$ & $0.006$ \scriptsize(0.001) & $0.009$ \scriptsize(0.001) & $0.023$ \scriptsize(0.002) & $0.026$ \scriptsize(0.003) \\
 & $8000$ & $0.080$ & $0.006$ \scriptsize(0.001) & $0.009$ \scriptsize(0.001) & $0.015$ \scriptsize(0.002) & $0.019$ \scriptsize(0.002) \\
\midrule
Torus $(2,5)$ & $4000$ & $0.080$ & $0.012$ \scriptsize(0.001) & $0.013$ \scriptsize(0.001) & $0.065$ \scriptsize(0.021) & $0.063$ \scriptsize(0.016) \\
 & $8000$ & $0.080$ & $0.010$ \scriptsize(0.001) & $0.012$ \scriptsize(0.001) & $0.042$ \scriptsize(0.008) & $0.043$ \scriptsize(0.008) \\
 & $16000$ & $0.080$ & $0.010$ \scriptsize(0.001) & $0.011$ \scriptsize(0.001) & $0.031$ \scriptsize(0.005) & $0.033$ \scriptsize(0.005) \\
 & $32000$ & $0.080$ & $0.010$ \scriptsize(0.001) & $0.011$ \scriptsize(0.001) & $0.021$ \scriptsize(0.002) & $0.023$ \scriptsize(0.004) \\
\midrule
Swiss roll $(2,5)$ & $4000$ & $0.080$ & $0.009$ \scriptsize(0.002) & $0.010$ \scriptsize(0.001) & $0.044$ \scriptsize(0.011) & $0.046$ \scriptsize(0.009) \\
 & $8000$ & $0.080$ & $0.007$ \scriptsize(0.001) & $0.009$ \scriptsize(0.001) & $0.028$ \scriptsize(0.006) & $0.030$ \scriptsize(0.006) \\
 & $16000$ & $0.080$ & $0.007$ \scriptsize(0.001) & $0.009$ \scriptsize(0.001) & $0.020$ \scriptsize(0.003) & $0.023$ \scriptsize(0.003) \\
 & $32000$ & $0.080$ & $0.006$ \scriptsize(0.001) & $0.009$ \scriptsize(0.001) & $0.015$ \scriptsize(0.003) & $0.017$ \scriptsize(0.003) \\
\midrule
Bowl $(3,10)$ & $8000$ & $0.080$ & $0.009$ \scriptsize(0.001) & $0.009$ \scriptsize(0.000) & $0.053$ \scriptsize(0.014) & $0.055$ \scriptsize(0.012) \\
 & $16000$ & $0.080$ & $0.008$ \scriptsize(0.002) & $0.010$ \scriptsize(0.002) & $0.041$ \scriptsize(0.009) & $0.041$ \scriptsize(0.008) \\
 & $32000$ & $0.080$ & $0.008$ \scriptsize(0.001) & $0.009$ \scriptsize(0.001) & $0.031$ \scriptsize(0.009) & $0.032$ \scriptsize(0.009) \\
 & $64000$ & $0.080$ & $0.007$ \scriptsize(0.000) & $0.007$ \scriptsize(0.001) & $0.023$ \scriptsize(0.005) & $0.024$ \scriptsize(0.005) \\
\midrule
Double-curvature graph $(3,10)$ & $8000$ & $0.080$ & $0.007$ \scriptsize(0.001) & $0.007$ \scriptsize(0.000) & $0.024$ \scriptsize(0.008) & $0.033$ \scriptsize(0.007) \\
 & $16000$ & $0.080$ & $0.008$ \scriptsize(0.002) & $0.010$ \scriptsize(0.002) & $0.017$ \scriptsize(0.006) & $0.021$ \scriptsize(0.004) \\
 & $32000$ & $0.080$ & $0.007$ \scriptsize(0.002) & $0.007$ \scriptsize(0.002) & $0.011$ \scriptsize(0.003) & $0.015$ \scriptsize(0.004) \\
 & $64000$ & $0.080$ & $0.005$ \scriptsize(0.002) & $0.003$ \scriptsize(0.002) & $0.009$ \scriptsize(0.005) & $0.011$ \scriptsize(0.005) \\
\bottomrule
\end{tabular}
}
\caption{Case~2. Latent RMSE (standard deviation in parentheses) across varying $n$ under default settings with $\varepsilon=1$.}
\label{tab:app-sweep-n}
\end{table}

\begin{table}[htbp]\centering\small\setlength{\tabcolsep}{4pt}
\fittable{\begin{tabular}{lcccccc}
\toprule
Geometry & $\sigma$ & RAW & NP-MD-H & NP-MD-S & DP-MD-H & DP-MD-S \\
\midrule
Trefoil $(1,5)$ & $0.02$ & $0.080$ & $0.004$ \scriptsize(0.000) & $0.005$ \scriptsize(0.000) & $0.034$ \scriptsize(0.003) & $0.037$ \scriptsize(0.005) \\
 & $0.04$ & $0.080$ & $0.007$ \scriptsize(0.001) & $0.009$ \scriptsize(0.001) & $0.036$ \scriptsize(0.006) & $0.039$ \scriptsize(0.005) \\
 & $0.06$ & $0.080$ & $0.011$ \scriptsize(0.001) & $0.012$ \scriptsize(0.001) & $0.040$ \scriptsize(0.008) & $0.041$ \scriptsize(0.006) \\
 & $0.08$ & $0.080$ & $0.013$ \scriptsize(0.002) & $0.016$ \scriptsize(0.001) & $0.042$ \scriptsize(0.009) & $0.045$ \scriptsize(0.010) \\
\midrule
Torus $(2,5)$ & $0.02$ & $0.080$ & $0.006$ \scriptsize(0.001) & $0.007$ \scriptsize(0.001) & $0.040$ \scriptsize(0.005) & $0.043$ \scriptsize(0.008) \\
 & $0.04$ & $0.080$ & $0.010$ \scriptsize(0.001) & $0.012$ \scriptsize(0.001) & $0.042$ \scriptsize(0.008) & $0.043$ \scriptsize(0.008) \\
 & $0.06$ & $0.080$ & $0.014$ \scriptsize(0.001) & $0.018$ \scriptsize(0.002) & $0.044$ \scriptsize(0.011) & $0.047$ \scriptsize(0.010) \\
 & $0.08$ & $0.080$ & $0.019$ \scriptsize(0.002) & $0.021$ \scriptsize(0.002) & $0.049$ \scriptsize(0.014) & $0.049$ \scriptsize(0.010) \\
\midrule
Swiss roll $(2,5)$ & $0.02$ & $0.080$ & $0.004$ \scriptsize(0.000) & $0.005$ \scriptsize(0.001) & $0.026$ \scriptsize(0.004) & $0.029$ \scriptsize(0.005) \\
 & $0.04$ & $0.080$ & $0.007$ \scriptsize(0.001) & $0.009$ \scriptsize(0.001) & $0.028$ \scriptsize(0.006) & $0.030$ \scriptsize(0.006) \\
 & $0.06$ & $0.080$ & $0.011$ \scriptsize(0.001) & $0.012$ \scriptsize(0.001) & $0.031$ \scriptsize(0.011) & $0.033$ \scriptsize(0.008) \\
 & $0.08$ & $0.080$ & $0.013$ \scriptsize(0.001) & $0.016$ \scriptsize(0.001) & $0.034$ \scriptsize(0.012) & $0.037$ \scriptsize(0.008) \\
\midrule
Bowl $(3,10)$ & $0.02$ & $0.080$ & $0.006$ \scriptsize(0.000) & $0.006$ \scriptsize(0.000) & $0.040$ \scriptsize(0.007) & $0.040$ \scriptsize(0.007) \\
 & $0.04$ & $0.080$ & $0.008$ \scriptsize(0.002) & $0.010$ \scriptsize(0.002) & $0.041$ \scriptsize(0.009) & $0.041$ \scriptsize(0.008) \\
 & $0.06$ & $0.080$ & $0.010$ \scriptsize(0.001) & $0.012$ \scriptsize(0.001) & $0.042$ \scriptsize(0.009) & $0.043$ \scriptsize(0.009) \\
 & $0.08$ & $0.080$ & $0.013$ \scriptsize(0.001) & $0.013$ \scriptsize(0.001) & $0.044$ \scriptsize(0.011) & $0.045$ \scriptsize(0.012) \\
\midrule
Double-curvature graph $(3,10)$ & $0.02$ & $0.080$ & $0.004$ \scriptsize(0.001) & $0.004$ \scriptsize(0.001) & $0.016$ \scriptsize(0.005) & $0.021$ \scriptsize(0.004) \\
 & $0.04$ & $0.080$ & $0.008$ \scriptsize(0.002) & $0.010$ \scriptsize(0.002) & $0.017$ \scriptsize(0.006) & $0.021$ \scriptsize(0.004) \\
 & $0.06$ & $0.080$ & $0.009$ \scriptsize(0.003) & $0.010$ \scriptsize(0.002) & $0.018$ \scriptsize(0.006) & $0.022$ \scriptsize(0.005) \\
 & $0.08$ & $0.080$ & $0.009$ \scriptsize(0.003) & $0.009$ \scriptsize(0.003) & $0.018$ \scriptsize(0.007) & $0.022$ \scriptsize(0.006) \\
\bottomrule
\end{tabular}
}
\caption{Varying measurement noise $\sigma$. Latent RMSE (standard deviation in parentheses) under default settings with $\varepsilon=1$.}
\label{tab:app-sweep-sigma}
\end{table}

\begin{table}[htbp]\centering\small\setlength{\tabcolsep}{4pt}
\fittable{\begin{tabular}{lcccccc}
\toprule
Geometry & $a$ & RAW & NP-MD-H & NP-MD-S & DP-MD-H & DP-MD-S \\
\midrule
Double-curvature graph $(3,10)$ & $0.08$ & $0.080$ & $0.008$ \scriptsize(0.002) & $0.010$ \scriptsize(0.002) & $0.017$ \scriptsize(0.006) & $0.021$ \scriptsize(0.004) \\
 & $0.10$ & $0.100$ & $0.005$ \scriptsize(0.001) & $0.007$ \scriptsize(0.002) & $0.017$ \scriptsize(0.007) & $0.021$ \scriptsize(0.005) \\
 & $0.12$ & $0.120$ & $0.006$ \scriptsize(0.001) & $0.007$ \scriptsize(0.002) & $0.018$ \scriptsize(0.008) & $0.022$ \scriptsize(0.005) \\
 & $0.14$ & $0.140$ & $0.006$ \scriptsize(0.001) & $0.007$ \scriptsize(0.003) & $0.017$ \scriptsize(0.005) & $0.022$ \scriptsize(0.004) \\
\bottomrule
\end{tabular}
}
\caption{Case~3 on the double-curvature graph. Latent RMSE (standard deviation in parentheses) across varying $a$ under default settings with $\varepsilon=1$.}
\label{tab:app-hard-soft-a-dg}
\end{table}

\begin{table}[htbp]\centering\small\setlength{\tabcolsep}{4pt}
\fittable{\begin{tabular}{lcccccc}
\toprule
Geometry & $a$ & RAW & NP-MD-H & NP-MD-S & DP-MD-H & DP-MD-S \\
\midrule
Trefoil $(1,5)$ & $0.08$ & $0.080$ & $0.007$ \scriptsize(0.001) & $0.009$ \scriptsize(0.001) & $0.035$ \scriptsize(0.006) & $0.037$ \scriptsize(0.006) \\
 & $0.10$ & $0.100$ & $0.007$ \scriptsize(0.001) & $0.010$ \scriptsize(0.001) & $0.037$ \scriptsize(0.005) & $0.042$ \scriptsize(0.006) \\
 & $0.12$ & $0.120$ & $0.008$ \scriptsize(0.001) & $0.011$ \scriptsize(0.001) & $0.044$ \scriptsize(0.008) & $0.050$ \scriptsize(0.007) \\
 & $0.14$ & $0.140$ & $0.009$ \scriptsize(0.001) & $0.012$ \scriptsize(0.001) & $0.055$ \scriptsize(0.009) & $0.059$ \scriptsize(0.007) \\
\midrule
Torus $(2,5)$ & $0.08$ & $0.080$ & $0.010$ \scriptsize(0.001) & $0.011$ \scriptsize(0.001) & $0.039$ \scriptsize(0.008) & $0.041$ \scriptsize(0.009) \\
 & $0.10$ & $0.100$ & $0.011$ \scriptsize(0.001) & $0.013$ \scriptsize(0.001) & $0.040$ \scriptsize(0.008) & $0.046$ \scriptsize(0.010) \\
 & $0.12$ & $0.120$ & $0.011$ \scriptsize(0.001) & $0.014$ \scriptsize(0.001) & $0.048$ \scriptsize(0.014) & $0.052$ \scriptsize(0.013) \\
 & $0.14$ & $0.140$ & $0.012$ \scriptsize(0.001) & $0.016$ \scriptsize(0.001) & $0.059$ \scriptsize(0.017) & $0.063$ \scriptsize(0.014) \\
\midrule
Swiss roll $(2,5)$ & $0.08$ & $0.080$ & $0.007$ \scriptsize(0.001) & $0.009$ \scriptsize(0.001) & $0.025$ \scriptsize(0.006) & $0.026$ \scriptsize(0.005) \\
 & $0.10$ & $0.100$ & $0.007$ \scriptsize(0.001) & $0.010$ \scriptsize(0.001) & $0.028$ \scriptsize(0.007) & $0.031$ \scriptsize(0.007) \\
 & $0.12$ & $0.120$ & $0.008$ \scriptsize(0.001) & $0.011$ \scriptsize(0.001) & $0.033$ \scriptsize(0.009) & $0.037$ \scriptsize(0.009) \\
 & $0.14$ & $0.140$ & $0.009$ \scriptsize(0.001) & $0.012$ \scriptsize(0.001) & $0.043$ \scriptsize(0.012) & $0.044$ \scriptsize(0.010) \\
\midrule
Bowl $(3,10)$ & $0.08$ & $0.080$ & $0.008$ \scriptsize(0.001) & $0.009$ \scriptsize(0.002) & $0.040$ \scriptsize(0.010) & $0.040$ \scriptsize(0.008) \\
 & $0.10$ & $0.100$ & $0.007$ \scriptsize(0.000) & $0.008$ \scriptsize(0.000) & $0.040$ \scriptsize(0.007) & $0.041$ \scriptsize(0.009) \\
 & $0.12$ & $0.120$ & $0.007$ \scriptsize(0.001) & $0.009$ \scriptsize(0.001) & $0.042$ \scriptsize(0.009) & $0.043$ \scriptsize(0.009) \\
 & $0.14$ & $0.140$ & $0.008$ \scriptsize(0.001) & $0.009$ \scriptsize(0.001) & $0.044$ \scriptsize(0.009) & $0.046$ \scriptsize(0.010) \\
\midrule
Double-curvature graph $(3,10)$ & $0.08$ & $0.080$ & $0.007$ \scriptsize(0.002) & $0.008$ \scriptsize(0.002) & $0.016$ \scriptsize(0.005) & $0.015$ \scriptsize(0.005) \\
 & $0.10$ & $0.100$ & $0.005$ \scriptsize(0.001) & $0.006$ \scriptsize(0.001) & $0.016$ \scriptsize(0.006) & $0.015$ \scriptsize(0.005) \\
 & $0.12$ & $0.120$ & $0.005$ \scriptsize(0.001) & $0.006$ \scriptsize(0.002) & $0.016$ \scriptsize(0.006) & $0.015$ \scriptsize(0.005) \\
 & $0.14$ & $0.140$ & $0.005$ \scriptsize(0.001) & $0.006$ \scriptsize(0.002) & $0.015$ \scriptsize(0.004) & $0.015$ \scriptsize(0.004) \\
\bottomrule
\end{tabular}
}
\caption{Case~3. Manifold RMSE (standard deviation in parentheses) across varying $a$ under default settings with $\varepsilon=1$. Manifold RMSE measures the distance to the nearest point on the true manifold.}
\label{tab:app-manifold}
\end{table}

\begin{figure}[htbp]
\centering
\includegraphics[width=\textwidth]{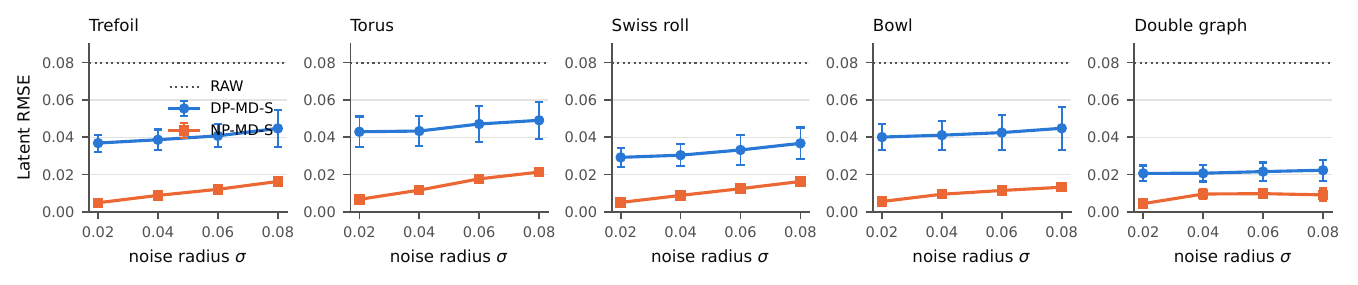}
\caption{Latent RMSE across four geometries as $\sigma$ varies under default settings with $\varepsilon=1$.}
\label{fig:app-sigma}
\end{figure}

\subsection{Further geometries}
\label{app:sim-robustness}
\label{app:sim-challenging}

Table~\ref{tab:app-extra} and Figure~\ref{fig:app-extra-geometry} report results on three further geometries under default settings with $\varepsilon=1$, using $n=2000$ for the circle and $n=8000$ for the plane and catenoid.

\begin{figure}[htbp]
\centering
\includegraphics[width=\textwidth]{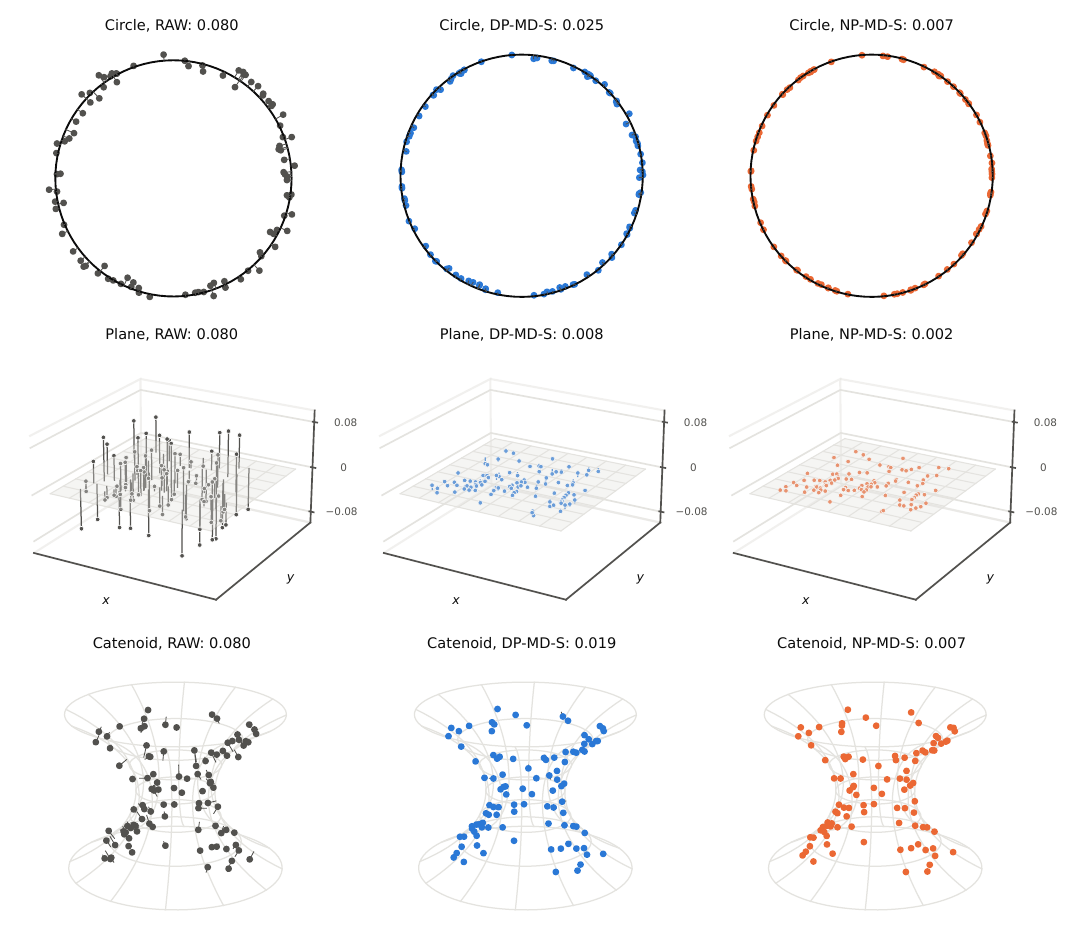}
\caption{Reconstructions of 100 public queries on each additional geometry under default settings with $\varepsilon=1$. Columns show RAW, DP-MD-S, and NP-MD-S. The true manifold is shown as a curve or surface, and thin segments link each output to its latent target $\mathbf z_j^\star$.}
\label{fig:app-extra-geometry}
\end{figure}

\begin{table}[htbp]\centering\small\setlength{\tabcolsep}{4pt}
\fittable{\begin{tabular}{lcccccc}
\toprule
Geometry & $n$ & RAW & NP-MD-H & NP-MD-S & DP-MD-H & DP-MD-S \\
\midrule
Circle $(1,5)$ & $2000$ & $0.080$ & $0.006$ \scriptsize(0.000) & $0.007$ \scriptsize(0.000) & $0.024$ \scriptsize(0.003) & $0.026$ \scriptsize(0.003) \\
\midrule
Plane $(2,5)$ & $8000$ & $0.080$ & $0.003$ \scriptsize(0.001) & $0.002$ \scriptsize(0.001) & $0.007$ \scriptsize(0.003) & $0.009$ \scriptsize(0.002) \\
\midrule
Catenoid $(2,5)$ & $8000$ & $0.080$ & $0.005$ \scriptsize(0.001) & $0.007$ \scriptsize(0.001) & $0.017$ \scriptsize(0.006) & $0.018$ \scriptsize(0.004) \\
\bottomrule
\end{tabular}
}
\caption{Further geometries. Latent RMSE (standard deviation in parentheses) under default settings with $\varepsilon=1$.}
\label{tab:app-extra}
\end{table}

\subsection{Non-uniform bounded noise}
\label{app:sim-noise}

Table~\ref{tab:app-noise} reports results on five geometries under non-uniform bounded noise, where each perturbation has a uniform random direction and a radius on $[0,\sigma]$ with density proportional to $1-r/\sigma$. All other simulation settings follow Section~\ref{sec:sim}.

\begin{table}[htbp]\centering\small\setlength{\tabcolsep}{4pt}
\fittable{\begin{tabular}{lcccccc}
\toprule
Geometry & noise & RAW & NP-MD-H & NP-MD-S & DP-MD-H & DP-MD-S \\
\midrule
Trefoil $(1,5)$ & uniform & $0.080$ & $0.007$ \scriptsize(0.001) & $0.009$ \scriptsize(0.001) & $0.036$ \scriptsize(0.006) & $0.039$ \scriptsize(0.005) \\
 & central & $0.080$ & $0.004$ \scriptsize(0.001) & $0.005$ \scriptsize(0.000) & $0.034$ \scriptsize(0.003) & $0.037$ \scriptsize(0.005) \\
\midrule
Torus $(2,5)$ & uniform & $0.080$ & $0.010$ \scriptsize(0.001) & $0.012$ \scriptsize(0.001) & $0.042$ \scriptsize(0.008) & $0.043$ \scriptsize(0.008) \\
 & central & $0.080$ & $0.006$ \scriptsize(0.001) & $0.007$ \scriptsize(0.001) & $0.040$ \scriptsize(0.006) & $0.043$ \scriptsize(0.008) \\
\midrule
Swiss roll $(2,5)$ & uniform & $0.080$ & $0.007$ \scriptsize(0.001) & $0.009$ \scriptsize(0.001) & $0.028$ \scriptsize(0.006) & $0.030$ \scriptsize(0.006) \\
 & central & $0.080$ & $0.004$ \scriptsize(0.001) & $0.005$ \scriptsize(0.001) & $0.026$ \scriptsize(0.004) & $0.030$ \scriptsize(0.006) \\
\midrule
Bowl $(3,10)$ & uniform & $0.080$ & $0.008$ \scriptsize(0.002) & $0.010$ \scriptsize(0.002) & $0.041$ \scriptsize(0.009) & $0.041$ \scriptsize(0.008) \\
 & central & $0.080$ & $0.005$ \scriptsize(0.000) & $0.005$ \scriptsize(0.000) & $0.039$ \scriptsize(0.007) & $0.040$ \scriptsize(0.007) \\
\midrule
Double-curvature graph $(3,10)$ & uniform & $0.080$ & $0.008$ \scriptsize(0.002) & $0.010$ \scriptsize(0.002) & $0.017$ \scriptsize(0.006) & $0.021$ \scriptsize(0.004) \\
 & central & $0.080$ & $0.003$ \scriptsize(0.001) & $0.004$ \scriptsize(0.001) & $0.016$ \scriptsize(0.005) & $0.021$ \scriptsize(0.004) \\
\bottomrule
\end{tabular}
}
\caption{Uniform versus non-uniform bounded noise. Latent RMSE (standard deviation in parentheses) under default settings with $\varepsilon=1$.}
\label{tab:app-noise}
\end{table}

\subsection{Shared versus independent releases}
\label{app:sim-release}
\label{app:sim-hard-soft}

Case~4 compares the joint shared field with independent per-scale and per-query releases under the same $(\varepsilon,\delta)$ budget. The per-scale baseline uses nine bandwidth-specific fields with noise multiplier $\sqrt9$, whereas the per-query baseline uses separate query-specific fields with multiplier $\sqrt m$. Table~\ref{tab:app-release-m} reports the full results.

\begin{table}[htbp]\centering\small\setlength{\tabcolsep}{4pt}
\fittable{\begin{tabular}{lcccc}
\toprule
Geometry & $m$ & Joint & Per-scale & Per-query \\
\midrule
Trefoil $(1,5)$ & $10$ & $0.039$ \scriptsize(0.007) & $0.060$ \scriptsize(0.010) & $0.103$ \scriptsize(0.031) \\
 & $20$ & $0.039$ \scriptsize(0.005) & $0.061$ \scriptsize(0.008) & $0.131$ \scriptsize(0.031) \\
 & $50$ & $0.037$ \scriptsize(0.003) & $0.061$ \scriptsize(0.007) & $0.193$ \scriptsize(0.031) \\
 & $100$ & $0.037$ \scriptsize(0.002) & $0.060$ \scriptsize(0.006) & $0.260$ \scriptsize(0.018) \\
\midrule
Torus $(2,5)$ & $10$ & $0.047$ \scriptsize(0.011) & $0.063$ \scriptsize(0.015) & $0.090$ \scriptsize(0.029) \\
 & $20$ & $0.044$ \scriptsize(0.008) & $0.063$ \scriptsize(0.015) & $0.104$ \scriptsize(0.025) \\
 & $50$ & $0.041$ \scriptsize(0.006) & $0.058$ \scriptsize(0.012) & $0.131$ \scriptsize(0.027) \\
 & $100$ & $0.040$ \scriptsize(0.003) & $0.058$ \scriptsize(0.011) & $0.161$ \scriptsize(0.025) \\
\bottomrule
\end{tabular}
}
\caption{Case~4. Latent RMSE (standard deviation in parentheses) of DP-MD-S across varying $m$ under default settings with $\varepsilon=1$.}
\label{tab:app-release-m}
\end{table}

\subsection{Disconnected manifolds}
\label{app:sim-disconnected}
The disconnected geometry in Case~5 consists of two linked tori with major radius $R_{\mathrm{maj}}=0.7$, minor radius $r_{\mathrm{min}}=0.2$, and centers $\mathbf c_1=(-0.35,0,0)$ and $\mathbf c_2=(0.35,0,0)$. Their core circles lie in the $x_1x_2$- and $x_1x_3$-planes, respectively, giving a minimum surface separation of $0.3$. Reference points are drawn independently from an equal-weight mixture of the two surface-area distributions, and public queries are split equally between the two components. Table~\ref{tab:app-disconnected} reports the full results, and Figure~\ref{fig:app-disconnected} illustrates one realization.

\begin{figure}[htbp]
\centering
\includegraphics[width=\textwidth]{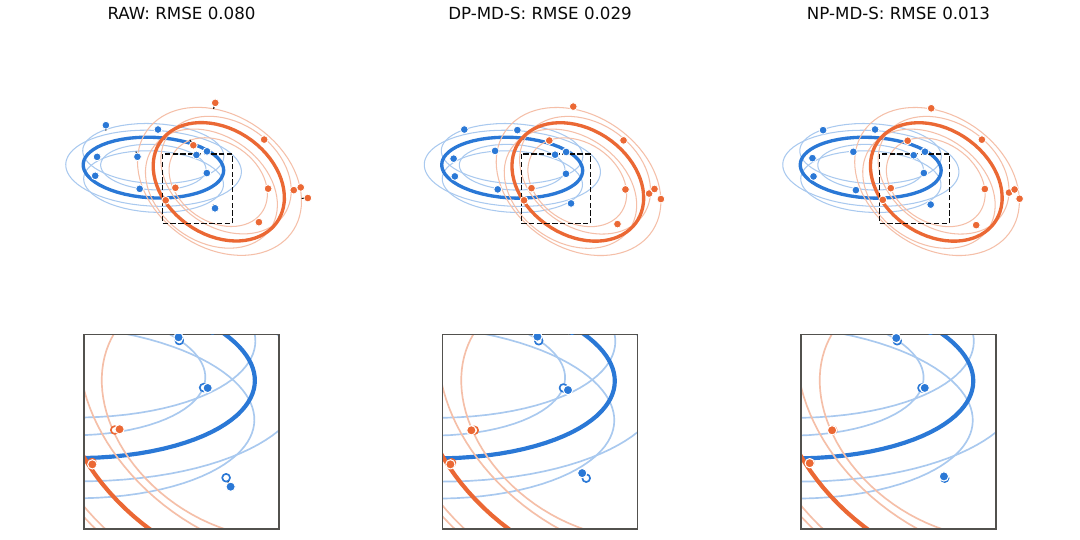}
\caption{Reconstructions of 20 public queries on one realization of the two linked tori ($n=16000$, $\sigma=0.04$, $\varepsilon=2$). Columns show RAW, DP-MD-S, and NP-MD-S. The bottom row magnifies the dashed region, with open circles marking the latent targets $\mathbf z_j^\star$.}
\label{fig:app-disconnected}
\end{figure}

\begin{table}[htbp]\centering\small\setlength{\tabcolsep}{4pt}
\fittable{\begin{tabular}{lccccc}
\toprule
Geometry & setting & RAW & NP-MD-S & DP-MD-S & DP-PCA \\
\midrule
Linked tori $(2,5)$ & $n=8000$ & $0.080$ & $0.017$ \scriptsize(0.002) & $0.056$ \scriptsize(0.011) & $0.046$ \scriptsize(0.004) \\
 & $n=16000$ & $0.080$ & $0.016$ \scriptsize(0.001) & $0.040$ \scriptsize(0.006) & $0.046$ \scriptsize(0.004) \\
 & $n=32000$ & $0.080$ & $0.016$ \scriptsize(0.001) & $0.033$ \scriptsize(0.004) & $0.046$ \scriptsize(0.004) \\
 & $n=64000$ & $0.080$ & $0.016$ \scriptsize(0.002) & $0.026$ \scriptsize(0.004) & $0.046$ \scriptsize(0.004) \\
 & $\varepsilon=1$ & $0.080$ & $0.016$ \scriptsize(0.001) & $0.051$ \scriptsize(0.010) & $0.046$ \scriptsize(0.004) \\
 & $\varepsilon=4$ & $0.080$ & $0.016$ \scriptsize(0.001) & $0.034$ \scriptsize(0.004) & $0.046$ \scriptsize(0.004) \\
 & $\varepsilon=8$ & $0.080$ & $0.016$ \scriptsize(0.001) & $0.029$ \scriptsize(0.005) & $0.046$ \scriptsize(0.004) \\
\bottomrule
\end{tabular}
}
\caption{Case~5 on the two linked tori. Latent RMSE (standard deviation in parentheses) as $\varepsilon$ varies at $n=16000$ and as $n$ varies at $\varepsilon=2$.}
\label{tab:app-disconnected}
\end{table}

\subsection{Alternative private estimators}
\label{app:sim-baselines}
\label{app:sim-external}
\label{app:sim-dp-baseline}

{Tables~\ref{tab:app-baselines} and~\ref{tab:app-baselines-single} compare DP-MD-S with global DP-PCA and subsample-and-aggregate adaptations of MF2023 \citep{yao2023manifold} and MF2025 \citep{yao2025manifold}, with NP-MD-S included as a non-private reference.\par}

\paragraph{Global DP-PCA.}
Here we consider a global DP-PCA baseline, which can be viewed as the global counterpart of our hard correction. We partition the reference indices publicly into disjoint sets $I_1$ and $I_2$ containing $10\%$ and $90\%$ of the records, respectively, with $n_2:=|I_2|$. We select a clipping radius $c$ from a public logarithmic grid via an $(\varepsilon,0)$-DP exponential mechanism for the median of $\{\|\mathbf y_i\|:i\in I_1\}$.

For $i\in I_2$, let $\mathbf y_i^{\mathrm{clip}}:=\mathbf y_i\min(1,c/\|\mathbf y_i\|)$, and set
\[
\boldsymbol\mu:=\frac1{n_2}\sum_{i\in I_2}\mathbf y_i^{\mathrm{clip}},
\qquad
\mathbf M_2:=\frac1{n_2}\sum_{i\in I_2}\mathbf y_i^{\mathrm{clip}}(\mathbf y_i^{\mathrm{clip}})^\top.
\]
We release
\[
\begin{pmatrix}
\tilde{\boldsymbol\mu}\\
\operatorname{svec}(\widetilde{\mathbf M}_2)/\sqrt2
\end{pmatrix}
=
\begin{pmatrix}
\boldsymbol\mu\\
\operatorname{svec}(\mathbf M_2)/\sqrt2
\end{pmatrix}
+\boldsymbol\xi,
\qquad
\boldsymbol\xi\sim\mathcal N(\mathbf 0,\varsigma_{\mathrm{PCA}}^2\mathbf I_{p-1}),
\]
where $\varsigma_{\mathrm{PCA}}$ is calibrated via \eqref{eq:gaussian-calibration} to the joint $\ell_2$-sensitivity bound $\sqrt{4c^2+c^4}/n_2$. Overall $(\varepsilon,\delta)$-DP follows by parallel composition.

Let $\widetilde{\boldsymbol\Sigma}_{\mathrm{PCA}}:=\widetilde{\mathbf M}_2-\tilde{\boldsymbol\mu}\tilde{\boldsymbol\mu}^{\top}$, with eigenvalues $\tilde{s}_1^{\mathrm{PCA}}\ge\cdots\ge\tilde{s}_D^{\mathrm{PCA}}$ and leading-$k$ eigenvectors $\widetilde{\mathbf U}_k$. Analogously to \eqref{eq:soft-noise-scale}, we estimate the covariance perturbation scale $\nu_{\mathrm{PCA}}$ from $B=10$ auxiliary draws of $\boldsymbol\xi$. The selected rank and query correction are
\[
\widehat r:=\operatorname*{argmax}_{1\le k<D}
\frac{\tilde{s}_k^{\mathrm{PCA}}\vee\nu_{\mathrm{PCA}}}
{\tilde{s}_{k+1}^{\mathrm{PCA}}\vee\nu_{\mathrm{PCA}}},
\qquad
\tilde{\mathbf z}^{\mathrm{PCA}}_j
:=\tilde{\boldsymbol\mu}
+\widetilde{\mathbf U}_{\widehat r}
\widetilde{\mathbf U}_{\widehat r}^{\top}
(\mathbf z_j-\tilde{\boldsymbol\mu}).
\]

\paragraph{Subsample-and-aggregate.}
We randomly partition the reference indices into $K=50$ disjoint blocks $I_1,\dots,I_K$. For each query $\mathbf z_j$, we apply MF2023 or MF2025 on each block $I_k$ using its default neighborhood radius, yielding an estimate $\widehat{\mathbf z}_{k,j}$, and clip the displacement $\mathbf v_{k,j}:=\widehat{\mathbf z}_{k,j}-\mathbf z_j$ at radius $c_{\mathrm{SA}}=0.2$:
\[
\mathbf v_{k,j}^{\mathrm{clip}}
:=\mathbf v_{k,j}\min\bigl(1,c_{\mathrm{SA}}/\|\mathbf v_{k,j}\|\bigr).
\]
The joint Gaussian release for all $m$ queries is
\[
\tilde{\mathbf z}^{\mathrm{SA}}_j
:=\mathbf z_j+\frac1K\sum_{k=1}^K\mathbf v_{k,j}^{\mathrm{clip}}+\boldsymbol\xi_j,
\qquad j\in[m],
\]
where $\boldsymbol\xi_1,\dots,\boldsymbol\xi_m\stackrel{\mathrm{i.i.d.}}{\sim}\mathcal N(\mathbf 0,\varsigma_{\mathrm{SA}}^2\mathbf I_D)$ with $\varsigma_{\mathrm{SA}}$ calibrated via \eqref{eq:gaussian-calibration} to the joint $\ell_2$-sensitivity bound $2c_{\mathrm{SA}}\sqrt m/K$, since changing one record affects at most one block.

\begin{table}[htbp]\centering\small\setlength{\tabcolsep}{3pt}
\fittable{\begin{tabular}{lcccccc}
\toprule
Geometry & $\varepsilon$ & DP-MD-S & NP-MD-S & DP-PCA & SA-MF2023 & SA-MF2025 \\
\midrule
Trefoil $(1,5)$ & $1$ & $0.039$ \scriptsize(0.005) & $0.009$ \scriptsize(0.001) & $0.125$ \scriptsize(0.086) & $0.262$ \scriptsize(0.021) & $0.239$ \scriptsize(0.019) \\
 & $2$ & $0.027$ \scriptsize(0.002) & $0.009$ \scriptsize(0.001) & $0.058$ \scriptsize(0.004) & $0.171$ \scriptsize(0.012) & $0.132$ \scriptsize(0.010) \\
 & $4$ & $0.020$ \scriptsize(0.002) & $0.009$ \scriptsize(0.001) & $0.057$ \scriptsize(0.003) & $0.133$ \scriptsize(0.008) & $0.076$ \scriptsize(0.005) \\
 & $8$ & $0.014$ \scriptsize(0.002) & $0.009$ \scriptsize(0.001) & $0.056$ \scriptsize(0.003) & $0.119$ \scriptsize(0.006) & $0.047$ \scriptsize(0.003) \\
\midrule
Torus $(2,5)$ & $1$ & $0.043$ \scriptsize(0.008) & $0.012$ \scriptsize(0.001) & $0.046$ \scriptsize(0.005) & $0.272$ \scriptsize(0.019) & $0.269$ \scriptsize(0.019) \\
 & $2$ & $0.033$ \scriptsize(0.005) & $0.012$ \scriptsize(0.001) & $0.046$ \scriptsize(0.005) & $0.149$ \scriptsize(0.012) & $0.144$ \scriptsize(0.011) \\
 & $4$ & $0.024$ \scriptsize(0.004) & $0.012$ \scriptsize(0.001) & $0.046$ \scriptsize(0.005) & $0.091$ \scriptsize(0.006) & $0.082$ \scriptsize(0.005) \\
 & $8$ & $0.020$ \scriptsize(0.001) & $0.012$ \scriptsize(0.001) & $0.046$ \scriptsize(0.005) & $0.063$ \scriptsize(0.004) & $0.048$ \scriptsize(0.003) \\
\midrule
Swiss roll $(2,5)$ & $1$ & $0.030$ \scriptsize(0.006) & $0.009$ \scriptsize(0.001) & $0.046$ \scriptsize(0.005) & $0.270$ \scriptsize(0.018) & $0.269$ \scriptsize(0.017) \\
 & $2$ & $0.023$ \scriptsize(0.003) & $0.009$ \scriptsize(0.001) & $0.046$ \scriptsize(0.005) & $0.148$ \scriptsize(0.012) & $0.144$ \scriptsize(0.010) \\
 & $4$ & $0.017$ \scriptsize(0.002) & $0.009$ \scriptsize(0.001) & $0.046$ \scriptsize(0.005) & $0.088$ \scriptsize(0.007) & $0.081$ \scriptsize(0.005) \\
 & $8$ & $0.013$ \scriptsize(0.002) & $0.009$ \scriptsize(0.001) & $0.046$ \scriptsize(0.005) & $0.057$ \scriptsize(0.005) & $0.046$ \scriptsize(0.003) \\
\midrule
Bowl $(3,10)$ & $1$ & $0.041$ \scriptsize(0.008) & $0.010$ \scriptsize(0.002) & $0.095$ \scriptsize(0.014) & $0.400$ \scriptsize(0.023) & $0.400$ \scriptsize(0.023) \\
 & $2$ & $0.031$ \scriptsize(0.007) & $0.010$ \scriptsize(0.002) & $0.095$ \scriptsize(0.014) & $0.216$ \scriptsize(0.010) & $0.216$ \scriptsize(0.010) \\
 & $4$ & $0.023$ \scriptsize(0.005) & $0.010$ \scriptsize(0.002) & $0.095$ \scriptsize(0.014) & $0.118$ \scriptsize(0.006) & $0.118$ \scriptsize(0.006) \\
 & $8$ & $0.020$ \scriptsize(0.001) & $0.010$ \scriptsize(0.002) & $0.095$ \scriptsize(0.014) & $0.069$ \scriptsize(0.004) & $0.069$ \scriptsize(0.003) \\
\midrule
Double-curvature graph $(3,10)$ & $1$ & $0.021$ \scriptsize(0.004) & $0.010$ \scriptsize(0.002) & $0.071$ \scriptsize(0.011) & $0.402$ \scriptsize(0.018) & $0.402$ \scriptsize(0.018) \\
 & $2$ & $0.014$ \scriptsize(0.003) & $0.010$ \scriptsize(0.002) & $0.071$ \scriptsize(0.011) & $0.214$ \scriptsize(0.011) & $0.214$ \scriptsize(0.011) \\
 & $4$ & $0.011$ \scriptsize(0.005) & $0.010$ \scriptsize(0.002) & $0.071$ \scriptsize(0.011) & $0.118$ \scriptsize(0.006) & $0.119$ \scriptsize(0.006) \\
 & $8$ & $0.009$ \scriptsize(0.005) & $0.010$ \scriptsize(0.002) & $0.071$ \scriptsize(0.011) & $0.067$ \scriptsize(0.003) & $0.068$ \scriptsize(0.003) \\
\bottomrule
\end{tabular}
}
\caption{Alternative private estimators. Latent RMSE (standard deviation in parentheses) under default settings with $m=20$.}
\label{tab:app-baselines}
\end{table}

\begin{table}[htbp]\centering\small\setlength{\tabcolsep}{3pt}
\fittable{\begin{tabular}{lcccccc}
\toprule
Geometry & $\varepsilon$ & DP-MD-S & NP-MD-S & DP-PCA & SA-MF2023 & SA-MF2025 \\
\midrule
Trefoil $(1,5)$ & $1$ & $0.044$ \scriptsize(0.023) & $0.008$ \scriptsize(0.003) & $0.109$ \scriptsize(0.094) & $0.116$ \scriptsize(0.027) & $0.052$ \scriptsize(0.019) \\
 & $2$ & $0.025$ \scriptsize(0.011) & $0.008$ \scriptsize(0.003) & $0.051$ \scriptsize(0.018) & $0.108$ \scriptsize(0.024) & $0.035$ \scriptsize(0.012) \\
 & $4$ & $0.019$ \scriptsize(0.008) & $0.008$ \scriptsize(0.003) & $0.049$ \scriptsize(0.018) & $0.106$ \scriptsize(0.021) & $0.025$ \scriptsize(0.008) \\
 & $8$ & $0.013$ \scriptsize(0.005) & $0.008$ \scriptsize(0.003) & $0.049$ \scriptsize(0.018) & $0.106$ \scriptsize(0.020) & $0.022$ \scriptsize(0.006) \\
\midrule
Torus $(2,5)$ & $1$ & $0.050$ \scriptsize(0.023) & $0.012$ \scriptsize(0.005) & $0.040$ \scriptsize(0.022) & $0.067$ \scriptsize(0.023) & $0.056$ \scriptsize(0.020) \\
 & $2$ & $0.034$ \scriptsize(0.020) & $0.012$ \scriptsize(0.005) & $0.040$ \scriptsize(0.022) & $0.052$ \scriptsize(0.017) & $0.036$ \scriptsize(0.011) \\
 & $4$ & $0.028$ \scriptsize(0.016) & $0.012$ \scriptsize(0.005) & $0.039$ \scriptsize(0.023) & $0.046$ \scriptsize(0.015) & $0.024$ \scriptsize(0.009) \\
 & $8$ & $0.019$ \scriptsize(0.009) & $0.012$ \scriptsize(0.005) & $0.039$ \scriptsize(0.023) & $0.043$ \scriptsize(0.013) & $0.021$ \scriptsize(0.007) \\
\midrule
Swiss roll $(2,5)$ & $1$ & $0.036$ \scriptsize(0.021) & $0.009$ \scriptsize(0.004) & $0.039$ \scriptsize(0.023) & $0.068$ \scriptsize(0.021) & $0.061$ \scriptsize(0.020) \\
 & $2$ & $0.023$ \scriptsize(0.012) & $0.009$ \scriptsize(0.004) & $0.039$ \scriptsize(0.023) & $0.043$ \scriptsize(0.019) & $0.033$ \scriptsize(0.012) \\
 & $4$ & $0.019$ \scriptsize(0.010) & $0.009$ \scriptsize(0.004) & $0.039$ \scriptsize(0.023) & $0.035$ \scriptsize(0.018) & $0.021$ \scriptsize(0.006) \\
 & $8$ & $0.014$ \scriptsize(0.007) & $0.009$ \scriptsize(0.004) & $0.039$ \scriptsize(0.023) & $0.032$ \scriptsize(0.018) & $0.015$ \scriptsize(0.005) \\
\midrule
Bowl $(3,10)$ & $1$ & $0.058$ \scriptsize(0.027) & $0.010$ \scriptsize(0.003) & $0.082$ \scriptsize(0.058) & $0.088$ \scriptsize(0.018) & $0.088$ \scriptsize(0.019) \\
 & $2$ & $0.046$ \scriptsize(0.020) & $0.010$ \scriptsize(0.003) & $0.081$ \scriptsize(0.059) & $0.051$ \scriptsize(0.011) & $0.051$ \scriptsize(0.011) \\
 & $4$ & $0.037$ \scriptsize(0.021) & $0.010$ \scriptsize(0.003) & $0.081$ \scriptsize(0.059) & $0.032$ \scriptsize(0.008) & $0.032$ \scriptsize(0.008) \\
 & $8$ & $0.026$ \scriptsize(0.014) & $0.010$ \scriptsize(0.003) & $0.081$ \scriptsize(0.059) & $0.024$ \scriptsize(0.004) & $0.024$ \scriptsize(0.005) \\
\midrule
Double-curvature graph $(3,10)$ & $1$ & $0.028$ \scriptsize(0.020) & $0.010$ \scriptsize(0.006) & $0.070$ \scriptsize(0.038) & $0.086$ \scriptsize(0.022) & $0.086$ \scriptsize(0.021) \\
 & $2$ & $0.020$ \scriptsize(0.012) & $0.010$ \scriptsize(0.006) & $0.070$ \scriptsize(0.039) & $0.048$ \scriptsize(0.011) & $0.050$ \scriptsize(0.011) \\
 & $4$ & $0.017$ \scriptsize(0.012) & $0.010$ \scriptsize(0.006) & $0.070$ \scriptsize(0.038) & $0.029$ \scriptsize(0.007) & $0.031$ \scriptsize(0.007) \\
 & $8$ & $0.010$ \scriptsize(0.007) & $0.010$ \scriptsize(0.006) & $0.070$ \scriptsize(0.038) & $0.020$ \scriptsize(0.005) & $0.022$ \scriptsize(0.005) \\
\bottomrule
\end{tabular}
}
\caption{Alternative private estimators ($m=1$). Latent RMSE (standard deviation in parentheses) under otherwise default settings, with a separate release for each query.}
\label{tab:app-baselines-single}
\end{table}

\subsection{High ambient dimension and computation}
\label{app:sim-highd}

Case~6 varies the ambient dimension $D\in\{10,20,30,40,50\}$ on the trefoil knot with $n=8000$ and on the bowl with $n=16000$, under otherwise default settings with $\varepsilon=2$. Table~\ref{tab:highd} reports the full results underlying Figure~\ref{fig:highd}, including the non-private estimators.

\begin{table}[htbp]\centering\small\setlength{\tabcolsep}{4pt}
\fittable{\begin{tabular}{lcccccc}
\toprule
Geometry & $D$ & RAW & NP-MD-H & NP-MD-S & DP-MD-H & DP-MD-S \\
\midrule
Trefoil ($n=8000$) & $10$ & $0.080$ & $0.005$ \scriptsize(0.000) & $0.006$ \scriptsize(0.000) & $0.013$ \scriptsize(0.001) & $0.017$ \scriptsize(0.002) \\
 & $20$ & $0.080$ & $0.005$ \scriptsize(0.000) & $0.005$ \scriptsize(0.000) & $0.019$ \scriptsize(0.002) & $0.022$ \scriptsize(0.001) \\
 & $30$ & $0.080$ & $0.004$ \scriptsize(0.000) & $0.005$ \scriptsize(0.000) & $0.024$ \scriptsize(0.001) & $0.026$ \scriptsize(0.001) \\
 & $40$ & $0.080$ & $0.004$ \scriptsize(0.000) & $0.005$ \scriptsize(0.000) & $0.028$ \scriptsize(0.001) & $0.029$ \scriptsize(0.001) \\
 & $50$ & $0.080$ & $0.004$ \scriptsize(0.000) & $0.004$ \scriptsize(0.000) & $0.030$ \scriptsize(0.001) & $0.032$ \scriptsize(0.001) \\
\midrule
Bowl ($n=16000$) & $10$ & $0.080$ & $0.008$ \scriptsize(0.002) & $0.010$ \scriptsize(0.002) & $0.030$ \scriptsize(0.007) & $0.031$ \scriptsize(0.007) \\
 & $20$ & $0.080$ & $0.009$ \scriptsize(0.001) & $0.009$ \scriptsize(0.002) & $0.038$ \scriptsize(0.004) & $0.038$ \scriptsize(0.004) \\
 & $30$ & $0.080$ & $0.009$ \scriptsize(0.002) & $0.008$ \scriptsize(0.001) & $0.043$ \scriptsize(0.007) & $0.044$ \scriptsize(0.007) \\
 & $40$ & $0.080$ & $0.009$ \scriptsize(0.001) & $0.008$ \scriptsize(0.002) & $0.045$ \scriptsize(0.006) & $0.046$ \scriptsize(0.006) \\
 & $50$ & $0.080$ & $0.009$ \scriptsize(0.001) & $0.008$ \scriptsize(0.002) & $0.047$ \scriptsize(0.007) & $0.048$ \scriptsize(0.007) \\
\bottomrule
\end{tabular}
}
\caption{Case~6. Latent RMSE (standard deviation in parentheses) across ambient dimensions with $\varepsilon=2$.}
\label{tab:highd}
\end{table}

\paragraph{Computational complexity.}
For $m$ queries and $J$ candidate bandwidths, computing the initial empirical moments requires $O(nJmD^2)$ operations. A dense implementation samples the $mJp$-dimensional Gaussian vector by factoring the $mJp\times mJp$ covariance $\mathbf K$ in $O(m^3J^3p^3)$ operations, or $O(Jm^3p^3)$ under per-bandwidth independent releases, where $p=O(D^2)$. 

For the soft correction, denoising and output-variance estimation perform $(R+1)mJ$ convex projections \eqref{eq:moment-stabilization} of $(D+1)\times(D+1)$ matrices, together with $(R+1)mJ$ eigendecompositions and $(R+1)mJB$ spectral-norm evaluations of $D\times D$ matrices. 

For bandwidth selection via \eqref{eq:bandwidth-score}, evaluating the $mJ^2$ empirical kernel masses by reusing pairwise squared distances across scales takes $O\bigl(nmJ(D+J)\bigr)$ operations. Setting $N:=mJp$ and $S:=mJ^2$, reusing the factor of $\mathbf K$ for a dense conditional draw in \eqref{eq:shared-conditional-mass} requires $O(N^2S+NS^2+S^3)$ operations and $O(NS+S^2)$ additional storage.

\subsection{Fourier torus}
\label{app:sim-fourier}

The Fourier torus embeds a two-dimensional manifold nonlinearly into $\mathbb R^{10}$, so no low-dimensional global linear subspace captures its structure. Figure~\ref{fig:fourier} and Table~\ref{tab:app-fourier} report the results across sample sizes $n$ and privacy budgets $\varepsilon$, complementing Table~\ref{tab:sim-dp-pca}.

\begin{figure}[htbp]
\centering
\includegraphics[width=0.7\linewidth]{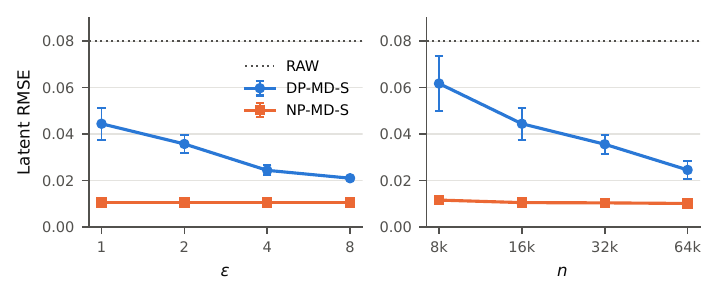}
\caption{Latent RMSE on the Fourier torus across sample sizes and privacy budgets.}
\label{fig:fourier}
\end{figure}

\begin{table}[htbp]\centering\small\setlength{\tabcolsep}{4pt}
\fittable{\begin{tabular}{lccccc}
\toprule
Geometry & setting & RAW & NP-MD-S & DP-MD-S & DP-PCA \\
\midrule
Fourier torus $(2,10)$ & $n=8000$ & $0.080$ & $0.012$ \scriptsize(0.001) & $0.062$ \scriptsize(0.012) & $0.167$ \scriptsize(0.003) \\
 & $n=16000$ & $0.080$ & $0.010$ \scriptsize(0.001) & $0.044$ \scriptsize(0.007) & $0.166$ \scriptsize(0.002) \\
 & $n=32000$ & $0.080$ & $0.010$ \scriptsize(0.000) & $0.036$ \scriptsize(0.004) & $0.166$ \scriptsize(0.001) \\
 & $n=64000$ & $0.080$ & $0.010$ \scriptsize(0.000) & $0.025$ \scriptsize(0.004) & $0.166$ \scriptsize(0.001) \\
 & $\varepsilon=2$ & $0.080$ & $0.010$ \scriptsize(0.001) & $0.036$ \scriptsize(0.004) & $0.166$ \scriptsize(0.001) \\
 & $\varepsilon=4$ & $0.080$ & $0.010$ \scriptsize(0.001) & $0.024$ \scriptsize(0.002) & $0.166$ \scriptsize(0.001) \\
 & $\varepsilon=8$ & $0.080$ & $0.010$ \scriptsize(0.001) & $0.021$ \scriptsize(0.001) & $0.166$ \scriptsize(0.001) \\
\bottomrule
\end{tabular}
}
\caption{Fourier torus. Latent RMSE (standard deviation in parentheses) as $n$ varies at $\varepsilon=1$ and as $\varepsilon$ varies at $n=16000$.}
\label{tab:app-fourier}
\end{table}

\section{Additional details for the UK Biobank application}
\label{app:paper-ukb}

\subsection{Accelerometry}
\label{app:paper-ukb-acc}

\subsubsection{Cohort, preprocessing, and privacy}
\label{app:paper-ukb-mechanism}
The UK Biobank wrist-accelerometry cohort includes $82292$ participants, each with at least six valid recording days ($\ge 16$ valid hours per day). We partition these participants into disjoint reference, validation, and evaluation cohorts. From the reference cohort, $n=20000$ participants form the private reference dataset. We draw $20$ disjoint subsets of $m=100$ participants from the evaluation cohort to serve as public queries across the $20$ repetitions.

Each daily recording is summarized by $D=28$ GGIR activity features. For every participant, we average odd and even recording days separately to form the input query profile and the repeat-day evaluation target, respectively.

\paragraph{Private preprocessing.}
We normalize the data privately in two stages:
\begin{itemize}[leftmargin=1.5em,itemsep=0.25em]
    \item {\emph{Coordinatewise scaling.} Impose fixed bounds $[-24,48]$ for clock-time variables, $[0,16000]$ for magnitudes, and $[0,1440]$ for durations, then transform all non-clock variables and bounds via $x\mapsto\log(1+x)$. For each $k\in[D]$, use the joint exponential mechanism of \citet{gillenwater2021quantiles} to release the $25\%$, $50\%$, and $75\%$ quantiles $(\tilde q_{25,k}, \tilde q_{50,k}, \tilde q_{75,k})$, and set $c_k:=\tilde q_{50,k}$ and $s_k:=(\tilde q_{75,k}-\tilde q_{25,k})\vee 0.001$.\par}
    \item \emph{Global radial scaling.} Let $v_{ik}$ denote the transformed $k$th coordinate of participant $i$. Using a predetermined partition of the $n=20000$ reference indices into $10000$ disjoint pairs $(i_\ell,j_\ell)$, compute
    \[
    d_\ell := \Biggl(\sum_{k=1}^D \Bigl(\frac{v_{i_\ell k}-v_{j_\ell k}}{s_k}\Bigr)^2\Biggr)^{1/2}\vee 10^{-6},
    \qquad \ell\in[10000].
    \]
    Release the $75\%$ quantile $\tilde r_{75}$ via a continuous exponential mechanism on $[10^{-6},U_{\mathrm{dist}}]$ with score $u(t):=-\bigl|\,|\{\ell:d_\ell\le t\}|-7500\,\bigr|$, where $U_{\mathrm{dist}}:=\bigl(\sum_{k=1}^D((U_k-L_k)/s_k)^2\bigr)^{1/2}$ and $[L_k,U_k]$ are the transformed bounds. Under replace-one adjacency, each participant appears in one pair, so the conditional sensitivity is $1$.
\end{itemize}
Each coordinate is then normalized as $z_{ik}:=2(v_{ik}-c_k)/(s_k\tilde r_{75})$.

\paragraph{Privacy accounting.}
{We compose privacy via Gaussian differential privacy (GDP) \citep{dong2022gaussian}, where budgets add in $\mu^2=\sum_j\mu_j^2$. The target $(\varepsilon,\delta)=(2,n^{-1.1})$ is converted to a total $\mu$-GDP budget and split as:\par}
\begin{itemize}[leftmargin=1.5em,itemsep=0.2em]
    \item $16\%$ of $\mu^2$ split equally across the $28$ coordinatewise quantile releases;
    \item $4\%$ of $\mu^2$ for the conditional distance quantile, totaling $20\%$ for preprocessing;
    \item $80\%$ of $\mu^2$ for the main release.
\end{itemize}
Each component is calibrated to its share $\mu_j$ using $\mu_{\mathrm{BR}}(\omega)=2\Phi^{-1}\bigl(1/(1+e^{-\omega/2})\bigr)$ for an $\omega$-bounded-range exponential mechanism \citep{chen2026optimalconversion} and $\varsigma^{-1}$-GDP for a Gaussian mechanism with unit $\ell_2$-sensitivity and noise scale $\varsigma$.

\subsubsection{Competing methods and evaluation}
\label{app:paper-ukb-baseline}
We compare DP-MD-S with NP-MD-S, DP mean shrinkage, and global DP-PCA in the same standardized coordinates, using the same private preprocessing for all private estimators.

\paragraph{DP mean shrinkage.}
Each query $\mathbf z$ is mapped to $(1-\alpha)\mathbf z + \alpha\tilde{\boldsymbol\mu}$, where $\tilde{\boldsymbol\mu}$ is released by the Gaussian mechanism with clipping radius $2$. The James--Stein variant sets
\[
\alpha := 1\wedge\frac{(D-2)\hat\sigma^2}{\|\mathbf z-\tilde{\boldsymbol\mu}\|^2},
\qquad
\hat\sigma^2 := \frac{L-1}{DL}\sum_{t=1}^{L}\bigl\|\mathbf z^{(-t)}-\bar{\mathbf z}^{\mathrm{jack}}\bigr\|^2,
\]
where $L$ is the participant's number of odd recording days, $\mathbf z^{(-t)}$ is the standardized odd-day profile after omitting odd day $t$, and $\bar{\mathbf z}^{\mathrm{jack}}:=L^{-1}\sum_{t=1}^{L}\mathbf z^{(-t)}$. The validation-tuned variant selects $\alpha\in\{0,0.01,\dots,1\}$ by minimizing mean repeat-day discrepancy on $200$ validation participants, yielding $\alpha=0.28$ in every repetition.

\paragraph{Global DP-PCA.}
We follow the private moment release and spectrum-based rank rule of Appendix~\ref{app:sim-dp-baseline}, which consistently selects $\widehat r=1$. The validation-tuned variant selects $\widehat r\in[27]$ by minimizing mean repeat-day discrepancy on the same $200$ validation participants, yielding $\widehat r\in\{8,9\}$.

\paragraph{Evaluation metrics.}
Let $\mathbf z_j$, $\mathbf z'_j$, and $\tilde{\mathbf z}_j$ denote the uncorrected odd-day, even-day, and corrected odd-day profiles in the standardized coordinates, with sample covariance matrices $\mathbf S_{\mathbf z}$ and $\mathbf S_{\tilde{\mathbf z}}$. We report relative discrepancy reduction, between-participant variance retained, and the proportion of participants improved:
\[
1-\frac{\sum_{j=1}^m\|\tilde{\mathbf z}_j-\mathbf z'_j\|}{\sum_{j=1}^m\|\mathbf z_j-\mathbf z'_j\|},
\qquad
\frac{\operatorname{tr}(\mathbf S_{\tilde{\mathbf z}})}{\operatorname{tr}(\mathbf S_{\mathbf z})},
\qquad
\frac1m\sum_{j=1}^m\mathbb I\bigl(\|\tilde{\mathbf z}_j-\mathbf z'_j\|<\|\mathbf z_j-\mathbf z'_j\|\bigr).
\]

\subsection{Spirometry}
\label{app:paper-ukb-spiro}

\subsubsection{Cohort, perturbation, and privacy}
\label{app:paper-ukb-spiro-data}
At the baseline assessment visit (2006--2010), UK Biobank recorded volume--time curves for up to three forced expirations per participant at $10\,\mathrm{ms}$ intervals, together with a quality indicator for each blow. Retaining only acceptable blows, we randomly select one blow per participant and determine time zero by back-extrapolating the steepest tangent over an $80\,\mathrm{ms}$ window to zero volume.

From each aligned curve, we extract $D=10$ features: exhaled volumes at $0.25$, $0.5$, $0.75$, $1$, $1.5$, $2$, $3$, $4$, and $6\,\mathrm{s}$, together with peak expiratory flow, defined as the steepest tangent slope. Curves terminating before $6\,\mathrm{s}$ are carried forward at their final recorded volume. Coordinates are truncated to $[0,8]\,\mathrm{L}$ for volumes and $[0,16]\,\mathrm{L/s}$ for peak flow, and $\mathrm{FEV}_1/\mathrm{FEV}_6$ is the ratio of the fourth to the ninth coordinate.

\paragraph{Cohort partition and COPD matching.}
Participants are partitioned into three disjoint sets:
\begin{itemize}[leftmargin=1.5em,itemsep=0.25em]
    \item \emph{Private reference} ($n=20000$). Drawn randomly from participants with at least two acceptable blows.
    \item \emph{Calibration cohort} ($4000$). Drawn from participants with at least two acceptable blows to set the perturbation scale and evaluation coordinates, with $200$ participants used for validation-tuned DP-PCA.
    \item \emph{Matched COPD queries} ($5000$ pairs). Cases are participants with incident COPD recorded after baseline in the UK Biobank algorithmically defined outcomes, whereas controls have no recorded COPD diagnosis. Outside the reference and calibration sets, eligible cases are matched in random order $1{:}1$ without replacement to controls on sex, smoking status, and five-year age band to form $5000$ pairs, which are partitioned into $20$ disjoint subsets of $250$ pairs ($m=500$ queries per repetition).
\end{itemize}
Baseline characteristics of the $5000$ matched pairs are summarized below.
\begin{center}
\small\setlength{\tabcolsep}{8pt}
\begin{tabular}{@{}lr@{}}
\toprule
Characteristic & Value \\
\midrule
Mean age in cases, years (s.d.) & $60.7$ ($6.6$) \\
Male (\%) & $55.4$ \\
Never / previous / current smoker (\%) & $18.4$ / $45.8$ / $35.8$ \\
Mean $\mathrm{FEV}_1/\mathrm{FEV}_6$, cases & $0.691$ \\
Mean $\mathrm{FEV}_1/\mathrm{FEV}_6$, controls & $0.750$ \\
\bottomrule
\end{tabular}
\end{center}

\paragraph{Perturbation model.}
\label{app:paper-ukb-spiro-noise}
Let $\mathbf v_j^\star$ denote the unperturbed recorded curve in physical units. Each public query is generated as $\mathbf v_j := \mathbf v_j^\star + \boldsymbol\epsilon_j$ with independent coordinates
\[
\epsilon_{jk}\sim\mathcal N\biggl(0,\,\Bigl(\frac{\sigma_{\mathrm{rel}} s_k^{\mathrm{cal}} r_{75}^{\mathrm{cal}}}{2}\Bigr)^2\biggr),
\]
where $s_k^{\mathrm{cal}}$ is the interquartile range of coordinate $k$ and $r_{75}^{\mathrm{cal}}$ is the $75\%$ quantile of standardized pairwise distances across fixed disjoint pairs, both computed on the calibration cohort ($2/r_{75}^{\mathrm{cal}}=0.52$). At $\sigma_{\mathrm{rel}}=0.2$, the coordinatewise noise standard deviations increase from $200\,\mathrm{mL}$ at $0.25\,\mathrm{s}$ to $388\,\mathrm{mL}$ at $1\,\mathrm{s}$ and $498\,\mathrm{mL}$ at $6\,\mathrm{s}$, reaching $1.0\,\mathrm{L/s}$ for peak flow.

\paragraph{Privacy accounting.}
\label{app:paper-ukb-spiro-privacy}
Private preprocessing and the Gaussian-field release follow the GDP accounting in Appendix~\ref{app:paper-ukb-mechanism}, allocating $20\%$ of $\mu^2$ to preprocessing ($16\%$ across the $10$ coordinatewise quantiles and $4\%$ to the distance quantile) and $80\%$ to the main release.

\subsubsection{Competing methods and evaluation}
\label{app:paper-ukb-spiro-eval}
In each repetition, DP-MD-S, NP-MD-S, and global DP-PCA are applied in the same privately scaled coordinates.

\paragraph{Global DP-PCA.}
We follow Appendix~\ref{app:sim-dp-baseline}, where the spectrum-based rank rule selects $\widehat r=1$ in every repetition. The validation-tuned variant selects $\widehat r\in[9]$ by minimizing reconstruction error on $200$ calibration curves perturbed under the same noise model, choosing $\widehat r=2$ in every repetition.

\paragraph{Reconstruction evaluation.}
All estimates are mapped back to physical units and standardized into the fixed calibration coordinates $2(v_k - c_k^{\mathrm{cal}})/(s_k^{\mathrm{cal}} r_{75}^{\mathrm{cal}})$, where $c_k^{\mathrm{cal}}$ is the calibration-cohort median of coordinate $k$ and $s_k^{\mathrm{cal}}, r_{75}^{\mathrm{cal}}$ are defined above. Writing $\mathbf z_j^\star$, $\mathbf z_j$, and $\tilde{\mathbf z}_j$ for the unperturbed, perturbed, and corrected curves in these common coordinates, we measure the relative reduction in Euclidean error by
\[
1 - \frac{\sum_{j=1}^m \|\tilde{\mathbf z}_j - \mathbf z_j^\star\|}{\sum_{j=1}^m \|\mathbf z_j - \mathbf z_j^\star\|}.
\]

\paragraph{Downstream COPD analysis.}
\label{app:paper-ukb-spiro-assoc}
For each curve, $\mathrm{FEV}_1/\mathrm{FEV}_6$ is truncated to $[0,1]$ and standardized by the sample standard deviation of the unperturbed ratios within each subset of $500$ participants. Conditional logistic regression is fitted as intercept-free logistic regression on the $250$ within-pair differences.

\end{document}